\documentclass[twoside,11pt]{article}

\usepackage[preprint]{jmlr2e}
\hypersetup{
    colorlinks=true,
    citecolor=blue,
    linkcolor=blue,
    urlcolor=blue
}

\usepackage{lastpage}
\jmlrheading{??}{2026}{1-\pageref{LastPage}}{06/25; Revised ??/??}{??/??}{[PAPER ID]}{Thomas Boudou, Batiste Le Bars, Nirupam Gupta, Aurélien Bellet}

\ShortHeadings{Dangerous Liaisons of Convex Learning and Non-Affine Aggregation}{Boudou, Le Bars, Gupta, and Bellet}
\firstpageno{1}

\usepackage{mathtools}
\usepackage{mathrsfs}  
\usepackage{booktabs}
\usepackage{hyperref}
\usepackage{comment}
\usepackage{enumitem}
\usepackage{bibentry}
\usepackage{mdframed}
\usepackage{tikz}
\usetikzlibrary{calc}
\usepackage[capitalize,noabbrev]{cleveref}
\newtheorem{assumption}{Assumption}

\makeatletter
\newcommand{\newreptheorem}[2]{%
  \newenvironment{rep#1}[1]{%
    \begingroup
    \def\rep@ref{##1}%
    \edef\oldtheoremnumber{\csname the#1\endcsname}%
    \expandafter\def\csname the#1\endcsname{\ref{\rep@ref}}%
    \begin{#1}%
  }{%
    \end{#1}%
    \endgroup
  }%
}
\makeatother
\newreptheorem{theorem}{Theorem}
\newreptheorem{lemma}{Lemma}
\usepackage{tcolorbox}
\newtcolorbox{block}[1]{
  colback=white,
  colframe=black,
  fonttitle=\bfseries,
  title=#1
}

\newcommand\mc{\mathcal}
\newcommand\mb{\mathbb}

\DeclarePairedDelimiter\floor{\lfloor}{\rfloor}

\newcommand{\R}{\mathbb{R}}

\newcommand{\E}{\mathbb{E}}

\newcommand{\F}{\mathcal{F}}
\newcommand{\A}{\mathcal{A}}
\DeclareMathOperator*{\esssup}{ess\,sup}

\usepackage{amsmath}
\makeatletter
\let\save@mathaccent\mathaccent
\newcommand*\if@single[3]{%
  \setbox0\hbox{${\mathaccent"0362{#1}}^H$}%
  \setbox2\hbox{${\mathaccent"0362{\kern0pt#1}}^H$}%
  \ifdim\ht0=\ht2 #3\else #2\fi
  }
\newcommand*\rel@kern[1]{\kern#1\dimexpr\macc@kerna}
\newcommand*\widebar[1]{\@ifnextchar^{{\wide@bar{#1}{0}}}{\wide@bar{#1}{1}}}
\newcommand*\wide@bar[2]{\if@single{#1}{\wide@bar@{#1}{#2}{1}}{\wide@bar@{#1}{#2}{2}}}
\newcommand*\wide@bar@[3]{%
  \begingroup
  \def\mathaccent##1##2{%
    \let\mathaccent\save@mathaccent
    \if#32 \let\macc@nucleus\first@char \fi
    \setbox\z@\hbox{$\macc@style{\macc@nucleus}_{}$}%
    \setbox\tw@\hbox{$\macc@style{\macc@nucleus}{}_{}$}%
    \dimen@\wd\tw@
    \advance\dimen@-\wd\z@
    \divide\dimen@ 3
    \@tempdima\wd\tw@
    \advance\@tempdima-\scriptspace
    \divide\@tempdima 10
    \advance\dimen@-\@tempdima
    \ifdim\dimen@>\z@ \dimen@0pt\fi
    \rel@kern{0.6}\kern-\dimen@
    \if#31
      \overline{\rel@kern{-0.6}\kern\dimen@\macc@nucleus\rel@kern{0.4}\kern\dimen@}%
      \advance\dimen@0.4\dimexpr\macc@kerna
      \let\final@kern#2%
      \ifdim\dimen@<\z@ \let\final@kern1\fi
      \if\final@kern1 \kern-\dimen@\fi
    \else
      \overline{\rel@kern{-0.6}\kern\dimen@#1}%
    \fi
  }%
  \macc@depth\@ne
  \let\math@bgroup\@empty \let\math@egroup\macc@set@skewchar
  \mathsurround\z@ \frozen@everymath{\mathgroup\macc@group\relax}%
  \macc@set@skewchar\relax
  \let\mathaccentV\macc@nested@a
  \if#31
    \macc@nested@a\relax111{#1}%
  \else
    \def\gobble@till@marker##1\endmarker{}%
    \futurelet\first@char\gobble@till@marker#1\endmarker
    \ifcat\noexpand\first@char A\else
      \def\first@char{}%
    \fi
    \macc@nested@a\relax111{\first@char}%
  \fi
  \endgroup
}
\makeatother

\usepackage{todonotes}

\newcommand{\tb}[1]{\textcolor{orange}{#1}}

\providecommand{\iprod}[2]{\ensuremath{\left\langle #1,\,#2  \right\rangle}}

\usepackage[normalem]{ulem}
\usepackage{xparse}
\newif\ifdraft
\drafttrue
\NewDocumentCommand{\Revision}{m m m}{%
  \ifdraft
    {\color{#1}\sout{#2}}%
    {\color{#1}#3}%
  \else
    #3%
  \fi
}
\newcommand{\DefineReviewer}[2]{%
  \expandafter\newcommand\csname #1\endcsname[2]{%
    \Revision{#2}{##1}{##2}%
  }%
}
\DefineReviewer{tbrev}{orange}
\DefineReviewer{batrev}{magenta}
\DefineReviewer{ngprev}{blue!60!white}
\DefineReviewer{aurev}{green!80!black}

\usepackage{times}

\begin{document}

\title{Dangerous Liaisons of Convex Learning and Non-Affine Aggregation}

\author{\!
  \name Thomas Boudou 
  \email thomas.boudou@inria.fr \\
  \addr PreMeDICaL team, Inria, Idesp, Inserm, Université de Montpellier, Montpellier, France\\
  \AND
  \name Batiste Le Bars 
  \email batiste.le-bars@inria.fr \\
  \addr Univ.\ Lille, Inria, CNRS, Centrale Lille, UMR 9189, CRIStAL, F-59000 Lille, France\\
  \AND
  \name Nirupam Gupta 
  \email nigu@di.ku.dk \\
  \addr Department of Computer Science, University of Copenhagen, Copenhagen, Danemark\\
  \AND
  \name Aurélien Bellet 
  \email aurelien.bellet@inria.fr \\
  \addr PreMeDICaL team, Inria, Idesp, Inserm, Université de Montpellier, Montpellier, France
}

\editor{??}

\maketitle

\begin{abstract}%
  Last-iterate convergence and generalization guarantees in first-order convex learning hinge on the monotonicity of the update operator.
  While linear averaging preserves the monotonicity of gradient updates, this property is often violated when gradients are aggregated non-affinely, as in modern pipelines enforcing constraints like adaptivity, privacy, robustness or fairness. 
  Whether it is possible to design non-affine aggregation rules that maintain monotonicity has remained an open question.
  We answer this question negatively: we prove that the monotonicity of aggregated gradients is preserved if and only if the aggregation rule is positively affine. 
  Consequently, non-affine aggregation prevents steady convergence and substantially degrade algorithmic stability. 
  We quantify these drawbacks and propose a path forward by identifying sufficient conditions under which monotonicity can be restored.
  Our results provide a unified theoretical framework explaining the disparate failure modes observed in modern learning systems.
\end{abstract}

\section{Introduction}\label{I-intro}
Most modern learning algorithms rely on iterative first-order optimization to minimize the expected risk $R(\theta) \vcentcolon= \E_{z \sim p}[\ell(\theta; z)]$ of a loss $\ell$  over an unknown distribution $p$, typically accessible only through a finite sample~\citep{NIPS2007_bottou_bousquet}.
At their core, these methods compute update directions by aggregating gradient information. 
Gradients may be obtained directly from $p$ (stochastic optimization), estimated from samples (empirical risk minimization), or perturbed from the nominal distribution (robust optimization, transfer learning). 
We abstract away these distinctions, as the specific source of gradients is secondary to our main result.
Formally, given $d, n, T \in \mb{N} \setminus \{0\}$, we consider algorithms that update model parameters $\theta_t \in \mb{R}^d$ at step $t \in [T] \vcentcolon= \{1, \ldots, T\}$ via an aggregation rule $F: (\mb{R}^d)^{n} \to \mb{R}^d$ applied to input vectors $\{g_i(\theta_t)\}_{i\in[n]}$,
\begin{equation}\label{eq:algo-update}
    % G^F_{\gamma}(\theta_t)
    \theta_{t+1} = \theta_t - \gamma F(g_1(\theta_t), \ldots, g_n(\theta_t)), \text{ where } \theta_0 \in \R^d, \gamma > 0.
\end{equation}
We denote $\F: \R^d \to \R^d, \theta \mapsto F(g_1(\theta), \ldots, g_n(\theta))$ the induced operator. % from $F$ and $\{g_i\}_{i\in[n]}$. 
Typically, the gradients take the form $g_i(\theta) = \nabla_\theta \ell(\theta;z_i)$ for some data $z_i$, and $n$ abstracts the relevant scale depending on context (e.g., number of samples, batch size, client count).

While a standard choice of $F$ is simple averaging, as in mini-batch stochastic gradient descent (SGD), this framework also captures many modern learning paradigms in which $F$ is a \emph{non-affine aggregation rule} designed to enforce application-specific constraints.
Examples include:
(1) Adaptive optimization:~\citet{kingma2017adammethodstochasticoptimization} reduce tuning overhead via coordinate-wise adaptive learning rates, $F(g_1,\ldots,g_n) = \operatorname{diag}(v(g_1, \ldots, g_n))^{-\frac{1}{2}} \sum_{i\in[n]} \frac{g_i}{n}$; % (effectively applying a diagonal scaling to the gradient) 
(2) Robust (distributed) learning:~\citet{pmlr-v206-allouah23a,Diakonikolas_Kane_2023} mitigate outliers via non-linear filtering (cf.\ Definition~\ref{cwtm-definition});
(3) Privacy-preserving learning:~\citet{Abadi_2016} control aggregate sensitivity via individual gradient clipping, $F(g_1, \ldots, g_n) = \frac{1}{n}\sum_{i\in[n]} g_i \min(1, \frac{C}{\|g_i\|_2})$;
(4) Fairness and multi-objective learning:~\citet{pmlr-v80-hashimoto18a} adaptively reweight gradients, $F(g_1,\ldots,g_n) = \sum_{i=1}^n w_i g_i$ where $w_i$ depends on $\{g_i\}_{i\in[n]}$ or the model state; %the $g_i's$ or the current model state.
%(5) Communication-efficient learning:~\citet{alistarh2017qsgd} reduces inputs via a quantization operator $\mc{Q}$, where $F(g^{(1)}, \ldots, g^{(n)}) = \sum_{i=1}^n \mc{Q}(g^{(i)})$.
%(5) Communication-efficient learning:~\citet{NEURIPS2018_b440509a} sparsify gradients via top-$k$ thresholding $\mathcal{S}_k$, $F(g_1, \ldots, g_n) = \sum_{i=1}^n \mathcal{S}_k(g_i)$.
(5) Communication-efficient learning:~\cite{bernstein2018signsgd} compress gradients to their signs $F(g_1, \ldots, g_n) = \operatorname{sign}(\sum_{i=1}^n \operatorname{sign}(g_i))$.
% \autodo{in this paper I think they don't sparsify gradients; so I'd pick another ref, such as SignSGD \url{https://arxiv.org/pdf/1802.04434}. \\ \tb{TB: But the paper focus on non-convex optimization. Maybe https://arxiv.org/abs/1610.02132?}}
%(6) Generative Modeling:~\citet{} ?.

In convex learning, where inputs $g_i$ are gradients of a convex loss function (definitions deferred to Appendix~\ref{loss-regularities}), strong theoretical guarantees of the last-iterate expected risk of the update rule~\eqref{eq:algo-update} hinge on the \emph{monotonicity} of $\F$. 
Specifically, $\F$ is monotone if
\begin{equation}\label{eq:operator-convexity}
    \forall \theta, \omega \in \R^d,\quad \langle \theta - \omega, \F(\theta) - \F(\omega) \rangle \geq 0.
\end{equation}
This property extends the notion of convexity to general operators, generalizing the fact that convex scalar functions have monotone gradients~\citep[Proposition 17.10]{Bauschke2011ConvexAA}.
When $F$ is the arithmetic mean (as in mini-batch SGD) and the inputs $g_i$ are monotone, the resulting operator $\F$ preserves this monotonicity. 
In contrast, while non-affine aggregation rules are motivated by important learning objectives, as illustrated above, it remains unclear whether they can achieve these goals without compromising the monotonicity of updates.

\paragraph{Main result.}
Our main result is of theoretical nature and can be summarized by the following (informal) impossibility theorem, explicitly stated and proved in Section~\ref{sec:impossinility}.  

\begin{mdframed}[backgroundcolor=gray!5]
\vspace{0.5\topsep}    
\textbf{Theorem~\ref{thm:universal_convexity} (informal)} \emph{Only positively affine aggregation universally preserves monotonicity.}
\vspace{0.5\topsep}
\end{mdframed}
% \batodo{
%     Maybe we should clearly define what we mean by 'universal monotonicity preserver'. 
%     This may also simplify some theorem statements later. 
%     Example: "Here, and in the rest of the paper, a \emph{universal} monotonicity preserver refers to an Aggregation rule that preserves the monotonicity of any monotonic input gradients. 
%     We should also explain why finding such aggregation rule is important: having simple proofs with no additional errors or good results with no need for extra assumption on the input monotonic gradient vectors. 
%     I think this will better highlight the importance of your negative result. 
%     \au{I tried to incorporate this a bit}
% }
\noindent Here, \emph{universal} refers to an aggregation rule that preserves the monotonicity of \emph{any} set of monotone input gradients. 
Theorem~\ref{thm:universal_convexity} shows that for every non-affine aggregation, there exist monotone input gradients $g_i$ such that the induced operator $\F$ violates monotonicity. 
In other words, only affine aggregation guarantees monotonicity without imposing additional assumptions on the inputs.
This result answers negatively the open question of whether a non-affine aggregation rule can preserve monotonicity for arbitrary input gradient vectors. 
Consequently, theoretical guarantees for updates using non-affine aggregation generally cannot be derived directly from those for monotone updates.

We highlight three consequences of Theorem~\ref{thm:universal_convexity} for algorithms following update~\eqref{eq:algo-update} under non-affine aggregation: 
(I) it explains why these algorithms do not enjoy last-iterate convergence guarantees, unless one make additional assumptions;
% explains why last-iterate convergence break down; 
% \tbtodo{I am unsure of the claim. In fact, for instance Adagrad-Norm is proven convergent (but with complex hyperparameter). But maybe it does not this is linked with non-affine at stationarity which may not be the case when we assume the gradients come from the same loss and datadistribution?}
% \ngtodo{
%    I'm still afraid that the above does not prove non-convergence, since the trajectory may return to $\theta_{eq}$ from outside the cone $\mc{K}$. 
%    But it does show that $\theta_{eq}$ is no longer a (Lyapunov-)stable equilibrium. 
%    In any case, it shows that the distance need not decrease monotonically and therefore the existing optimal convergence proofs for gradient-descent under convexity breaks down.
% }
% \autodo{"it explains why last-iterate convergence proofs break down" feels a bit weak (both due to the use of "explains" and "proofs"); it seems we claim more in Sec 3.1 ("the loss of monotonicity precludes guarantees of last-iterate convergence"). To remain a bit conservative, I would simply remove "proofs": "it explains why last-iterate convergence break down" or (perhaps better) reformulate to "it explains why these algorithms do not enjoy last-iterate convergence guarantees"}
% \autodo{similarly, "it clarifies the need to account for additional algorithmic instability" feels weak due to "clarifies". Could we replace by something like "It implies an additional algorithmic instability" or "It implies a degradation in algorithmic stability (guarantees)"}
(II) it implies an additional algorithmic instability; and 
(III) it offers insights into designing non-affine rules that preserve monotonicity for \emph{restricted} convex inputs.
While some of these points were previously noted in specific contexts, our result provides a unified theory for these scattered observations (cf. the discussion of related work in Section~\ref{related-work}).

\paragraph{Consequence I:\ Last-iterate convergence failures (cf.~Section~\ref{sec:convergenceimplications}).}
Last-iterate convergence, i.e., the convergence of the update~\eqref{eq:algo-update}, is a strong theoretical property~\citep{bach-ltfp,doi:10.1137/24M1717762-last-iterate-cv,preobrazhenskaia2026last-iterate-adagrad,kornowski2026gradient-last-iterate-speed-cv} that, unlike convergence of averaged iterates, better aligns with practical implementations and preserves structural properties such as parameter sparsity. 
While averaged convergence can be obtained even for non-monotonic updates, last-iterate convergence of first-order methods for convex losses rely on the update's ability to monotonically decrease the distance to the optimum---or equilibrium, i.e., a stationary point of $\F$~\citep{Bauschke2011ConvexAA,primer-monotone}.
However, in the absence of monotonicity, the operator may locally push the iterate away from the equilibrium $\theta_{eq}$. 
Specifically, we demonstrate in Theorem~\ref{cor:cone-instability} that there exists a cone $\mc{K}$ emanating from $\theta_{eq}$ such that $\forall \theta_0 \in \mc{K}$, $\langle \F(\theta_0), \theta_0 - \theta_{eq} \rangle < 0$.
Consequently, for any step-size $\gamma > 0$, the update step $\theta_1 = \theta_0 - \gamma \F(\theta_0)$ strictly increases the distance to the equilibrium
\begin{equation*}
    \| \theta_{1} - \theta_{eq} \|_2^2 
    = \| \theta_0 - \theta_{eq} \|_2^2 + \gamma^2 \| \F(\theta_0) \|_2^2 - 2\gamma \langle \F(\theta_0), \theta_0 - \theta_{eq} \rangle
    > \| \theta_0 - \theta_{eq} \|_2^2.
\end{equation*}
This inequality invalidates descent arguments central to last-iterate convergence, highlighting that non-affine aggregations are prone to pathological behavior, even in deterministic settings.

Our analysis provides a geometric perspective that is distinct from, and orthogonal to, the phenomena of unavoidable bias and variance-induced oscillations. 
For instance, even an affine aggregation rule may be biased (pointing away from the desired solution, i.e., the solution with simple averaging), and yet still be monotone, thereby preserving stable convergence. 
Similarly, stochastic gradient noise induces oscillations while the update operator remains monotone~\citep{dieuleveut2020bridging}, enabling convergence via step-size decay \citep{NIPS2011_40008b9a-moulines} or variance-reduction methods~\citep{defazio2014saga}.
In contrast, violations of monotonicity fundamentally alter the update geometry: even a deterministic update can produce spurious fixed points~\citep{reddi2018convergenceadam} or stable limit cycles~\citep[Figure 6]{heavy-ball-limit-cycle}, failures that cannot be remedied by step-size decay or variance reduction.

\paragraph{Consequence II:\ Degraded algorithmic stability (cf.~Section~\ref{sec:stabilityimplications}).}
From a learning theory perspective, when we consider a finite sample $\mc{S}$ drawn from $p$, the update~\eqref{eq:algo-update} usually aims to minimize an empirical risk $\widehat{R}(\theta) = \frac{1}{|\mc{S}|} \sum_{z\in\mc{S}} \ell(\theta;z)$ as a proxy for minimizing the true risk $R(\theta_T)$. 
The generalization error $R(\theta_T) - \widehat{R}(\theta_T)$ can be controlled via algorithmic stability~\citep{Bousquet2002StabilityAG,pmlr-v125-bousquet20b}, a framework that has proven particularly effective across a wide range of learning problems and algorithms~\citep{hardt2016train,bassily2019private,schliserman2022stability,boudou2025generalization}. 
Algorithmic stability quantifies the worst-case sensitivity of an algorithm to changes in its training data. 
Formally, this requires bounding the divergence between two trajectories produced by the algorithm on adjacent datasets (differing by a single sample).
This crucially hinges on the non-expansiveness~\citep[Definition 2.3]{hardt2016train} of~\eqref{eq:algo-update}, a property ensured when $\F$ is co-coercive, i.e., when there exists $L_\F > 0$ such that
\begin{equation}\label{eq:operator-cocoercivity}
    \forall \theta, \omega \in \R^d,\quad \langle \theta - \omega, \F(\theta) - \F(\omega) \rangle \geq \frac{1}{L_\F} \| \F(\theta) - \F(\omega) \|_2^2.
\end{equation} 
This condition is a stronger form of monotonicity that holds for gradients of convex, Lipschitz-smooth functions~\citep[e.g.,][Proposition 5.4]{bach-ltfp}.
However, extending Theorem~\ref{thm:universal_convexity}, we prove that non-affine aggregation violates co-coercivity (Theorem~\ref{cor:nococoercivity}).
Hence, for any learning rate $\gamma > 0$ and parameters $\theta, \theta'$ that violate co-coercivity~\eqref{eq:operator-cocoercivity}, the update rule~\eqref{eq:algo-update} is strictly expansive,
\begin{align*}
    \| \theta - \gamma \F(\theta) - (\theta' - \gamma \F(\theta')) \|_2^2 
    &= \| \theta - \theta' \|_2^2
    + \gamma^2 \| \F(\theta) - \F(\theta') \|_2^2
    - 2 \gamma \langle \theta - \theta', \F(\theta) - \F(\theta') \rangle \\
    &> \| \theta - \theta' \|_2^2.
\end{align*}
This strict inequality invalidates stability proofs based on non-expansiveness. 
% Furthermore, we show that non-affine aggregations incur an additional stability penalty compared to their affine counterparts. 
% Indeed, under mild conditions, we construct lower bounds demonstrating that any tight stability analysis must account for the expansivity of~\eqref{eq:algo-update}.
Furthermore, due to this expansiveness, we show in Theorem~\ref{cor:unavoidable-instability} and Lemma~\ref{thm:1-step-dsep} that non-affine aggregations incur an additional stability penalty relative to positively affine rules. 
Under mild conditions, this implies that their overall instability is strictly larger.
% Furthermore, due to this expansiveness, we show in Theorems~\ref{cor:unavoidable-instability} and~\ref{thm:1-step-dsep} that non-affine aggregations incur an additional stability penalty compared to their positively affine counterparts. 
% We further explain that, under mild conditions, the overall stability is strictly greater for non-affine aggregation rules compared to their positively affine counterparts.
This demonstrates that non-affine aggregations degrade stability and, consequently, weaken the generalization guarantees that hold for positively affine aggregation rules.

\paragraph{Consequence III:\ Non-affine monotone aggregations for restricted inputs (cf.~Section~\ref{sec:positive-result-tm}).}
Beyond establishing a fundamental limit of non-affine aggregation, the proof of Theorem~\ref{thm:universal_convexity} also provides a recipe for identifying non-affine aggregations that preserve monotonicity when the input gradients are derived from a restricted class of convex functions.
First, it establishes a testable condition that delineates the boundary: valid non-affine rules exist only if the underlying loss functions prevent arbitrary rank-1 symmetric p.s.d.\ matrix as Hessians. 
Second, the proof is constructive, allowing us to reverse-engineer the properties required for such rules. 
We illustrate this approach by identifying the Coordinate-Wise Trimmed Mean \citep[$\mathrm{CWTM}$,][]{yin2018byzantine} as a monotonicity-preserving rule when the input gradients are coordinate-wise separable (cf. Section~\ref{sec:positive-result-tm}).
Overall, this analysis bridges our impossibility theorem with prior work reporting positive results for non-affine aggregation under restricted inputs (cf.~Section~\ref{related-work}). 
Specifically, our framework explains these exceptions and allows to uncover new situations where non-affine rules can preserve monotonicity.

\section{Only Positively Affine Aggregation Universally Preserves Monotonicity}\label{sec:impossinility}

In this section, we present our main result, Theorem~\ref{thm:universal_convexity}, answering whether a non-affine aggregation rule can preserve monotonicity of any monotone input vectors. 
We focus here on the properties of the induced operator $\F$; their implications for learning algorithms are discussed in Section~\ref{sec:consequences}. \looseness = -1
We denote, for a $v \in \R^d$, $\operatorname{cone}(v) \vcentcolon= \{\lambda v ; \lambda \geq 0\}.$
\begin{theorem}\label{thm:universal_convexity}
    Let $d>1$, and consider the model update rule~\eqref{eq:algo-update}. 
    $\F$ is a monotonic operator for all gradients $\{g_i(\cdot)\}_{i\in[n]}$ induced by convex functions
    if and only if $F$ is a positively affine aggregation, i.e., there exists $\alpha \in \mb{R}_+^n$ and $C \in \mb{R}^d$, such that for any input vectors $(g_1, \ldots, g_n) \in (\mb{R}^d)^n$,
    \[
        F(g_1, \ldots, g_n) = \sum_{k=1}^n \alpha_k g_k + C.
    \]
    If $F$ is also permutation-invariant and idempotent ($F(g, \ldots,g)\! =\! g$), then $F$ is the arithmetic mean.%\looseness=-1
\end{theorem}
\begin{proof}
    If $F$ is a positively affine aggregation, it is immediate that $\F$ preserves monotonicity.
    We now turn to the converse. 
    The proof has three steps.
    First, Step 1 proves a rigid alignment condition: any variation in a single input vector must yield a variation of the aggregation output in the exact same direction.
    We prove this by contradiction, showing that if this property is violated, one can construct valid inputs $\{g_i(\cdot)\}_{i\in[n]}$ such that monotonicity~\eqref{eq:operator-convexity} is violated.
    Then, Steps 2 and 3 deduce the necessary algebraic form of $F$ from this constraint, proceeding first input-wise and then globally.

    \paragraph{Step 1.} We prove that any variation in a single input vector must result in an output shift that is a non-negative scalar multiple of that variation. 
    Fix an index $i \in [n]$ and an arbitrary vector $g_{-i} \vcentcolon= (\ldots, g_{i-1}, g_{i+1}, \ldots) \in (\R^d)^{n-1}$. 
    Consider the partial function
    \[
        f_{g_{-i}}(u) := F(g_1, \ldots, u, \ldots, g_n).
    \]
    Let $u, v \in \mathbb{R}^d$ be arbitrary vectors. 
    We claim that if~\eqref{eq:operator-convexity} is satisfied,
    \[
        f_{g_{-i}}(u) - f_{g_{-i}}(v) \in \operatorname{cone}(u-v) %\vcentcolon= \{\lambda(u-v) ; \lambda \geq 0\}.
    \] 
    If $u = v$, or $f_{g_{-i}}(u) = f_{g_{-i}}(v)$, the claim is trivial. 
    Otherwise, the proof proceeds by contradiction. Suppose that $f_{g_{-i}}(u) - f_{g_{-i}}(v) \notin \operatorname{cone}(u-v)$.
    We define $\delta$ as the bisector of the angle between $(u-v)$ and $-(f_{g_{-i}}(u) - f_{g_{-i}}(v))$,
    \[
        \delta \vcentcolon= \frac{u-v}{\|u-v\|_2} - \frac{f_{g_{-i}}(u) - f_{g_{-i}}(v)}{\|f_{g_{-i}}(u) - f_{g_{-i}}(v)\|_2}.
    \]
    %By the hyperplane separation theorem~\citep[separating a point from a half line,][Theorem 11.4]{rockafellar1970convex}, there exists $\delta \in \mathbb{R}^d$ such that
    As directions differ, the strict Cauchy-Schwarz inequality applies and implies the separation 
    \begin{equation}\label{eq:strict_separation}
        \langle f_{g_{-i}}(u) - f_{g_{-i}}(v), \delta \rangle < 0 \quad \text{and} \quad \langle u-v, \delta \rangle > 0.
    \end{equation}
    Next, we construct convex functions whose gradients at $0$ and $\delta$ match these input configurations.
    \begin{enumerate}[label=$\spadesuit$, ref=($\spadesuit$)]
        \item\label{item:def_R_i} For index $i$, define $\ell_i(\theta) = \frac{1}{2} \theta^\intercal H \theta + v^\intercal \theta$, with $H = \frac{(u-v) (u-v)^\intercal}{\langle u-v, \delta \rangle}$. 
        Since $\langle u-v, \delta \rangle > 0$, $H_i$ is symmetric p.s.d.\ and $\ell_i$ is convex.
        Note that $\nabla \ell_i(0) = v$ and $\nabla \ell_i(\delta) = H \delta + v  = u$.
    \end{enumerate}
    \begin{enumerate}[label=$\clubsuit$, ref=($\clubsuit$)]
        \item\label{item:def_R_j} For indices $j \neq i$, define $\ell_j(\theta) = g_j^\intercal \theta$. 
        Then $\nabla \ell_j(0) = \nabla \ell_j(\delta) = g_j$.
    \end{enumerate}
    Therefore, we get 
    \[
        \F(\delta) = F(\nabla \ell_1(\delta), \ldots, \nabla \ell_n(\delta)) = f_{g_{-i}}(u) \text{ and } \F(0) = F(\nabla \ell_1(0), \ldots, \nabla \ell_n(0)) = f_{g_{-i}}(v).
    \]
    Due to the monotonicity condition~\eqref{eq:operator-convexity}, we have
    \begin{equation}\label{eq:monotonicitycontradiction}
        \langle \F(\delta) - \F(0), \delta - 0 \rangle = \langle f_{g_{-i}}(u) - f_{g_{-i}}(v), \delta \rangle \geq 0.
    \end{equation}
    However, inequality~\eqref{eq:monotonicitycontradiction} contradicts~\eqref{eq:strict_separation}. 
    Thus, for all $u \neq v \in \R^d$, there exists $\lambda(u, v) \geq 0$ such that
    \begin{equation}\label{eq:positivedirectional}
        f_{g_{-i}}(u) - f_{g_{-i}}(v) = \lambda(u,v)(u-v).
    \end{equation}

    \paragraph{Step 2.} We demonstrate that the rigid geometric constraint of preserving input-wise update directions forces the aggregation $F$ to be affine with respect to each individual input vector.
    Evaluating~\eqref{eq:positivedirectional} with $v=0$ yields, for all $u \in \R^d \setminus \{0\}$, $f_{g_{-i}}(u) = \lambda(u, 0)u + f_{g_{-i}}(0)$. 
    We only require to prove that $u \mapsto \lambda(u, 0)$ is a constant. 
    %\Cref{eq:positivedirectional} implies, for all $u \neq v \in \R^d$,
    %\begin{equation}\label{eq:directinaltoaffine-intermediate}
    %    \lambda(u, 0)u - \lambda(v, 0)v = \lambda(u,v)(u-v).
    %\end{equation}
    %Rearranging terms from~\eqref{eq:directinaltoaffine-intermediate}, we have the following linear equation
    %\begin{equation}\label{eq:directinaltoaffine-intermediate2}
    %    (\lambda(u,0) - \lambda(u,v))u + (\lambda(u,v) - \lambda(v,0))v = 0.
    %\end{equation}
    Rearranging the terms in~\eqref{eq:positivedirectional} yields, for all $u \neq v \in \R^d$, the following linear equation
    \begin{equation}\label{eq:directinaltoaffine-intermediate2} 
        (\lambda(u,0) - \lambda(u,v))u + (\lambda(u,v) - \lambda(v,0))v = 0. 
    \end{equation}

    \textit{Case 1: linearly independent inputs.}
    Let any fixed linearly independent vectors $u, v \in \R^d$. 
    By linear independence, the coefficients in~\eqref{eq:directinaltoaffine-intermediate2} must be $0$, i.e., $\lambda(u,0) = \lambda(u,v) = \lambda(v,0)$.

    \textit{Case 2: linearly dependent inputs.}
    If $u = 0 \neq v \in \R^d$,~\eqref{eq:directinaltoaffine-intermediate2} implies $\lambda(v,0) = \lambda(0,v)$.
    Now, let $u \neq v \in \R^d \setminus \{0\}$ be non-zero collinear vectors. 
    Since $d \ge 2$, there exists a vector $z \in \R^d \setminus \{0\}$ linearly independent of $u$, and hence linearly independent of $v$.
    Applying the result from Case 1, $\lambda(u,0) = \lambda(z,0)$ and $\lambda(v,0) = \lambda(z,0)$.
    By transitivity, $\lambda(u,0) = \lambda(v,0)$.

    Consequently, for any pair of vectors $u \neq v \in \R^d $, $\lambda(u,0) = \lambda(v,0)$. 
    Hence $\lambda(\cdot, 0)$ is a constant function.
    Let $\alpha \in \R$ denote this constant, we have proven that 
    \begin{equation}\label{eq:indexaffine-intermediate}
        f_{g_{-i}}(u) = \alpha u + f_{g_{-i}}(0), \text{ for all } u\in\R^d.
    \end{equation}

    \paragraph{Step 3.} We synthesize these input-wise affineness results, proving that the aggregation $F$ must be a globally affine map over the joint domain.
    We derived~\eqref{eq:indexaffine-intermediate} for an arbitrary fixed index $i\in[n]$ and configuration $g_{-i} \in (\R^d)^{n-1}$.
    Therefore, $F$ is input-wise affine, that is, for all $i \in [n]$, all $g_{-i} \in (\R^d)^{n-1}$, we have for all $g_i \in \R^d$,
    \begin{equation}\label{eq:indexaffine}
        F(g_1, \ldots, g_n) = \alpha_i (g_{-i}) g_i + b_i(g_{-i}),
        \quad \text{where } b_i(g_{-i}) \vcentcolon= f_{g_{-i}}(0).
    \end{equation}

    We first show that for any pair of distinct indices $i \neq j \in [n]$, $\alpha_i$ is independent of $g_j$. 
    Without loss of generality, consider indices $1$ and $2$ and fix $\bar{g} \vcentcolon= (g_3, \ldots, g_n) \in (\R^d)^{n-2}$. 
    Equating the representations from~\eqref{eq:indexaffine} yields
    \begin{equation}\label{eq:global-affine-intermediate}
        \alpha_1(g_2, \bar{g}) g_1 + b_1(g_2, \bar{g}) = \alpha_2(g_1, \bar{g}) g_2 + b_2(g_1, \bar{g}).
    \end{equation}
    Evaluating~\eqref{eq:global-affine-intermediate} at $g_1=0$ yields $b_1(g_2, \bar{g}) = \alpha_2(0, \bar{g}) g_2 + b_2(0, \bar{g})$, and further evaluating the relation at $g_2=0$ gives $b_1(0, \bar{g}) = b_2(0, \bar{g})$. 
    Similarly, $b_2(g_1, \bar{g}) = \alpha_1(0, \bar{g}) g_1 + b_1(0, \bar{g})$.
    Substituting this back into~\eqref{eq:global-affine-intermediate}, and rearranging terms to isolate $g_1$ and $g_2$, we have for all $(g_1, \ldots, g_n) \in (\R^d)^n$,
    \begin{equation}\label{eq:lemma_separation}
        (\alpha_1(g_2, \bar{g}) - \alpha_1(0, \bar{g})) g_1 = (\alpha_2(g_1, \bar{g}) - \alpha_2(0, \bar{g})) g_2.
    \end{equation}
    Fix any $g_2 \neq 0$.
    %As~\eqref{eq:lemma_separation} must hold for all $g_1 \in \R^d$, the vector $(\alpha_1(g_2, \bar{g}) - \alpha_1(0, \bar{g}))g_1$ must be collinear with $g_2$ for all $g_1 \in \mathbb{R}^d$.
    %For any $(g_1, \ldots, g_n) \in (\R^d)^n$,~\Cref{eq:lemma_separation} implies that, if $\alpha_1(g_2, g\bar{g}) = \alpha_1(0, \bar{g})$, then $\alpha_2(g_1, \bar{g}) = \alpha_2(0, \bar{g})$, and vice versa.
    %Now assume $\alpha_1(g_2, \bar{g}) \neq \alpha_1(0, \bar{g})$, which forces $\alpha_2(g_1, \bar{g}) \neq \alpha_2(0, \bar{g})$.
    %\Cref{eq:lemma_separation} implies that for any fixed $g_2 \neq 0$, the vector $(\alpha_1(g_2, \bar{g}) - \alpha_1(0, \bar{g}))g_1$ must be collinear with $g_2$ for all $g_1 \in \mathbb{R}^d$.
    Since $\alpha_1(g_2, \bar{g}) - \alpha_1(0, \bar{g})$ is independent of $g_1$, and $d \geq 2$, we can select a vector $g_1'$ linearly independent of $g_2$ in~\eqref{eq:lemma_separation}, so that $(\alpha_1(g_2, \bar{g}) - \alpha_1(0, \bar{g})) g'_1$ is collinear to $(\alpha_2(g'_1, \bar{g}) - \alpha_2(0, \bar{g})) g_2$. 
    The only vector collinear with $(\alpha_2(g'_1, \bar{g}) - \alpha_2(0, \bar{g})) g_2$ in the direction of $g_1'$ is $0$.
    Thus $\alpha_1(g_2, \bar{g}) = \alpha_1(0, \bar{g})$ for all $g_2\in\R^d$, i.e., $\alpha_1$ is independent of $g_2$. 

    By symmetry, for all $(g_1, \ldots, g_n) \in (\R^d)^n$ and each $i \in [n]$, $\alpha_i$ is independent of $g_j$ for all $j \neq i$, and is therefore constant. 
    %(we denote this constant $\alpha_i$, slightly overloading notation).

    Now, define the following residual function, for all $(g_1, \ldots, g_n) \in (\R^d)^n$,
    \begin{equation}\label{eq:residual}
        \textstyle R(g_1, \ldots, g_n) = F(g_1, \ldots, g_n) - \sum_{i=1}^n \alpha_i g_i. 
    \end{equation}
    We prove $R$ is a constant vector.
    Fix $j\in[n]$. Substituting~\eqref{eq:indexaffine} into~\eqref{eq:residual} yields
    \begin{equation}\label{eq:globalaffine-intermediate}
        \textstyle R(g_1, \ldots, g_n) = \alpha_j g_j + b_j(g_{-j}) - \sum_{i=1}^n \alpha_i g_i = b_j(g_{-j}) - \sum_{i \neq j} \alpha_i g_i.
    \end{equation}
    The right-hand side of~\eqref{eq:globalaffine-intermediate} depends only on $g_{-j}$, thus $R$ is independent of $g_j$.
    By symmetry, this holds for every $j \in [n]$, i.e., the residual $R$ is independent of its inputs and is therefore a constant vector.
    Denote this constant by $C \vcentcolon= R(0, \ldots, 0) = F(0, \ldots, 0)$. 
    We have thus shown that
    \begin{equation*}
        \textstyle F(g_1, \ldots, g_n) = \sum_{k=1}^n \alpha_k g_k + C, \text{ for all } (g_1, \ldots, g_n) \in (\R^d)^n.
    \end{equation*}

    Suppose in addition that $F$ is idempotent and, in fact, it suffices that there exist $u \neq v \in \R^d$ such that $F(u, \ldots, u) = u$ and $F(v, \ldots, v) = v$. Then $u-v = (\sum_{i=1}^n\alpha_i) (u-v)$, i.e.\ $\sum_{i=1}^n \alpha_i = 1$. 
    If, moreover, $F$ is permutation-invariant, all coefficients are equal: $\alpha = (\frac{1}{n}, \ldots, \frac{1}{n})^\intercal$.
\end{proof}
Derived from first principles, Theorem~\ref{thm:universal_convexity} requires no assumption on $F$ beyond structure~\eqref{eq:algo-update}, and shows that only positively affine aggregation universally preserves monotonicity over the set of convex functions.
That is, for \textit{any} non-affine aggregation, there exist convex functions such that the induced operator violates monotonicity. The proof above relies on the fact that if $\F$ is universally monotone for all convex functions, then it should be monotone for the specific ones constructed in step 1, namely quadratic function with rank-one hessians. The proof then proceeds with a contradiction argument.

\begin{remark}\label{rem:violation-caped}
    In Appendix~\ref{app:extension}, we extend~Theorem~\ref{thm:universal_convexity} to strong monotonicity (akin to strong convexity) and hypomonotonicity (monotonicity violation is bounded).
    $\F$ is said $\rho$-hypomonotone if there exists $\rho > 0$, such that
    \begin{equation*}\label{eq:hypomonotonicity-main}
        \langle \theta - \omega, \F(\theta) - \F(\omega) \rangle \geq -\rho \| \theta - \omega \|^2_2, \quad\text{ for all } \theta, \omega \in \Theta.
    \end{equation*}
    We show that under the assumption of~Theorem~\ref{thm:universal_convexity}, the same conclusion holds for $\rho$-hypomonotone operator, for any $\rho>0$, in Theorem~\ref{thm:universal_hypomonotone}.
    That is, monotonicity violation is unbounded under the assumption of~Theorem~\ref{thm:universal_convexity}. 
    However, assuming $L$-Lipschitz continuous inputs caps this violation whenever $F$ is $\Lambda_F$-Lipschitz continuous; that is, $\F$ is $\Lambda_F L$-hypomonotone and $\Lambda_F L$-Lipschitz continuous. 
    Together, these additional results provide a complete picture of universal monotonicity preservation and the extent of its violation.
\end{remark}
Furthermore, our analysis generalizes to the setting of co-coercive operators.
\begin{theorem}\label{cor:nococoercivity}
    Let $d>1$, $L > 0$, and consider~\eqref{eq:algo-update}. 
    $\F$ is co-coercive for all gradients $\{g_i(\cdot)\}_{i\in[n]}$ induced by convex and $L$-Lipschitz smooth functions if and only if $F$ is a positively affine aggregation.
    \looseness=-1
\end{theorem}
\begin{proof}
    In the proof of~Theorem~\ref{thm:universal_convexity}, we only used convex and Lipschitz-smooth loss functions.
    Specifically, in~\ref{item:def_R_i} we set $M \vcentcolon=  \frac{\|u-v\|^2_2}{L \langle u-v, \delta \rangle} > 0$, $\delta_M = M \delta$ and $H \vcentcolon= \frac{(u-v) (u-v)^\intercal}{M \langle u-v, \delta \rangle} = L \frac{(u-v) (u-v)^\intercal}{\|u-v\|^2_2}$. 
    With this choice, $\| H \|_{op} \preceq L I$. 
    The remainder of the argument proceeds unchanged.
\end{proof}
We note that Theorems~\ref{thm:universal_convexity}~and~\ref{cor:nococoercivity} hold regardless of whether the $g_i$'s are stochastic (e.g., in SGD). 
In algorithms such as randomized coordinate descent (RCD), projecting onto a subset of coordinates breaks monotonicity (or co-coercivity) at the level of individual iterates. 
However, because the expectation is taken over a fixed sampling distribution, the expected operator $\E\F$ remains affine. 
Monotonicity (or co-coercivity) is thus recovered in expectation, which is typically sufficient for obtaining in-expectation guarantees.

Importantly, Theorem~\ref{thm:universal_convexity} does not rule out the possibility that a non-affine aggregation rule may preserve monotonicity for a \emph{restricted class of convex functions}.
In fact, our proof is constructive and yields actionable insights for identifying non-affine aggregations that preserve monotonicity under structural restrictions on the input gradients. 
We further develop this important implication of our main result in Section~\ref{sec:positive-result-tm}, after first discussing the general consequences of Theorems~\ref{thm:universal_convexity} and~\ref{cor:nococoercivity} for learning algorithms in the following Section~\ref{sec:consequences}.

\section{Consequences for Learning Algorithms}\label{sec:consequences}

Theorems~\ref{thm:universal_convexity}~and~\ref{cor:nococoercivity} trigger several strong implications for learning algorithms, as
one cannot introduce non-affine aggregation rules without distorting the convex geometry of the problem. 
In this section, we show that violating monotonicity affects convergence dynamics (cf.~Section~\ref{sec:convergenceimplications}) and degrades algorithmic stability (cf.~Section~\ref{sec:stabilityimplications}).

\subsection{Consequences for the Last-Iterate Convergence of Learning Algorithms}\label{sec:convergenceimplications}
When aggregation rules fail to preserve monotonicity, the resulting optimization dynamics depart from the classical convex setting. 
In particular, the loss of monotonicity precludes guarantees of last-iterate convergence (i.e, finding a fixed-point of $\F$).
In the following, we formalize this perspective.

Inspired by the proof of Theorem~\ref{thm:universal_convexity}, we formalize a property, shared by a broad class of practical aggregation rules, that precludes last-iterate convergence.
We say $F: (\mathbb{R}^d)^n \to \mathbb{R}^d$ is not positively affine at $g_1, \ldots, g_n \in \R^d$, if there exist $i \in [n]$ and $\delta \in \R^d \setminus \{0\}$ such that
\begin{equation}\label{eq:restrictednonaffine}
    F\left( g_1, \ldots, g_i + \delta, \ldots, g_n \right) - F\left( g_1, \ldots, g_i, \ldots,  g_n \right)  \not \in \operatorname{cone}(\delta).
\end{equation}
Any non-affine aggregation must exhibit at least one such configuration; otherwise, it would reduce to a positively affine mapping. 
When this behavior occurs at a \emph{stationary configuration}, that is, a configuration satisfying $F(g_1, \ldots, g_n) = 0$, we show that it prevents last-iterate convergence. 
Crucially, most practical non-affine aggregations feature this non-positive affineness for at least one stationary configuration. 
This includes point-wise clipping~\citep{Abadi_2016} and AdaGrad~\citep{duchi2011adaptive}---for which symmetric, non-zero inputs can be constructed to cancel each other out while still triggering the rule's non-linearity---as well as translation-equivariant robust aggregation schemes~\citep{pmlr-v202-allouah23a, Diakonikolas_Kane_2023}, where any non-affine configuration can be shifted to the origin.
Note that if an aggregation rule is positively affine, the condition~\eqref{eq:restrictednonaffine} cannot be satisfied. 
Moreover, a non-affine aggregation rule can still violate~\eqref{eq:restrictednonaffine} for all stationary configuration, e.g., global clipping~\citep[Theorem 2.3]{koloskova2023revisitingclipping}. 
In which case, the algorithm may still converge monotonically to a stationary point. 
The proof is deferred to Appendix~\ref{app:deferred-section-3-1}.
% \looseness=-1
\begin{theorem}\label{cor:cone-instability}
    Let $d>1$, and $F: (\mathbb{R}^d)^n \to \mathbb{R}^d$ be such that there exists $(g_1,\ldots,g_n)\in(\mathbb{R}^d)^n$ for which $F(g_1,\ldots,g_n)=0$ and at which $F$ is not positively affine.
    Then, there exist convex functions,
     
    \noindent
    \begin{minipage}[c]{0.79\textwidth}
        $\theta_{eq} \in\{z \mid \mathcal{F}(z)=0\} \neq \emptyset$, and an open set $\mathcal{C} \subset \mathbb{S}^{d-1} \vcentcolon= \{u\in\R^d; \|u\|_2 = 1\}$ such that for any 
        $\theta_0 \in \mathcal{K} \coloneqq \{ \theta_{eq} + \alpha c \mid c \in \mathcal{C}, \alpha > 0 \},$
        the update step~\eqref{eq:algo-update}, for any $\gamma > 0$, strictly increases the distance to the equilibrium
        \[
            \| \theta_{1} - \theta_{eq} \|_2 > \| \theta_0 - \theta_{eq} \|_2.
        \]
    \end{minipage}%
    \hfill
    \begin{minipage}[c]{0.21\textwidth}
        %\centering
        \
        \begin{tikzpicture}[scale=0.87, >=stealth]
            % definitions
            \coordinate (O) at (0,0);
            \def\coneangle{30}% Increased angle for a wider cone
            \def\radius{2.3}
            % 1. Cone
            \fill[red!30] (O) -- (\coneangle:\radius) arc (\coneangle:-\coneangle:\radius) -- cycle;
            \draw[red, dashed, thick] (O) -- (\coneangle:\radius);
            \draw[red, dotted, thick] (\coneangle:\radius) -- (\coneangle:\radius+0.25);
            \draw[red, dashed, thick] (O) -- (-\coneangle:\radius);
            \draw[red, dotted, thick] (-\coneangle:\radius) -- (-\coneangle:\radius+0.25);
            % 2. Equilibrium
            \node[circle, fill=black, inner sep=1.2pt, label={[font=\footnotesize]left:$\theta_{eq}$}] at (O) {};
            % 3. Points & Vector
            \coordinate (T0) at (4:1.2);
            \coordinate (T1) at (14:2);
            % Update Arrow
            \draw[->, blue, thick] (T0) -- (T1) node[midway, above, sloped, scale=0.85, font=\footnotesize] {$I-\gamma\F$};
            % Dots
            \node[circle, fill=black, inner sep=1.2pt, label={[font=\footnotesize, black]below:$\ \theta_0$}] at (T0) {};
            \node[circle, fill=black, inner sep=1.2pt, label={[font=\footnotesize, black]below:$\ \theta_1$}] at (T1) {};
            % Cap C
            \draw[thick, dashed, red] (\coneangle:0.5) arc (\coneangle:-\coneangle:0.5);
            \node[red, right, scale=1] at (0:0.5) {$\mathcal{C}$};
            % Label K inside the cone area
            \node[red, font=\small] at (-20:1.9) {$\mathcal{K}$};
            % Alpha: Starts at offset (235 degrees, 0.2cm) and draws parallel to the cone edge
        \end{tikzpicture}
    \end{minipage}
\end{theorem}
Theorems~\ref{thm:universal_convexity} and~\ref{cor:cone-instability} elucidate why existing literature often settles for the weaker averaged iterates convergence guarantees rather than the stronger last-iterate convergence of the update~\eqref{eq:algo-update} when using non-affine aggregations in convex optimization, even under deterministic settings.
This barrier stems from the non-positively affine nature of existing stationary configurations.
While the trivial $(0, \ldots, 0)$ stationary configuration is positively affine for most aggregation rules, practical optimization rarely converges to this ideal state due to data heterogeneity or gradient variance. 
% Consequently, to circumvent this limitation and guarantee strong convergence guarantees, existing literature frequently relies on assumptions that confine the optimization trajectory within the basin of attraction of a positively affine stationary configuration like $(0, \ldots, 0)$.\batodo{What do you mean by "relies on assumptions that confine the optimization trajectory within the basin of attraction of a positively affine stationary configuration like $(0, \ldots, 0)$"? I find this claim mathematically unclear \tb{This is explained in the next sentence. The most prevalent one is the strong growth condition: the only stationary configuration is $(0,...,0)$. Overall, if your only stationary configuration is positively affine, or you start nearby one (mathematically: in its bassin of attraction), then you should be able to prove last-iterate convergence}}
% This includes imposing the strong growth condition \citep[or interpolation regime,][and reference therein]{schmidt2013fast, pmlr-v80-ma18a, liu2024revisiting, attia2026fast} or introducing other structural constraints~\citep[see, e.g.,][Lemma 5.1; \citealp{koloskova2023revisitingclipping}, Theorem 2.3; \citealp{preobrazhenskaia2026last-iterate-adagrad}, Corollary 4.3]{song2021evading}.
% Such additional assumptions have also proven useful in the study of other pathological phenomena, including limit cycles and degraded or improved convergence rates, even when considering only averaged iterates convergence \citep[see, e.g.,][Theorem IV; \citealp{allouah2024adaptiveclipping}, Theorem 5.2]{karimireddy2022byzantinerobust}.
Consequently, to circumvent this limitation and guarantee strong convergence guarantees, existing literature frequently relies on assumptions to avoid non-positively affine stationary configurations.
This includes imposing limiting structural constraints on the problem or aggregation rule~\citep[see, e.g.,][Lemma 5.1; \citealp{koloskova2023revisitingclipping}, Theorem 2.3; \citealp{preobrazhenskaia2026last-iterate-adagrad}, Corollary 4.3]{song2021evading}.
Furthermore, the strong growth condition \citep[or interpolation regime,][and reference therein]{schmidt2013fast, pmlr-v80-ma18a, liu2024revisiting, attia2026fast} has shown particular efficacy for randomized algorithms.
Such additional assumptions have also proven useful in the study of other pathological phenomena, including limit cycles and degraded or improved convergence rates, even when considering only averaged iterates convergence \citep[see, e.g.,][Theorem IV; \citealp{allouah2024adaptiveclipping}, Theorem 5.2]{karimireddy2022byzantinerobust}.
Our theory therefore offers a concrete motivation for studying these pathological aggregation-induced convergence phenomena in the future.
% Our framework therefore offers a new perspective on the necessity of these assumptions and provides a lens through which to study related aggregation-induced phenomena, a direction that we leave for future work.
We defer a broader discussion of the literature to Section~\ref{related-work}.

\subsection{Consequences for the Algorithmic Stability of Learning Algorithms}\label{sec:stabilityimplications}

From a learning theory perspective, the objective is not only to find or approximate a stationary point of $\F$, i.e., a minimizer of $\widehat{R}$ in most settings, but also to control the generalization error, i.e., $R(\theta_T) - \widehat{R}(\theta_T)$, of the resulting model.
Algorithmic stability has emerged as a central framework for characterizing generalization, particularly in convex optimization. %~\citep{Bousquet2002StabilityAG,JMLR:v11:shalev-shwartz10a}. 
Most existing analysis focus on either uniform argument stability~\citep{hardt2016train,bassily2019private,pmlr-v180-zhang22b} or on-average argument stability~\citep{lei2020fine,schliserman2022stability,pmlr-v235-le-bars24a}.
Formally a learning algorithm $\A$ is said $\varepsilon$-uniformly argument stable if
\begin{equation}\label{eq:uas}
    \operatorname{Stab}(\mc{A}) \vcentcolon= \sup_{\mc{S} \sim \mc{S}'} \|\mc{A}(\mc{S}) - \mc{A}(\mc{S}')\|_2 \leq \varepsilon,
\end{equation}
where $\mc{S} \sim \mc{S}'$ denotes a pair of neighboring datasets differing in at most one sample.
%, and the expectation is over the internal randomness of $\mc{A}$.
% \tbtodo{
%     \textbf{I decided to remove the expectation here for in-expectation guarantees for randomized algorithms (but still relevant for h.p. generalization error).} 
%     There are few arguments:
%     (1) We never dicuss randomness in this section and we arguably only derive deterministic lowerbounds, hence keeping this expectation might be misleading. 
%     (2) Moreover, I discuss randomization in Section 5 for stability as a perspective (This would even be very cool to cite the non-monotonic effect). 
%     (3) We might place it here but I fear it would break the flow of the section. 
% }
The generalization error is upper bounded by the stability parameter $\varepsilon$, and can even scale proportionally with it~\citep[Lemma 4.1]{boudou2025generalization}.
The stability analysis entails bounding the divergence between two parameter trajectories produced by~\eqref{eq:algo-update} on datasets that differ by only one sample.
Classical stability bounds rely on the update~\eqref{eq:algo-update} being non-expansive~\citep[Definition 2.3]{hardt2016train}, together with a notion of sensitivity capturing the effect of the sample perturbation. 
We denote $B(0,C) \vcentcolon= \{ x\in\R^d; \|x\|_2 < C \}$ and $\widebar B(0, C) \vcentcolon= \{ x\in\R^d; \|x\|_2 \leq C \}$.
\begin{definition}\label{def:local-sensitivity}
    Let $F: (\widebar B(0, C))^n \to \mathbb{R}^d$ be an aggregation rule. 
    The sensitivity of $F$ is defined, for an allowed worst-case input perturbation $\rho$, as
    \begin{equation}
        S_{n, C}(F, \rho) \vcentcolon= \max_{i\in[n]}\sup_{\substack{g \in \widebar B(0, C)^n,\\ g'_i \in \widebar B(0, C),\ \|g_i - g'_i\|_2 \leq \rho}} 
        \| F(\dots, g_i, \dots) - F(\dots, g'_i, \dots) \|_2.
    \end{equation}
\end{definition}
% The parameter $\rho$ in Definition~\ref{def:local-sensitivity} formalizes the worst-case input perturbation. 
%for the aggregation rule $F$ when a single data point is exchanged in the underlying training set $\mathcal{S}$
% Typically, $\rho$ depends on the dataset size, and the data partitioning scheme\footnote{
%     %Typically, $\rho$ depends on the dataset size $N$, and the data partitioning scheme. 
%     Let $N$ be the dataset size.
%     For instance, (i) in centralized empirical risk minimization ($n=N$), replacing one sample yields $\rho(N) = 2C$;
%     (ii) in federated learning ($n$ workers, $m$ local samples each, $nm=N$), it gives $\rho(N) = \frac{2C}{m}$.
% }.\autodo{I am not sure I understand this footnote; why $\rho$ depends on $m$ in federated but not on $n$ in centralized case?}

In this section, we formalize that the loss of co-coercivity induced by non-affineness (Theorem~\ref{cor:nococoercivity}) leads to a strictly weaker stability guarantee than those achievable with positively affine aggregation rules. For proving this rigorously, we first show that the algorithmic stability of~\eqref{eq:algo-update} must account for the expansivity.
The proofs of this section are deferred to Appendix~\ref{app:deferred-section-3-2}.

% \tbtodo{\textbf{Mild conditions:}
%     (1) $\operatorname{NonAff}$ verifies Assumption~\ref{assump:one-step-lb-stab-continuous}, and $S_{n,C}(\operatorname{NonAff},\rho)=S_{n,C}(\operatorname{Aff},\rho)$. 
%     (2) $\operatorname{NonAff}$ is discontinuous (more precisely, its maximal discontinuity jump $\Delta = \sup_{p\in B(0, C)}\lim_{\delta \to 0^+} \sup_{x,y \in B(p,\delta)} \|f_{g_{-i}}(x) - f_{g_{-i}}(y)\|_2$ typically scale poorly with problem parameters relative to the sensitivity term, e.g., does not depend on the total dataset size).
% }
\begin{theorem}\label{cor:unavoidable-instability}
    Assume the setting of Theorem~\ref{cor:nococoercivity}, with gradients norm bounded by $C>0$ and neighboring datasets inducing gradient perturbations of at most $\rho > 0$.
    Let $\mathcal{A}_F$ denote the algorithm that outputs $\theta_T$ after $T \geq 1$ iterations of~\eqref{eq:algo-update} with aggregation rule $F$ and step size $\gamma \leq 1/L$.
    Then,
    \begin{equation}\label{eq:USparameter}
        \textstyle \operatorname{Stab}(\mathcal{A}_{F}) \leq \gamma T S_{n, C}(F, \rho) + \mathcal{I}(F),
    \end{equation}
    where $\mathcal{I}(F) \vcentcolon= \sup_{\mathcal{S} \sim \mathcal{S}'}\sum_{t=1}^{T-1} \| \theta_t - \gamma \mathcal{F}(\theta_t) - (\theta_t' - \gamma \mathcal{F}(\theta_t')) \|_2 - \| \theta_t - \theta_t' \|_2$.
    Moreover, there exist convex, $L$-smooth, and $C$-Lipschitz loss functions such that
    \begin{equation}\label{eq:unified-lower-bound}
        \textstyle \operatorname{Stab}(\mathcal{A}_{F}) \geq \frac{1}{3} \left( \gamma T S_{n, C}(F, \rho) + \mathcal{I}(F) \right).
    \end{equation}
\end{theorem}
\noindent 
In this tight characterization of algorithmic stability, $\mathcal{I}(F)$ quantifies the cumulative effect of the (non-)expansivity of the update~\eqref{eq:algo-update} when using $F$.
Non-expansivity is guaranteed by the co-coercivity of $\F$, i.e., for positively affine aggregation rules $\mathcal{I}(F) \leq 0$, as long as $\gamma \leq \frac{1}{L}$~\citep[Lemma 3.6]{hardt2016train}.
% , we have by~\citet[Lemma 3.6]{hardt2016train}  $\mathcal{I}(F) \leq 0$. 
Moreover, for any positively affine aggregation rule, there exists loss functions such that $\mathcal{I}(F)=0$~\citep[cf.][Theorem 1, or the proof of Theorem~\ref{cor:unavoidable-instability}]{pmlr-v180-zhang22b}.  
However, as shown in Theorem~\ref{cor:nococoercivity}, non-affine aggregation rules violate co-coercivity, resulting in an expansive update.
\begin{corollary}\label{cor:expansivity}
    Let $d>1$, $L>0$. Update~\eqref{eq:algo-update} is non-expansive for all gradients $\{g_i(\cdot)\}_{i\in[n]}$ induced by convex and $L$-Lipschitz smooth functions if and only if $F$ is a positively affine aggregation.
\end{corollary}
\textbf{Proof}\;
    By~Theorem~\ref{cor:nococoercivity}, there exist parameters $\theta, \theta'$ that violate co-coercivity~\eqref{eq:operator-cocoercivity}. 
    Thus, for any $\gamma > 0$
    \begin{align*}
        \| \theta - \gamma \F(\theta) - (\theta' - \gamma \F(\theta')) \|_2^2 
        &= \| \theta - \theta' \|_2^2
        + \gamma^2 \| \F(\theta) - \F(\theta') \|_2^2
        - 2 \gamma \langle \theta - \theta', \F(\theta) - \F(\theta') \rangle \\
        &> \| \theta - \theta' \|_2^2. \tag*{\BlackBox}
    \end{align*}
Corollary~\ref{cor:expansivity} already invalidates standard uniform or on-average argument stability proofs based on non-expansiveness. 
Indeed, for non-affine aggregation rules, Corollary~\ref{cor:expansivity} precludes a direct conclusion that $\mathcal{I}(F) \leq 0$. 
Nevertheless, it does not directly imply $\mathcal{I}(F) > 0$. 
% Below, we formally prove that, under mild conditions, we must have $\mathcal{I}(F) > 0$ for non-affine aggregation rules, proving that non-affine aggregations incur an additional stability penalty.
To bridge this gap, we derive a one-step lower bound 
% under mild conditions 
showing that $\mathcal{I}(F) > 0$ for continuous non-affine aggregation rules. 
This implies that practical non-affine aggregations necessarily incur an additional stability penalty.
While the sensitivity term $S_{n, C}(F, \rho)$ can decay with more data, $\mathcal{I}(F)$ remains a bottleneck which can 
% be independent of the total dataset size, or 
scale poorly with hyperparameters such as $T$ (cf.\ discussion at the end of this section).

\begin{assumption}\label{assump:one-step-lb-stab-continuous}
    There exist $i \in [n]$, $g_{-i} \in \widebar B(0, C)^{n-1}$, and $p \in \widebar B(0, C)$ where $f_{g_{-i}}(\cdot) \coloneqq F\left( \ldots, g_{i-1}, \cdot, g_{i+1}, \ldots \right)$ 
    is continuous in a neighborhood of $p$ with realizable non-affineness. Specifically, for any $r > 0$ there exists $q, b_0 \in B(p, r) \cap \widebar B(0, C)$ such that $d$ satisfying $f_{g_{-i}}(b_0) - f_{g_{-i}}(p) = d$ is a separating vector for $p-q$ and $f_{g_{-i}}(p) - f_{g_{-i}}(q)$.\footnote{A vector $d$ is said to be a separating vector for two vectors $u$ and $v$ if $\langle u, d \rangle > 0 $ and $\langle v, d \rangle < 0$. Basically, $d$ represents the normal of a hyperplane separating $u$ and $v$.}
    % and (ii)    
    % That is, there exists $d \in \R^d \setminus \{0\}$ satisfying, for any $r > 0$, 
    % \begin{align*}
    %     \textit{Non-affinity: } & \exists q \in B(p, r) \cap \widebar B(0, C) \text{ such that } \langle p - q, d \rangle > 0 \text{ and } \langle f_{g_{-i}}(p) - f_{g_{-i}}(q), d \rangle < 0 \\
    %     \textit{Realizability: } & \exists b_0 \in B(p, r) \cap \widebar B(0, C) \text{ and } \epsilon > 0 \text{ such that } f_{g_{-i}}(b_0) - f_{g_{-i}}(p) = \epsilon d.
    % \end{align*}
\end{assumption} 
Assumption~\ref{assump:one-step-lb-stab-continuous} basically requires the aggregation rule to simultaneously exhibit non-affineness and have a sufficiently rich image so as to trigger expansivity between near-identical trajectories of the update~\eqref{eq:algo-update}.
Ultimately, the key message is that violations of co-coercivity inevitably give rise to an additional term in algorithmic stability and must therefore be accounted for.
\begin{lemma}\label{thm:1-step-dsep}
    Assume the setting of Theorem~\ref{cor:unavoidable-instability} and Assumption~\ref{assump:one-step-lb-stab-continuous}, with $T=2$. 
    Then, for any $\theta_0 \in \R^d$, there exist convex, $L-$smooth and $C-$Lipschitz loss functions $\{\ell_j\}_{j \in [n]}, \ell'_i$ such that $ \mathcal{I}(F) > 0$.
    % \begin{equation*}
    %     %\|\theta_{2} - \theta'_{2}\|_2 > \sigma_{F,1} > 0,
    %     \mathcal{I}(F) > 0.
    % \end{equation*}
    % where $\sigma_{F,1} \vcentcolon= \| \theta_1 - \gamma \F(\theta_1) - (\theta_1' - \gamma \F(\theta_1')) \|_2  - \| \theta_1 - \theta_1' \|_2$ captures the expansivity of the update.
\end{lemma}
\noindent 
The construction of Lemma~\ref{thm:1-step-dsep}, deferred in Appendix~\ref{app:deferred-section-3-2}, is necessary to prove degraded stability for continuous aggregation rules, as they typically have a similar sensitivity than their affine counterparts.
In contrast, this construction is not necessary for discontinuous aggregation rules, which degrade stability even without violating co-coercivity, although such violations can still occur.
Indeed, discontinuities can prevent the sensitivity term $S_{n, C}(F, \rho)$ from decaying as the amount of data increases, yielding a bottleneck that dominates the sensitivity of their positively affine counterparts.
Such discontinuities naturally arise in robust aggregation rules~\citep{diakonikolas2019robust,blanchard-ml-with-adversaries,pmlr-v206-allouah23a,pmlr-v202-allouah23a,allouah2025towards}.
% \tb{(e.g., robust aggregation rules such as the $\operatorname{filter}$ algorithm~\citealt{diakonikolas2019robust}, $\operatorname{SMEA}$~\citealt{pmlr-v202-allouah23a}, $\operatorname{CAF}$~\citealt{allouah2025towards}, $\operatorname{Krum}$~\citealt{blanchard-ml-with-adversaries} or $\operatorname{NNM}$~\citealt{pmlr-v206-allouah23a})}.
Together, our framework encompasses practical non-affine aggregation rules and accounts for degraded algorithmic stability.
% We now state the following corollary, which follows directly from the preceding discussion, Theorem~\ref{cor:unavoidable-instability}, and Lemma~\ref{thm:1-step-dsep}.
\begin{corollary}\label{th:greaterinstability}
    Let $\operatorname{NonAff}$ and $\operatorname{Aff}$ be non-affine and positively affine aggregation rules.
    Then, 
    \begin{equation*}
        \operatorname{Stab}(\mathcal{A}_{\operatorname{NonAff}}) > \operatorname{Stab}(\mathcal{A}_{\operatorname{Aff}}),
    \end{equation*} 
    provided that (1) $\operatorname{NonAff}$ verifies Assumption~\ref{assump:one-step-lb-stab-continuous}, while having comparable baseline sensitivity (e.g., $S_{n,C} (\operatorname{NonAff},\rho) = S_{n,C}(\operatorname{Aff},\rho)$); 
    or (2)  $\operatorname{NonAff}$ is discontinuous, with a discontinuity jump\footnotemark\ outscaling the baseline sensitivity $S_{n,C}(\operatorname{Aff},\rho)$ (e.g., it does not decay with the dataset size).
\end{corollary}
% \vspace{0.5\topsep}
% \end{mdframed}
% \footnotetext{With comparable baseline sensitivity, e.g., $S_{n,C}(\operatorname{NonAff},\rho)=S_{n,C}(\operatorname{Aff},\rho)$.}
\footnotetext{More precisely, its maximal discontinuity jump 
    $\Delta = \max_{i \in [n]} \sup_{g \in \widebar B(0, C)^{n}} \lim_{\delta \to 0^+} \sup_{x,y \in B(g_i,\delta)} \|f_{g_{-i}}(x) - f_{g_{-i}}(y)\|_2$
    %$\Delta = \sup_{i \in [n], g_{-i} \in \widebar B(0, C)^{n-1}, p\in B(0, C)}\lim_{\delta \to 0^+} \sup_{x,y \in B(p,\delta)} \|f_{g_{-i}}(x) - f_{g_{-i}}(y)\|_2$ 
    scales poorly with problem parameters relative to the baseline sensitivity $S_{n,C}(\operatorname{Aff},\rho)$. %(e.g., it does not decay with the total dataset size).
    %This behavior is standard in practice. 
    %For instance, discontinuous rules broadly encompass methods in robust learning (e.g., the $\operatorname{filter}$ algorithm~\citealt{diakonikolas2019robust}, $\operatorname{SMEA}$~\citealt{pmlr-v202-allouah23a}, $\operatorname{CAF}$~\citealt{allouah2025towards}, $\operatorname{Krum}$~\citealt{blanchard-ml-with-adversaries} or $\operatorname{NNM}$~\citealt{pmlr-v206-allouah23a})
}
% The stated conditions are mild and standard in practice. 
% For instance, discontinuous rules broadly encompass methods in robust learning for which discontinuity induces the sensitivity term $S_{n, C}(F, \rho)$ to not decay nicely with more data.
%, and hence remains a bottleneck independent of the total dataset size 
% (e.g., the $\operatorname{filter}$ algorithm~\citealt{diakonikolas2019robust}, $\operatorname{SMEA}$~\citealt{pmlr-v202-allouah23a}, $\operatorname{CAF}$~\citealt{allouah2025towards}, $\operatorname{Krum}$~\citealt{blanchard-ml-with-adversaries} or $\operatorname{NNM}$~\citealt{pmlr-v206-allouah23a}).
% Whereas continuous non-affine rules represent the complement set of non-affine aggregation rules, and typically have a similar sensitivity than their affine counterparts.
Theorem~\ref{cor:unavoidable-instability}, Lemma~\ref{thm:1-step-dsep} and Corollary~\ref{th:greaterinstability} collectively show that the update rule~\eqref{eq:algo-update}, when equipped with a practical non-affine aggregation rule, necessarily incurs an additional instability term. 
Quantifying the order of this penalty is problem-specific and is left to future work tailored to particular applications.
This penalty is consistent with and explains similar instability overheads reported in the literature.
In robust distributed learning, \citet[Theorem 3.4]{boudou2025generalization} show an unavoidable stability penalty which outscales the affine baseline.
Similarly, for (S)GD with momentum,~\citet[Theorem 3]{JMLR:v25:22-0068stabilitymomentum} identify an additional instability term that grows unfavorably with the number of steps, consistent with lower bounds for accelerated methods \citep{attia2021algorithmic}.
Finally, \citet{bassily2019private} establish a tight instability overhead when non-smoothness precludes co-coercivity, even under positively affine aggregation.
Because this overhead accumulates over iterations, it likely contributes to degraded generalization performance.
We postpone a broader review of prior related work in Section~\ref{related-work}.

\section{Designing Non-Affine Aggregations for Restricted Inputs---A Positive Result}\label{sec:positive-result-tm} 
The impossibility result of Section~\ref{sec:impossinility} relies on the construction of convex loss functions with symmetric positive semidefinite rank-1 Hessians. 
When such functions are excluded from the set of possible convex loss functions, monotonicity-preserving non-affine rules may become possible.
In Appendix~\ref{sec:alternative-monotonone}, we provide an alternative proof of Theorem~\ref{thm:universal_convexity}  that replaces the contradiction-based argument with an algebraic characterization (Lemma~\ref{lem:psd-crux}), which gives further actionable insights for identifying non-affine rules that preserve monotonicity for restricted classes of convex functions.
Specifically, once the class of admissible convex functions is fixed (and thus the associated Hessian family $\mathcal{H}$), solving the algebraic condition from Lemma~\ref{lem:psd-crux} yields necessary conditions that an aggregation rule must satisfy to preserve monotonicity.
That is, the Jacobian of $\F$ must be of the following form.
\begin{mdframed}[backgroundcolor=gray!5]
    \vspace{0.5\topsep}
    \textbf{Algebraic condition from Lemma~\ref{lem:psd-crux}.}
    \emph{Given a matrix set $\mathcal{H} \subset \R^{d \times d}$, find $D$ such that $\forall H \in \mathcal{H}$, $DH \succeq 0$.
    That is, $\forall H \in \mathcal{H}, \forall v \in \R^d, v^\intercal DH v = v^\intercal \left(\frac{DH + (DH)^\intercal}{2}\right) v \geq 0$.}
    \vspace{0.5\topsep}
\end{mdframed}
This leads to the following general recipe: (1) Check whether the function class admits arbitrary rank-1 p.s.d. Hessians; (2) If not, solve the algebraic condition in Lemma~\ref{lem:psd-crux} under the restricted set of functions; (3) Use the resulting characterization to derive constraints on aggregation rules that preserve monotonicity. 
% \batodo{Before reading the illutration, this last step is not understandable}
 
As an illustration, consider problems where gradients are coordinate-wise separable, such as standard least-squares objectives. 
These functions take the form $\ell = \sum_{k=1}^d \ell_k$, where each $\ell_k$ is a univariate function acting solely on the $k$-th coordinate, $k\in\{1, \ldots, d\}$.
In this case, admissible Hessians are diagonal, which excludes arbitrary rank-1 p.s.d. matrices. 
% Solving Lemma~\ref{lem:psd-crux} then shows that any monotonicity-preserving aggregation must operate coordinate-wise. 
Solving the algebraic condition from Lemma~\ref{lem:psd-crux} for this restricted Hessian class shows that the admissible aggregation Jacobians are diagonal. 
Therefore, any monotonicity-preserving aggregation must operate coordinate-wise.
Consistent with this prediction, we show below that Coordinate-Wise Trimmed Mean ($\operatorname{CWTM}$) preserves monotonicity under coordinate-wise separability. 
$\mathrm{CWTM}$ is particularly useful in robust learning as an instance of $(f, \kappa)$-robust aggregation with an optimal asymptotic robustness coefficient~\citep[Definition 2 \& Proposition 2]{pmlr-v206-allouah23a}.

\begin{definition}\label{cwtm-definition}
    Given $n \in \mb{N}$, $f < n/2$, and real values $x_1, \ldots, x_n$, let $\sigma$ be a permutation such that $x_{\sigma(1)} \leq \cdots \leq x_{\sigma(n)}$, with $S_x=\{i \in [n]; f+1 \leq \sigma(i) \leq n-f\}$, then
    \[
        \operatorname{TM}(x_1, \ldots, x_n) 
        = \frac{1}{n-2f} \sum_{i=f+1}^{n-f} x_{\sigma(i)} 
        = \frac{1}{n-2f} \sum_{i \in S_x} x_i
    \]
    Denote ${\left[ \cdot \right]}_k$ the $k-$th coordinate.
    We define $\mathrm{CWTM}$ for vectors $x_1, \ldots, x_n \in \mb{R}^d$ as the coordinate-wise application of $\mathrm{TM}$.
    That is $\operatorname{CWTM}(x_1, \ldots, x_n)$ is the vector whose $k-th$ coordinate is
    \[
        {[\operatorname{CWTM}(x_1, \ldots, x_n)]}_k = \operatorname{TM}({[x_1]}_k, \ldots, {[x_n]}_k)
    \]
\end{definition}
\begin{lemma}\label{lemma-cwtm-1d-nonexpansive}
    For input gradients induced by coordinate-wise separable convex functions, $\operatorname{CWTM}$ is monotone. 
    If, in addition, the functions are $L$-smooth, then $\operatorname{CWTM}$ is $\tfrac{1}{L}$-co-coercive.
    % Assume coordinate-wise separable convex and $L-$smooth functions, then $\operatorname{CWTM}$ is $\frac{1}{L}$-co-coercive.
    % Let $G^{\mathrm{CWTM}}_{\gamma}$ the update~\eqref{eq:algo-update} when $F=\operatorname{CWTM}$. 
    % Assume that for all $z \in \mc{Z}$, the loss function $\ell(\cdot;z)$ is convex and $L$-smooth. 
    % If $\gamma \leq \frac{2}{L}$ then $G^{\mathrm{CWTM}}_{\gamma}$ is $1$-expansive. 
\end{lemma}
\begin{proof}
% \textbf{Proof}\;
    Let $\theta, \omega \in \mb{R}^d$ and $z^{(1)}, \ldots, z^{(n)} \in \mc{Z}$. 
    We denote $\ell'(\cdot)$ the derivative of $\ell(\cdot)$.
    Because the functions are coordinate-wise separable, the $d$-dimensional case follows directly from the scalar case. 
    Therefore, it suffices to prove the result for $d=1$. 

    In this case, we can leverage the co-coercivity property of the loss function.
    In fact, due to Lemma~\ref{trimmed-mean-lowerbound} (whose proof is deferred to Appendix~\ref{tm}), there exist indices $\alpha, \beta \in \{1, \ldots, n \}$ such that
    \begin{align} 
        \ell'_\beta (\theta) - \ell'_\beta (\omega) \leq \mathrm{TM}(\ell'_1 (\theta), \ldots, \ell'_n (\theta)) - \mathrm{TM}(\ell'_1 (\omega), \ldots, \ell'_n (\omega)) \leq \ell'_\alpha (\theta) - \ell'_\alpha (\omega) . \label{eqn:a-b-bnd}
    \end{align}
    Therefore, there exists $i \in \{\alpha, \beta\}$ such that
    \begin{equation}\label{eq:monotonicity-tm}
        (\theta - \omega) \left( \mathrm{TM}(\ell'_1 (\theta), \ldots, \ell'_n (\theta)) - \mathrm{TM}(\ell'_1 (\omega), \ldots, \ell'_n (\omega)) \right) 
        \geq (\theta - \omega) (\ell'_i (\theta) - \ell'_i (\omega)) 
        \geq 0,
    \end{equation}
    where we used the monotonicity of $\ell_i'$. %(cf. Lemma~\ref{lem:cocoercivity}). 
    Again, due to~\eqref{eqn:a-b-bnd}, there exists $j \in \{\alpha, \beta\}$ such that 
    \begin{multline*}
        \left( \mathrm{TM}(\ell'_1 (\theta), \ldots, \ell'_n (\theta)) - \mathrm{TM}(\ell'_1 (\omega), \ldots, \ell'_n (\omega)) \right)^2 \\
        \leq | \mathrm{TM}(\ell'_1 (\theta), \ldots, \ell'_n (\theta)) - \mathrm{TM}(\ell'_1 (\omega), \ldots, \ell'_n (\omega)) | | \ell'_j (\theta) - \ell'_j (\omega) | \\
        \leq L | \mathrm{TM}(\ell'_1 (\theta), \ldots, \ell'_n (\theta)) - \mathrm{TM}(\ell'_1 (\omega), \ldots, \ell'_n (\omega)) | | \theta - \omega | \\
        = L \left( \mathrm{TM}(\ell'_1 (\theta), \ldots, \ell'_n (\theta)) - \mathrm{TM}(\ell'_1 (\omega), \ldots, \ell'_n (\omega)) \right) \left( \theta - \omega \right), %\tag*{\BlackBox}
    \end{multline*}
    where we now used the smoothness of $\ell_j$, alongside~\eqref{eq:monotonicity-tm}. 
\end{proof}
Consequently, under coordinate-wise separable convex and smooth objectives, update~\eqref{eq:algo-update} with $\operatorname{CWTM}$ inherits the strong algorithmic stability guarantees of~\citet{hardt2016train}.
The update also enjoys last-iterate convergence and fast $\mathcal{O}(\frac{1}{T})$ rates.
This contrasts with the general convex and smooth setting, where, as illustrated in Remark~\ref{cwtm-2d-expansivity} and established in Corollary~\ref{cor:expansivity}, non-expansiveness of the update cannot be guaranteed.

Finally, in Appendix~\ref{tm}, we further prove that $\operatorname{CWTM}$ is Lipschitz continuous, a property that can be exploited in light of Remark~\ref{rem:violation-caped}.
Leveraging this result, in Appendix~\ref{deferred-sec-5} we derive algorithmic stability guarantees for robust distributed learning in smooth nonconvex settings that improve upon existing bounds~\cite[Appendix E]{boudou2025generalization}.

\section{Implications of Monotonicity Violations in Related Work}\label{related-work}
Below, we review the extensive literature in which our main theoretical result has an impact and how the violation of monotonicity manifests in those domains.

\paragraph{Accelerated and adaptive optimization methods.}
Adaptive methods like $\operatorname{Adam}$ \citep{kingma2017adammethodstochasticoptimization} employ non-linear scaling to accelerate training, yet they are known to fail to converge even on simple convex problems~\citep{reddi2018convergenceadam,luoadaptive,zhang2022adam}. 
Interestingly, the loss of monotonicity has been explicitly noted for the original $\operatorname{AdaGrad}$ algorithm~\citep[Example 1]{li2019convergence}.
While he authors argue that the loss of monotonicity can be handled under convexity, this is only due to their focus on a weaker notion of convergence, namely that of the (weighted) average of the iterates over all iterations~\citet[Theorem 3]{li2019convergence}.
\citet[][Theorem 2.3 and references therein]{preobrazhenskaia2026last-iterate-adagrad} show that the last iterate generated by $\operatorname{AdaGrad-Norm}$ may fail to converge, unless the hyperparameters are carefully tuned using the total number of iterations (the more iteration, the lesser the learning rate).
% This reliance on step-dependent tuning has also been shown to be crucial for achieving optimal last-iterate convergence rates in $\operatorname{(S)GD}$~\citep{doi:10.1137/24M1717762-last-iterate-cv,kornowski2026gradient-last-iterate-speed-cv}.\autodo{i would remove the previous sentence, which is out of scope}
Likewise, accelerated methods such as the heavy-ball method can converge to a limit cycle instead of the optimum, even for strongly convex objectives~\citep[Figure 6]{heavy-ball-limit-cycle}.
These methods have also been shown to be algorithmically unstable, with stability bounds that scale poorly with the number of iterations~\citep{wilson2017marginal,attia2021algorithmic,JMLR:v25:22-0068stabilitymomentum,sun2024understanding}.
Our results provide a unified explanation for these phenomena and suggest similar behavior in other adaptive schemes, such as dynamic importance sampling~\citep{zhao2015stochastic}.
While prior work by \citet{needell2014stochastic-fixed-importance-sampling} leverages importance sampling to accelerate SGD for smooth and strongly convex objectives, their use of a static sampling distribution ensures the aggregation remains affine. 
In contrast, our work provides a complementary perspective tailored to dynamic reweighting.
\looseness=-1 

\paragraph{Gradient clipping and differentially private learning.}
Gradient clipping is arguably the most ubiquitous non-affine aggregation rule.
While often treated as a heuristic to ``stabilize'' training~\citep{zhang2019clippingacceleratestraining}, our analysis indicates that clipping alters the optimization landscape by violating monotonicity.
This aligns with and generalizes recent findings on the influence of clipping \citep[and related work therein]{koloskova2023revisitingclipping,JMLR:v27:24-0637-yunwei-stability-clipping}, most notably in the context of private machine learning with $\operatorname{DP-SGD}$~\citep{Abadi_2016,papernot2021tempered,bagdasaryan2019differentialdisparateimpactonaccuracy}.
\citet{chen2020understandingclipping,pmlr-v130-qian21a} provide further geometric insights into clipping's impact on convergence.
% \autodo{I feel we could develop this part more, it is an important application area; how the "geometric insights" of Chen;Qian differ/are less nice than ours; and how the results from Koloskova relate to ours.}
Importantly, our work clarifies the limitations of recent positive results.
For instance, \citet{song2021evading} prove that for Generalized Linear Models (GLMs), $\operatorname{clipped-DP-SGD}$ minimizes a surrogate convex objective. 
We show this is an exception, not the rule: GLMs possess a specific gradient structure (constrained to be parallel to the data vector) that precludes~\ref{item:def_R_i} in our proof.
Similarly, \citet{koloskova2023revisitingclipping} rely on the collinearity between clipped and true gradients to prove convergence in deterministic settings. 
Finally, optimal DP guarantees in convex learning under Lipschitzness assumption~\citep{altschuler-dp} rely on non-expansiveness (which relies on co-coercivity), further highlighting the scope of our theory. 

\paragraph{Robust (distributed) learning.}
In robust distributed learning, a central learner aggregates gradients from multiple workers, some of which may be corrupted, while the majority are drawn from the true data distribution.
Robust aggregation rules necessarily introduce non-linearity to filter out outliers~\citep{blanchard-ml-with-adversaries,Diakonikolas_Kane_2023}. 
The oscillations and non-vanishing error observed in the literature~\citep{karimireddy2021learning,pmlr-v206-allouah23a} are typical phenomena that occur when an operator loses monotonicity, as our analysis predicts.
Moreover, a significant implication of~Theorem~\ref{cor:nococoercivity}, which settles an open problem stated by \citet{boudou2025generalization} in the negative, is that no robust aggregation rule can universally preserve monotonicity or co-coercivity.

% \autodo{I am still not completely satisfied; I would make "one not arbitrarily skewed by a fraction of inputs" more formal, as in Blanchard et al 2017 (a single Byzantine can make $F$ always select some arbitrary vector) }
\begin{corollary}\label{th:robust_against_cocoercivity}
    No robust aggregation rule, i.e.,  an aggregation rule for which a minority fraction of inputs cannot force it to output an arbitrary vector, universally preserves monotonicity of convex functions or co-coercivity of convex, smooth functions.
\end{corollary}
\begin{proof}
   By~Theorem~\ref{thm:universal_convexity}, any such rule must be affine. 
   However, affine aggregators are non-robust in the sense that a single input can arbitrarily skew the output~\cite[Lemma 1]{blanchard-ml-with-adversaries}. 
\end{proof}
Combined with our findings in Section~\ref{sec:stabilityimplications}, this further indicates that robust distributed learning inherently incurs additional instability. 
This stands in contrast to~\citet{iclr-2020-coherent-gradient}, who suggest that robust aggregation can improve on-average stability, albeit under specific, data-dependent assumptions. 
Investigating these differences presents a promising avenue for future work, potentially revealing the conditions under which robust aggregation can achieve favorable stability.

\section{Discussion}\label{sec:conclusion}

We provide a unifying perspective on why algorithms using non-affine aggregations generally fail to preserve the theoretical guarantees of their affine counterparts in convex first-order optimization.
Theorem~\ref{thm:universal_convexity} reveals that non-affine aggregation breaks monotonicity. 
Consequently, seemingly benign algorithmic designs used for robustness, privacy, or adaptivity inevitably generate non-monotone vector fields. 
These dynamics are prone to non-convergent behavior and incur additional instability, even in smooth, convex settings.
Recognizing this breakdown, we additionally identify situations where monotonicity can be preserved.

In light of our findings, a key avenue for future research is to identify the minimal structural assumptions necessary to restore restricted forms of monotonicity.
Another promising avenue for future work is studying alternative update rules beyond the scope of~\eqref{eq:algo-update}.
For simplicity of exposition, we assumed the aggregation operator $F$ to be fixed across iterations.
However, because Theorems~\ref{thm:universal_convexity},~\ref{cor:nococoercivity}~\ref{cor:cone-instability},~\ref{cor:unavoidable-instability}, Lemma~\ref{thm:1-step-dsep} and Corollaries~\ref{cor:expansivity},~\ref{th:greaterinstability} are, or can be reduced to single-step results, our framework naturally extends to time-varying operators $F_t$.
Consequently, the update rule~\eqref{eq:algo-update} captures any step involving a non-affine aggregation of the current gradients, including momentum-based aggregations, by simply evaluating a step-dependent operator at each iteration.
One compelling direction is to shift the theoretical perspective from classical minimization to equilibrium computation~\citep[variational inequality problems, cf.][]{facchinei2007finite}.
This framework is better equipped to handle the non-monotone dynamics induced by practical learning constraints, 
a capability well-demonstrated by the success of extragradient methods in saddle-point problems like generative adversarial training~\citep{gidelvariational2019,mertikopoulosoptimistic}.
Moreover, while Section~\ref{sec:stabilityimplications} settles the question of added instability for deterministic learning algorithms, existing workarounds generally fall into two categories: optimization-based dynamics or algorithmic randomization. 
The former paradigm, such as the framework by~\citet{pmlr-v80-charles18a}, however crucially relies on the convergence of~\eqref{eq:algo-update} to a global optimum, together with additional technical assumptions on the set of minimizers.
This condition is far from obvious when using non-affine aggregation, as discussed in~Section~\ref{sec:convergenceimplications}, thereby precluding the application of their technique.
The latter paradigm leverages randomization. 
For instance, connections between differential privacy and stability~\cite[Lemma 23]{wang2016privacylearning} demonstrate that any $(\varepsilon, \delta)-$differentially private algorithm~\cite[see][Definition 20]{wang2016privacylearning} achieves uniform argument stability. 
Specifically, over a parameter space with diameter $D$, such an algorithm is $D(e^\varepsilon - 1 + \delta)$-uniformly argument stable. 
This offers an alternative route to stability without co-coercivity, where randomization acts as a stabilizing force.
We leave for future work the question of how these two approaches---based respectively on non-expansivity and on statistical indistinguishability---can be unified or compared.

\acks{
  The work of Thomas Boudou and Aurélien Bellet is supported by grant ANR 22-PECY-0002 IPOP (Interdisciplinary Project on Privacy) project of the Cybersecurity PEPR under the France 2030 program.\
}

\appendix
% \crefalias{section}{appendix} % uncomment if you are using cleveref

\section{Regularity Assumptions}\label{loss-regularities}

To analyze iterative optimization algorithms, we typically rely on regularity assumptions on the loss function.
\begin{definition}[Regularity assumptions]
    Let $\ell: \Theta \to \mb{R}$ differentiable, $\Theta \subset \mb{R}^d$ a convex set. 
    \begin{itemize}       
        \item[\textit{(i)}] $\ell$ is $C-$Lipschitz-continuous (or $C-$Lipschitz) if there exists $C>0$ such that,
        \[
            \forall u, v \in \Theta,\quad |\ell(u)-\ell(v)| \leq C \|u-v\|_2.
        \] 
        This property is equivalent to the norm of the gradient of $\ell$ being uniformly bounded by $C$.
        
        \item[\textit{(ii)}] $\ell$ is $L-$Lipschitz-smooth (or $L-$smooth) if its gradient is $L-$Lipschitz. 
        This is equivalent to the following smoothness inequality
        \begin{equation}\label{smoothness-ineq2}
            \forall u, v \in \Theta,\quad \ell(u) \leq \ell(v) + \langle \nabla \ell(v), u-v \rangle + \frac{L}{2} \| u-v \|_2^2.
        \end{equation}

        \item[\textit{(iii)}] $\ell$ is convex if and only if,
        \begin{equation}\label{eq:convex-ineq}
            \forall u, v \in \Theta,\quad \langle \nabla \ell (u) - \nabla \ell (v), u-v\rangle \geq 0.
        \end{equation}
        
        Moreover it is $\mu-$strongly convex if there exists $\mu > 0$ such that,
        \begin{equation}\label{eq:strg-convex-ineq}
            \forall u, v \in \Theta,\quad \ell(u) \geq \ell(v) + \langle \nabla \ell(v), u-v \rangle + \frac{\mu}{2} \| u-v \|_2^2.
        \end{equation}
        We note that for a strongly convex function to have bounded gradients, it must be defined on a convex compact set, which then implies boundedness of both the loss and the gradient. 
        To ensure this, we can either penalize the problem or restrict the parameter domain and apply an Euclidean projection at every step. Throughout the paper, we will tacitly assume this when referring to strongly convex functions, 
        which does not limit the applicability of our analysis since the projection does not increase the distance between projected points.
    \end{itemize}
\end{definition}
Notably, convex and $L$-smooth functions satisfy an important property known as co-coercivity~\citep[e.g.,][Proposition 5.4]{bach-ltfp}. 
\begin{lemma}[Co-coercivity]\label{lem:cocoercivity}
    Let $\ell: \Theta \to \mb{R}$ differentiable, $\Theta \subset \mb{R}^d$. 
    If $\ell$ is a convex and $L-$smooth function, then it satisfies the co-coercivity inequality,
    \begin{equation}\label{co-coercivity}
        \forall u, v \in \Theta,\quad \langle \nabla \ell (u) - \nabla \ell (v), u-v\rangle \geq \frac{1}{L} \|\nabla \ell (u) - \nabla \ell (v)\|_2^2.
    \end{equation}
\end{lemma}
We recall the concept of absolute continuity.
\begin{definition}\label{def:AC}
A function $f: [a, b] \to \mathbb{R}^d$ is said to be absolutely continuous on $[a, b]$ if for every $\epsilon > 0$, there exists $\delta > 0$ such that
$
    \sum_{k=1}^N \|f(b_k) - f(a_k)\|_2 < \epsilon
$
whenever $\{(a_k, b_k)\}_{k=1}^N$ is a finite collection of disjoint subintervals of $[a, b]$ with
$
    \sum_{k=1}^N (b_k - a_k) < \delta.
$
An equivalent characterization is that $f$ is continuous, differentiable almost everywhere, its derivative $f'$ is integrable, and
\[
    f(x) = f(a) + \int_a^x f'(t) \, dt \quad \text{for all } x \in [a, b].
\]
\end{definition}
In the multivariate setting, we work with functions that are absolutely continuous on lines (ACL). 
\begin{definition}\label{def:ACL}
    Let $D,d \in \mathbb{N}$. 
    The map $F: \R^D \to \R^d$, is absolutely continuous on lines (ACL) if for each $k\in[D]$, the functions $t \mapsto F(\ldots, x_{k-1}, t, x_{k+1}, \ldots)$ are locally absolutely continuous in $t$ on $\R$, i.e., absolutely continuous (cf. Definition~\ref{def:AC}) on compact subsets of $\R$, for almost every $x_{-k} \vcentcolon= (\ldots, x_{k-1}, x_{k+1}, \ldots) \in \R^{D-1}$.
\end{definition}
This notion originates in the analysis of Sobolev functions, where ACL plays a key role in relating weak and pointwise differentiability.
In fact, it provides a characterization of the Sobolev spaces for locally integrable functions \citep[cf.][Theorem 11.45]{leoni2017first}.
\begin{definition}\label{def:sobolevfunc}
    Let $D,d \in \mathbb{N}$ and $\Omega \subset \R^D$ an open set.
    The local Sobolev space $W^{1,1}_{loc}(\Omega; \R^d)$ consists of functions $u \in L^1_{loc}(\Omega; \R^d)$ such that their weak partial derivatives exist and belong to $L^1_{loc}(\Omega; \R^d)$, 
    \[
    L^1_{loc}(\Omega; \R^d) = \{ f: \Omega \mapsto \R^d \text{ measurable s.t. } \int_K \|f(x)\|_2\ dx < \infty \text{ for every compact } K \subset \Omega \}.
    \]
\end{definition}
Finally, we formally define the notion of input-wise monotonicity.
\begin{definition}\label{def:workerwisemonotone}
    Let $d, n \in \mathbb{N}$.
    The map $F$ is input-wise monotone if for every input vector index $j \in [n]$ and every fixed configuration of other inputs $g_{-j} \in (\mathbb{R}^d)^{n-1}$, the partial map $f_j: \mathbb{R}^d \to \mathbb{R}^d$ defined by $f_j(u) \vcentcolon= F(\ldots, g_{j-1}, u, g_{j+1}, \ldots)$ satisfies the monotonicity condition
    \[ \langle f_j(u) - f_j(v), u - v \rangle \ge 0 \quad \text{for all } u, v \in \mathbb{R}^d. \]
\end{definition}

\section{\texorpdfstring{Extension of~Theorem~\ref{thm:universal_convexity}}{}}\label{app:extension}

If we do not assume any smoothness, we can make the violation of monotonicity arbitrarily large. 
This claim leads to a generalization of the theorem under even milder monotonicity assumptions. 
For instance, we consider the $\rho$-hypomonotonicity notion (relating to weak-convexity),
\begin{theorem}\label{thm:universal_hypomonotone}
    Assume the setting of~Theorem~\ref{thm:universal_convexity}.
    $\F$ is hypomonotone, i.e.\ there exists $\rho > 0$, such that\looseness=-1
    \begin{equation}\label{eq:hypomonotonicity}
        \langle \theta - \omega, \F(\theta) - \F(\omega) \rangle \geq -\rho \| \theta - \omega \|^2_2, \quad\text{ for all } \theta, \omega \in \Theta,
    \end{equation}
    if and only if $F$ is an affine aggregator.
\end{theorem}
\begin{proof}
Suppose $\F$ is $\rho$-hypomonotone on the class of $L$-smooth convex loss functions.
Suppose $L>0$ and $\rho < L\ \nu(F)$, where 
\begin{equation}\label{eq:hypomonotonicity-barrier}
    \nu(F) = \max_{i\in[n]} \sup_{\substack{{u \neq v \in \R^d,\ g_{-i} \in (\R^d)^{n-1},} \\ {\delta \in D_{\text{sep}}(u, v, f_{g_{-i}})}}} \left( \Big|\Big\langle \frac{f_{g_{-i}}(u) - f_{g_{-i}}(v)}{\|u-v\|_2}, \delta \Big\rangle\Big| \Big\langle \frac{u - v}{\|u-v\|_2}, \delta \Big\rangle \right),
\end{equation}
\begin{equation*} 
    \text{and }D_{\text{sep}}(u, v, f_{g_{-i}}) \vcentcolon= \left\{ \delta \in \mathbb{S}^{d-1} \text{ s.\ t.\ } \langle u-v, \delta \rangle > 0 > \langle f_{g_{-i}}(u)-f_{g_{-i}}(v), \delta \rangle \right\}.
\end{equation*}
The proof is identical as~Theorem~\ref{thm:universal_convexity} until~\eqref{eq:strict_separation}.
Specifically, we define the same functions as in~\ref{item:def_R_j} for indices $j \neq i$ and for $i$, instead of~\ref{item:def_R_i}, we set $M \vcentcolon=  \frac{\|u-v\|^2_2}{L \langle u-v, \delta \rangle} > 0$ and $\ell_i(\theta) = \frac{1}{2} \theta^\intercal H \theta + v^\intercal \theta$, with $H \vcentcolon= \frac{(u-v) (u-v)^\intercal}{M \langle u-v, \delta \rangle} = L \frac{(u-v) (u-v)^\intercal}{\|u-v\|^2_2}$.
We verify the properties of $\ell_i$: it is convex ($M\langle u-v, \delta \rangle > 0$), smooth ($\|H\|_{\text{op}} = L$), $\nabla \ell_i(0) = v$ and $\nabla \ell_i(M\delta) = u$.
Due to the $\rho-$hypomonotonicity condition~\eqref{eq:hypomonotonicity}, we have
\begin{equation*}
    \langle f_{g_{-i}}(u) - f_{g_{-i}}(v), M\delta \rangle = \langle \F(M\delta) - \F(0), M\delta - 0 \rangle \geq - \rho M^2 \|\delta\|^2_2,    
\end{equation*}
which yields
\begin{equation}\label{eq:rho-intermediate}
    \frac{\rho}{L} \geq \frac{|\langle f_{g_{-i}}(u) - f_{g_{-i}}(v), \delta \rangle| \langle u - v, \delta \rangle}{\|\delta\|^2_2 \|u-v\|_2^2} \vcentcolon= \nu_i(u, v, g_{-i}).
\end{equation}
This leads to a contradiction. 
Indeed, by the definition of the supremum in~\eqref{eq:hypomonotonicity-barrier}, the assumption $\nu(F) > \frac{\rho}{L}$ guarantees the existence of an index $i^* \in [n]$ and a configuration $(u^*, v^*, g_{-i^*}^*) \in \R^d \times \R^d \times (\R^d)^{(n-1)}$ such that $\nu_{i^*}(u^*, v^*, g_{-i^*}^*) > \frac{\rho}{L}$, condradicting~\eqref{eq:rho-intermediate}.   
Consequently, we have the following conclusions.
\begin{enumerate}
    \item[\textit{(1)}] If $L$ can be arbitrary (i.e., we consider general convex loss functions), then we must have $\nu(F) = 0$, and the rest of the proof is identical, i.e., $F$ must be affine for any $\rho \in \R$.\footnote{We can directly prove the contradiction with $\ell_i(\theta) = \max \left( \langle v, \theta \rangle,\ \langle u, \theta \rangle - \frac{| \langle f_{g_{-i}}(u) - f_{g_{-i}}(v), \delta \rangle | \langle u - v, \delta \rangle}{2\rho\|\delta\|_2^2} \right)$.}
    \item[\textit{(2)}] If $L$ is fixed, for $F$ to be potentially $\rho-$hypomonotone, we must have $\nu(F) \leq \frac{\rho}{L}$, i.e., at least a necessary condition is that $\nu(F) < \infty$.
\end{enumerate}
\end{proof}
Importantly, we remark that
\begin{equation*}
    \nu(F) 
    \leq \max_{i\in[n]} \sup_{u \neq v \in \R^d,\ g_{-i} \in (\R^d)^{n-1}} \left( \frac{\|f_{g_{-i}}(u) - f_{g_{-i}}(v)\|_2}{\|u-v\|_2} \right)
    = \max_{i\in[n]} \sup_{\ g_{-i} \in (\R^d)^{n-1}} \operatorname{Lip}(f_{g_{-i}}).
\end{equation*}
That is, a sufficient condition for $\nu(F)$ to be bounded is that $F$ is Lipschitz continuous.  
In fact, we prove it is a sufficient condition for $F$ to be hypomonotone.

\begin{theorem}\label{thm:lipschitz-to-hypomonotone}
Let $d, n \geq 1$, and $L>0$. 
If $\F: \R^d \to \R^d$ is $\Lambda-$Lipschitz continuous (e.g., $\operatorname{CWTM}$) and the loss functions are $L$-smooth, then $\F$ satisfies $\Lambda-$hypomonotonicity.
Moreover, by Lemma~\ref{lem:lipschitz-from-F}, $F$ must be Lipschitz continuous and $\Lambda \leq \Lambda_F L$, where
\begin{align}\label{eq:lambda_sigma_def} 
    \Lambda_F 
    &\vcentcolon= 
    \esssup_{(g_1, \ldots, g_n) \in (\mathbb{R}^d)^n} \sum_{i\in[n]} \left\| \frac{\partial F}{\partial g_i}(g_1, \ldots, g_n) \right\|_{\text{sp}}
    \leq \sup_{g_{-i} \in (\mathbb{R}^d)^{n-1}} \sum_{i\in[n]} \operatorname{Lip}(f_{g_{-i}}).
\end{align}
\end{theorem}
\begin{proof}
Using the Cauchy–Schwarz inequality yields
\begin{equation*}
    \langle \theta - \omega, \F(\theta) - \F(\omega) \rangle \ge - \|\theta - \omega\|_2 \|\F(\theta) - \F(\omega)\|_2.
\end{equation*}
Substituting the Lipschitz continuity bound $\|\F(\theta) - \F(\omega)\|_2 \leq \Lambda \|\theta - \omega\|_2$, we have
\begin{equation*}
    \langle \theta - \omega, \F(\theta) - \F(\omega) \rangle \geq - \Lambda \|\theta - \omega\|^2.
\end{equation*}

Moreover, by Lemma~\ref{lem:lipschitz-from-F}, $F$ must be Lipschitz continuous.
Let consider the maximum sum of input-wise Lipschitz constants~\eqref{eq:lambda_sigma_def}.
Let $\{ \ell_i \}_{i\in[n]}$ be $L-$smooth functions.
By the chain rule, for almost every $\theta \in \R^d$,
\begin{equation}\label{eq:chain-rule-lips}
    \nabla\F(\theta) = \sum_{i=1}^n \frac{\partial F}{\partial g_i}(\nabla \ell_1(\theta), \ldots, \nabla \ell_n(\theta))\ \nabla^2 \ell_i(\theta).
\end{equation}

From~\eqref{eq:chain-rule-lips}, using $\sup_{\theta\in\R^d}\|\nabla^2 \ell_i(\theta)\|_{\text{sp}} \leq L$, the triangular and Cauchy-Schwarz inequality, we have for almost every all $\theta\in\R^d$
\begin{equation*}\label{eq:Lipschitz-for-F}
    \|\nabla\F(\theta)\|_{\text{sp}} 
    \leq L \sum_{i=1}^n \left\| \frac{\partial F}{\partial g_i}(\nabla \ell_1(\theta), \ldots, \nabla \ell_n(\theta)) \right\|_{\text{sp}}
\end{equation*}
Hence, $\operatorname{Lip}(\F) = \operatorname{ess\ sup}_{\theta \in \R^d} \| \nabla\F(\theta) \|_{\text{sp}} \leq L \Lambda_F$.
\end{proof}

\begin{remark}
    The proof extends verbatim if we instead consider hypo–co-coercive
    \begin{equation}\label{eq:hypococo}
        \langle \theta - \omega, \F(\theta) - \F(\omega) \rangle \geq -\rho \| \F(\theta) - \F(\omega) \|_2^2.
    \end{equation}
    The only modification is that we obtain the condition
    \[
        \rho L \geq \frac{\|u - v\|_2^2}{\langle u-v, \delta \rangle} \frac{|\langle f_{g_{-i}}(u) - f_{g_{-i}}(v), \delta \rangle |}{\|f_{g_{-i}}(u) - f_{g_{-i}}(v)\|_2^2}.
    \]
    Hence, smaller values of $L$ make the contradiction easier to achieve.
\end{remark}

The proof extends to strong monotonicity setting, with $\{g_i\}_{i\in[n]}$ $\mu-$strongly montone gradients, adapting the reasoning with a limit argument.
For simplicity, we assume $F$ continuous (monotone implies almost everywhere continuous, so the argument should be valid with a careful derivation).

\begin{theorem}\label{thm:universal_strongmonotone}
    Assume the setting of~Theorem~\ref{thm:universal_convexity}, additionally assuming the gradients are $\mu-$strongly monotone and that $F$ is continuous.
    $\F$ is $\mu_\F-$strongly monotone, i.e.\ there exists $\mu_\F > 0$ such that
    \begin{equation}\label{eq:strongmonotone}
        \langle \theta - \omega, \F(\theta) - \F(\omega) \rangle \geq \mu_\F \| \theta - \omega \|^2_2, \quad\text{ for all } \theta, \omega \in \Theta,
    \end{equation}
    if and only if $F$ is an affine aggregator.
\end{theorem}
% \begin{proof}
\textbf{Proof}\;
Suppose $\F$ is $\mu_\F$-strong monotone on the class of $\mu$-strongly convex loss functions.
The proof is identical as~Theorem~\ref{thm:universal_convexity} until~\eqref{eq:strict_separation}.
However we require to adapt~\ref{item:def_R_i} to be $\mu-$stongly convex and we cannot use the linear losses~\ref{item:def_R_j}.
We instead consider the following functions.
\begin{itemize}
    \item For index $i$, define $\ell_i(\theta) = \frac{1}{2} \theta^\intercal H_\epsilon \theta + v^\intercal \theta$, where $H_\epsilon = \mu I + \frac{1}{\langle z, \delta \rangle} zz^\intercal$, $0 < \epsilon < \frac{\langle u-v, \delta \rangle}{\mu \|\delta\|^2_2}$, and $z = \frac{1}{\epsilon}(u-v) - \mu \delta$.
    As $\langle z, \delta \rangle = \frac{1}{\epsilon}\langle u-v, \delta \rangle - \mu \|\delta\|_2^2 > 0$ and $zz^\intercal$ is rank $1$ matrix, $H_\epsilon \succeq \mu I$.  
    %$H_\epsilon = \mu I + \frac{1}{\langle z, \delta \rangle} zz^\intercal \succeq \mu I$ because $\frac{1}{\langle z, \delta \rangle} zz^\intercal \succeq 0$ and is rank 1. 
    Moreover, $\nabla \ell_i(0) = v$ and $\nabla \ell_i(\epsilon \delta) = H_{\epsilon}\epsilon\delta + v = \mu\epsilon\delta + \epsilon z + v = u$.
    \item For indices $j \neq i$, define $\ell_j(\theta) = g_j^\intercal \theta + \frac{\mu}{2} \|\theta\|_2^2$. 
    Then $\nabla \ell_j(0) = 0$ and $\nabla \ell_j(\epsilon\delta) = g_j + \epsilon\mu\delta$.
\end{itemize}
We have,
\begin{equation*} 
    \F(0) = F(g_1, \dots, v, \dots, g_n) = F(v, g_{-i}),
\end{equation*}
and
\begin{equation*}
    \F(\epsilon\delta) = F(g_1 + \epsilon \mu \delta, \dots, u, \dots, g_n + \epsilon \mu \delta) = F(u, g_{-i} + \epsilon\mu\delta).
\end{equation*}
The strong monotonicity condition implies
\begin{equation}\label{eq:strong_ineq-intermediate}
    \epsilon \langle F(u, g_{-i} + \epsilon\mu\delta) - F(v, g_{-i}), \delta \rangle = \langle \F(\epsilon \delta) - \F(0), \epsilon \delta \rangle \geq \epsilon^2 \rho \|\delta\|^2.
\end{equation}
Taking the limit $\epsilon \to 0$ of~\eqref{eq:strong_ineq-intermediate} yields the contradiction
\[
    \lim_{\epsilon \to 0} \langle F(u, g_{-i} + \epsilon\mu\delta) - F(v, g_{-i}), \delta \rangle = \langle f_{g_{-i}}(u) - f_{g_{-i}}(v), \delta \rangle \geq 0 = \lim_{\epsilon \to 0} \epsilon \rho \|\delta\|^2.
    \tag*{\BlackBox}
\]
% \end{proof}

\section{Deferred proof of Section~\ref{sec:convergenceimplications}}\label{app:deferred-section-3-1}

This appendix contains the proof of Theorem~\ref{cor:cone-instability}, which was deferred from Section~\ref{sec:convergenceimplications}.
\begin{reptheorem}{cor:cone-instability}
    Let $d>1$, and $F: (\mathbb{R}^d)^n \to \mathbb{R}^d$ be such that there exists $(g_1,\ldots,g_n)\in(\mathbb{R}^d)^n$ for which $F(g_1,\ldots,g_n)=0$ and at which $F$ is not positively affine.
    Then, there exist convex functions,
     
    \noindent
    \begin{minipage}[c]{0.79\textwidth}
        $\theta_{eq} \in\{z \mid \mathcal{F}(z)=0\} \neq \emptyset$, and an open set $\mathcal{C} \subset \mathbb{S}^{d-1} \vcentcolon= \{u\in\R^d; \|u\|_2 = 1\}$ such that for any 
        $\theta_0 \in \mathcal{K} \coloneqq \{ \theta_{eq} + \alpha c \mid c \in \mathcal{C}, \alpha > 0 \},$
        the update step~\eqref{eq:algo-update}, for any $\gamma > 0$, strictly increases the distance to the equilibrium
        \[
            \| \theta_{1} - \theta_{eq} \|_2 > \| \theta_0 - \theta_{eq} \|_2.
        \]
    \end{minipage}%
    \hfill
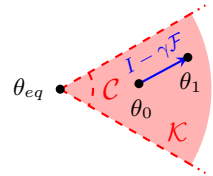
    \begin{minipage}[c]{0.21\textwidth}
        %\centering
        \
        \begin{tikzpicture}[scale=0.87, >=stealth]
            % definitions
            \coordinate (O) at (0,0);
            \def\coneangle{30}% Increased angle for a wider cone
            \def\radius{2.3}
            % 1. Cone
            \fill[red!30] (O) -- (\coneangle:\radius) arc (\coneangle:-\coneangle:\radius) -- cycle;
            \draw[red, dashed, thick] (O) -- (\coneangle:\radius);
            \draw[red, dotted, thick] (\coneangle:\radius) -- (\coneangle:\radius+0.25);
            \draw[red, dashed, thick] (O) -- (-\coneangle:\radius);
            \draw[red, dotted, thick] (-\coneangle:\radius) -- (-\coneangle:\radius+0.25);
            % 2. Equilibrium
            \node[circle, fill=black, inner sep=1.2pt, label={[font=\footnotesize]left:$\theta_{eq}$}] at (O) {};
            % 3. Points & Vector
            \coordinate (T0) at (4:1.2);
            \coordinate (T1) at (14:2);
            % Update Arrow
            \draw[->, blue, thick] (T0) -- (T1) node[midway, above, sloped, scale=0.85, font=\footnotesize] {$I-\gamma\F$};
            % Dots
            \node[circle, fill=black, inner sep=1.2pt, label={[font=\footnotesize, black]below:$\ \theta_0$}] at (T0) {};
            \node[circle, fill=black, inner sep=1.2pt, label={[font=\footnotesize, black]below:$\ \theta_1$}] at (T1) {};
            % Cap C
            \draw[thick, dashed, red] (\coneangle:0.5) arc (\coneangle:-\coneangle:0.5);
            \node[red, right, scale=1] at (0:0.5) {$\mathcal{C}$};
            % Label K inside the cone area
            \node[red, font=\small] at (-20:1.9) {$\mathcal{K}$};
            % Alpha: Starts at offset (235 degrees, 0.2cm) and draws parallel to the cone edge
        \end{tikzpicture}
    \end{minipage}
\end{reptheorem}
\noindent\textbf{Proof}\;
    For any $i \in [n]$, any $(g_1,\ldots, g_{i-1}, g_{i+1}, \ldots, g_n) \in (\R^d)^{n-1}$, and any vector $w \in \R^d$, we define an operator $f_{g_{-i}}(w) = F\left( g_1, \ldots, g_n \right) $ where the $i$-th vector $g_i = w$. 
    By assumption, there exists $n$ vectors $g_1, \ldots, g_n \in \R^d$, $i \in [n]$ and $u \in \R^d \setminus \{0\}$ such that 
    \begin{equation*}
        f_{g_{-i}}(v) = F\left( g_1, \ldots, g_n \right) = 0
        \quad\text{and}\quad
        f_{g_{-i}}(u) - f_{g_{-i}}(v)  \not \in \operatorname{cone}(u-v).
    \end{equation*}
    % In the following, let $f_{g_{-i}}(v) = F\left( g_1, \ldots, g_n \right) = 0$. 
    % Since $F$ is assumed to be locally non-affine at the stationary configuration, there exists $i$ and $u \in \R^d$ such that $f_{g_{-i}}(u) - f_{g_{-i}}(v) = f_{g_{-i}}(u) \not \in \operatorname{cone}(u-v)$. 
    Define the unit vector
    \begin{align*}
        c = \left\| \frac{u - v}{\| u - v \|_2} - \frac{f_{g_{-i}}(u) }{\| f_{g_{-i}}(u) \|_2} \right\|_2^{-1} \left( \frac{u - v}{ \| u - v \|_2} - \frac{f_{g_{-i}}(u) }{\| f_{g_{-i}}(u) \|_2}\right).
    \end{align*}
    For an arbitrary $\alpha > 0$, let $\delta = \alpha c$.
    Since $\iprod{u - v}{\delta} > 0$, the matrix $H = \frac{(u - v) (u - v)^\intercal}{\iprod{u - v}{\delta}}$ is positive semi-definite. 
    Consider the following set of convex loss functions.
    \begin{itemize}
        \item For index $i$, define $\ell_i(\theta) = \frac{1}{2} \left(\theta - \theta_{eq} \right)^\intercal H \left(\theta - \theta_{eq} \right) + v^\intercal \theta$. 
        Then, $\nabla \ell_i\left( \theta_{eq} \right) = v$ and $\nabla \ell_i \left( \theta_{eq} + \delta \right) = H \delta + v  = u$.
        \item For indices $j \neq i$, define $\ell_j(\theta) = g_j^\intercal \theta$. 
        Then, $\nabla \ell_j \left( \theta_{eq} \right) = \nabla \ell_j \left( \theta_{eq} + \delta \right) = g_j$.
    \end{itemize}
    In this particular setting, we indeed have $\F\left( \theta_{eq} \right) = f_{g_{-i}}(v) = 0$ and $\F\left( \theta_{eq} + \delta \right) = f_{g_{-i}}(u)$. 
    In the following, set $\theta_0 = \theta_{eq} + \delta = \theta_{eq} + \alpha c $. We obtain that
    \begin{multline}\label{eqn:non-monotone-1}
        \iprod{\F\left( \theta_0 \right) }{ \theta_0 - \theta_{eq}} 
        = \iprod{f_{g_{-i}}(u)}{\delta} \\
        = \alpha \left\| \frac{u - v}{\| u - v \|_2} - \frac{f_{g_{-i}}(u) }{\| f_{g_{-i}}(u) \|_2 } \right\|_2^{-1} \left( \frac{\iprod{f_{g_{-i}}(u)}{u - v}}{\| u - v \|_2} - \| f_{g_{-i}}(u) \|_2\right) < 0.
    \end{multline}
    Consider one step update $\theta_1 = \theta_0 - \gamma \F(\theta_0)$. We have
    \begin{align*}
        \| \theta_1 - \theta_{eq} \|_2^2 & = \| \theta_0 - \theta_{eq}\|_2^2 - \gamma \iprod{\F(\theta_0)}{\theta_0 - \theta_{eq}} + \gamma^2 \| \F(\theta_0) \|_2^2 \\
        & \geq \| \theta_0 - \theta_{eq} \|_2^2 - \gamma \iprod{\F(\theta_0)}{\theta_0 - \theta_{eq}} \underset{\eqref{eqn:non-monotone-1}}{>} \| \theta_0 - \theta_{eq} \|_2^2.
    \end{align*} 
    This construction holds for any $c'$ separating vector of $u-v$ and $f_{g_{-i}}(u)$, i.e., $\iprod{f_{g_{-i}}(u)}{c'} < 0 <\iprod{u - v}{c'}$. 
    In particular, it holds for every $c' \in \mathcal{C} \vcentcolon= \{c' \in \mathbb{S}^{d-1}; \|c'-c\|_2 < \omega \}$, where 
    \begin{equation*}
        \omega 
        = \iprod{c}{\frac{u - v}{\| u - v \|_2}} 
        = - \iprod{c}{\frac{f_{g_{-i}}(u) }{\| f_{g_{-i}}(u) \|_2}}
        = \sqrt{\frac{1}{2} -  \frac{1}{2}\iprod{\frac{u - v}{\| u - v \|_2}}{\frac{f_{g_{-i}}(u) }{\| f_{g_{-i}}(u) \|_2}}} 
    \end{equation*}
    Since $f_{g_{-i}}(u) \not \in \operatorname{cone}(u-v)$, we have $\omega > 0$. 
    Hence $\mathcal{C}$ is a non-empty open set of $\mathbb{S}^{d-1}$, and every $c' \in \mathcal{C}$ separates $u-v$ and $f_{g_{-i}}(u)$. 
    Indeed,
    \begin{equation*}
        \iprod{c'}{\frac{u - v}{\| u - v \|_2}} = \iprod{c}{\frac{u - v}{\| u - v \|_2}} + \iprod{c'-c}{\frac{u - v}{\| u - v \|_2}} 
        \geq \omega - \|c'-c\|_2 
        > 0,
    \end{equation*}
    and similarly, 
    \begin{equation*}
        \iprod{c'}{\frac{f_{g_{-i}}(u) }{\| f_{g_{-i}}(u) \|_2}} = \iprod{c}{\frac{f_{g_{-i}}(u) }{\| f_{g_{-i}}(u) \|_2}} + \iprod{c'-c}{\frac{f_{g_{-i}}(u) }{\| f_{g_{-i}}(u) \|_2}} 
        \leq -\omega + \|c'-c\|_2 
        < 0. \tag*{\BlackBox}
    \end{equation*}

\section{Deferred proofs of Section~\ref{sec:stabilityimplications}}\label{app:deferred-section-3-2}
This appendix contains the proof of Theorem~\ref{cor:unavoidable-instability} and Lemma~\ref{thm:1-step-dsep}, which was deferred from Section~\ref{sec:stabilityimplications}.
\begin{reptheorem}{cor:unavoidable-instability}
    Assume the setting of Theorem~\ref{cor:nococoercivity}, with gradients norm bounded by $C>0$ and neighboring datasets inducing gradient perturbations of at most $\rho > 0$.
    Let $\mathcal{A}_F$ denote the algorithm that outputs $\theta_T$ after $T \geq 1$ iterations of~\eqref{eq:algo-update} with aggregation rule $F$ and step size $\gamma \leq 1/L$.
    Then,
    \begin{equation}\label{eq:USparameter-app}
        \textstyle \operatorname{Stab}(\mathcal{A}_{F}) \leq \gamma T S_{n, C}(F, \rho) + \mathcal{I}(F),
    \end{equation}
    where $\mathcal{I}(F) \vcentcolon= \sup_{\mathcal{S} \sim \mathcal{S}'}\sum_{t=1}^{T-1} \| \theta_t - \gamma \mathcal{F}(\theta_t) - (\theta_t' - \gamma \mathcal{F}(\theta_t')) \|_2 - \| \theta_t - \theta_t' \|_2$.
    Furthermore, there exist convex, $L$-smooth, and $C$-Lipschitz loss functions such that
    \begin{equation}\label{eq:unified-lower-bound-app}
        \textstyle \operatorname{Stab}(\mathcal{A}_{F}) \geq \frac{1}{3} \left( \gamma T S_{n, C}(F, \rho) + \mathcal{I}(F) \right).
    \end{equation}
\end{reptheorem}
\noindent\textbf{Proof}\;
    Let $\{\theta_t\}_{t=0}^T$ and $\{\theta'_t\}_{t=0}^T$ be coupled trajectories initialized at $\theta_0 = \theta'_0$, generated by neighboring datasets $\mathcal{S} \sim \mathcal{S}'$. 
    The iterates evolve according to $\theta_{t+1} = \theta_t - \gamma \F(\theta_t)$ and $\theta'_{t+1} = \theta'_t - \gamma \F'(\theta'_t)$, where $\F$ and $\F'$ differ in a single gradient input. 
    This difference is bounded in magnitude by at most $\rho$, as specified in Definition~\ref{def:local-sensitivity}.
    For $t\in \{0, \ldots, T-1\}$, define the sensitivity term 
    \begin{equation}\label{eq:sensitivity-assumption-eq}
        s_{F, t} = \| \F'(\theta'_t) - \F(\theta'_t) \|_2 \leq S_{n, C}(F, \rho),
    \end{equation}
    %$s_{F, t} = \| \F'(\theta'_t) - \F(\theta'_t) \|_2$ 
    and the expansivity term 
    \begin{equation*}
        \sigma_{F, t} = \| \theta_t - \gamma \F(\theta_t) - (\theta_t' - \gamma \F(\theta_t')) \|_2  - \| \theta_t - \theta_t' \|_2.
    \end{equation*}
    We evaluate the distance between trajectories at step $t+1$
    \begin{equation}\label{eq:before-bounding}
        \| \theta_{t+1} - \theta'_{t+1} \|_2 
        = \| [\theta_t - \gamma \F(\theta_t) - (\theta'_t - \gamma \F(\theta'_t))] + \gamma(\F(\theta'_t) - \F'(\theta'_t)) \|_2.
    \end{equation}
    We then use the triangular inequality and reverse triangular inequality to upper and lower bound~\eqref{eq:before-bounding},
    % \[
    %     \|\theta_{t+1} - \theta'_{t+1}\|_2
    %     \begin{cases}
    %     \leq \| \theta_t - \gamma \F(\theta_t) - (\theta'_t - \gamma \F(\theta'_t)) \|_2 + \gamma s_{F,t} = \|\theta_t - \theta'_t\|_2 + \sigma_{F,t} + \gamma s_{F,t}, \\
    %     \geq \| \theta_t - \gamma \F(\theta_t) - (\theta'_t - \gamma \F(\theta'_t)) \|_2 - \gamma s_{F,t} = \|\theta_t - \theta'_t\|_2 + \sigma_{F,t} - \gamma s_{F,t}.
    %     \end{cases}
    % \]
    \begin{align*}
        \| \theta_{t+1} - \theta'_{t+1} \|_2 
        \leq \| \theta_t - \gamma \F(\theta_t) - (\theta'_t - \gamma \F(\theta'_t)) \|_2 + \gamma s_{F,t} 
        = \| \theta_t - \theta'_t \|_2 + \sigma_{F,t} + \gamma s_{F,t},
    \end{align*}
    and
    \begin{align*}
        \| \theta_{t+1} - \theta'_{t+1} \|_2 
        \geq \| \theta_t - \gamma \F(\theta_t) - (\theta'_t - \gamma \F(\theta'_t)) \|_2 - \gamma s_{F,t}
        = \| \theta_t - \theta'_t \|_2 + \sigma_{F,t} - \gamma s_{F,t},
    \end{align*}
    Summing over $t=0$ to $T-1$, and using the bounded sensitivity assumption~\eqref{eq:sensitivity-assumption-eq} yields
    \begin{equation}\label{eq:telescoping_traj_ub}
        \| \theta_T - \theta'_T \|_2 
        \leq \sum_{t=0}^{T-1} \sigma_{F,t} + \gamma \sum_{t=0}^{T-1} s_{F,t}
        \leq \sum_{t=0}^{T-1} \sigma_{F,t} + \gamma T S_{n, C}(F, \rho),
    \end{equation}
    and 
    \begin{equation}\label{eq:telescoping_traj}
        \| \theta_T - \theta'_T \|_2 
        \geq \sum_{t=0}^{T-1} \sigma_{F,t} - \gamma \sum_{t=0}^{T-1} s_{F,t}
        \geq \sum_{t=0}^{T-1} \sigma_{F,t} - \gamma T S_{n, C}(F, \rho).
    \end{equation}
    Taking the supremum over all neighboring datasets $\mathcal{S} \sim \mathcal{S}'$ in~\eqref{eq:telescoping_traj_ub} and~\eqref{eq:telescoping_traj}, 
    noting that $\gamma T S_{n,C}(F, \rho)$ is independent of $\mathcal{S}, \mathcal{S}'$, we recover the definition of $\mathcal{I}(F)$ and obtain the following upper bound
    \begin{equation*}\label{eq:NonAff-stab}
        \operatorname{Stab}(\mc{A}_{F}) 
        = \sup_{\mathcal{S} \sim \mathcal{S}'} \| \theta_T - \theta'_T \|_2
        \leq \gamma T  S_{n, C}(F) + \sup_{\mc{S} \sim \mc{S}'} \sum_{t=0}^T \sigma_{F, t}
        = \gamma T  S_{n, C}(F) + \mathcal{I}(F),
    \end{equation*}
    and the following lower bound
    \begin{equation}\label{eq:lb-expansivity}
        \operatorname{Stab}(\mathcal{A}_F)  
        \geq \sup_{\mathcal{S} \sim \mathcal{S}'} \sum_{t=0}^{T-1} \sigma_{F,t} - \gamma T S_{n,C}(F, \rho) 
        = \mathcal{I}(F) - \gamma T S_{n,C}(F, \rho).
    \end{equation}

    \paragraph{Second lower bound.}
    Moreover, by definition of $S_{n, C}(F, \rho)$, for any arbitrarily small $\eta > 0$, there exist $i \in [n]$ and gradients $g_1, \ldots, g_n, g_i' \in \widebar B(0, C)$ such that $\|g_i - g_i'\| \leq \rho$ and 
    \[
    \| F(\dots, g_i, \dots) - F(\dots, g'_i, \dots) \|_2 \geq S_{n, C}(F, \rho) - \eta.
    \]
    We define $\ell_j(\cdot) = \langle g_j, \cdot \rangle$ for $j \neq i$, and for the differing $i-$th functions, we set $\ell_i(\cdot) = \langle g_i, \cdot \rangle$ and $\ell_i'(\cdot) = \langle g_i', \cdot \rangle$. 
    These functions are convex, $L$-smooth (with $L=0$), and $C$-Lipschitz. 
    Because the gradients are constant, the induced operators $\mathcal{F}$ and $\mathcal{F}'$ are constant and, for any $\theta_0 \in \R^d$ and $t \geq 1$,
    \begin{equation*}
        \| \theta_T - \theta_T' \|_2 
        = \| \gamma T \big( F(\dots, g_i, \dots) - F(\dots, g'_i, \dots) \big) \|_2
        \geq \gamma T (S_{n, C}(F, \rho) - \eta).
    \end{equation*}
    Because $\operatorname{Stab}(\mathcal{A}_F)$ is the supremum over all neighboring datasets, it must be that $\operatorname{Stab}(\mathcal{A}_F) \geq \gamma T (S_{n, C}(F, \rho) - \eta)$. 
    Since this holds for all $\eta > 0$, we conclude 
    \begin{equation}\label{eq:lb-sensitivity}
    \operatorname{Stab}(\mathcal{A}_F) \geq \gamma T S_{n, C}(F, \rho).
    \end{equation}

    \paragraph{Combining the lower bounds.}
    Finally, summing the two lower bounds~\eqref{eq:lb-sensitivity} and~\eqref{eq:lb-expansivity} yields
    \begin{align*}
        3 \operatorname{Stab}(\mathcal{A}_F) 
        = \operatorname{Stab}(\mathcal{A}_F) + 2\operatorname{Stab}(\mathcal{A}_F)
        & \geq (\mathcal{I}(F) - \gamma T S_{n,C}(F, \rho)) + 2 \gamma T S_{n,C}(F, \rho) \\
        & = \mathcal{I}(F) + \gamma T S_{n,C}(F, \rho). 
        \tag*{\BlackBox}
    \end{align*}

\noindent While $\mathcal{I}(F) = 0$ for positively affine aggregation rules, we show that this property does not hold for non-affine aggregation rules under mild assumptions.
Indeed, building on the techniques of Theorems~\ref{thm:universal_convexity} and~\ref{cor:nococoercivity}, we formalize below mild conditions capturing practical non-affine aggregations under which the expansivity of the update provably degrade algorithmic stability.
The core difficulty in proving these lower bounds arises from the constraints of the stability analysis. 
While Corollary~\ref{cor:expansivity} establishes expansivity for specific input pairs, stability analysis requires expansivity to occur along two coupled trajectories $\{\theta_t\}_{t=1}^T$ and $\{\theta_t'\}_{t=1}^T$, with common initialization $\theta_0=\theta_0'$, 
generated by the update~\eqref{eq:algo-update} whose inputs differ only through the $i$-th gradient, a difference induced by a single-sample perturbation.

Formally, for the first step, we require $\theta_1 - \theta_1' = \gamma ( f_{g_{-i}}(g_i(\theta_0)) - f_{g_{-i}}(g_i'(\theta_0)) )$ to be a separating vector of $g_i(\theta_1) - g_i(\theta_1')$ and $f_{g_{-i}}(g_i(\theta_1)) - f_{g_{-i}}(g_i(\theta_1'))$. 
This means the hyperplane normal to $\theta_1 - \theta_1'$ separates these two vectors, 
\[
    \textstyle \langle g_i(\theta_1) - g_i(\theta_1'), \theta_1 - \theta_1' \rangle > 0 \text{ and } \langle f_{g_{-i}}(g_i(\theta_1)) - f_{g_{-i}}(g_i(\theta_1')), \theta_1 - \theta_1' \rangle < 0. 
\]
The set of separating vectors is non-empty if and only if $f_{g_{-i}}(g_i(\theta_1))-f_{g_{-i}}(g_i(\theta_1')) \not\in \operatorname{cone}(g_i(\theta_1)- g_i(\theta_1'))$, a condition guaranteed for some gradient configurations since $F$ is non-affine. 
However, constructing a setup where the \textit{specific} vector $\theta_1-\theta_1'$ is a separating vector is non-trivial. %, and represents the crux of our construction.

This complex construction is necessary to prove degraded stability for continuous aggregation rules, as they typically have a similar sensitivity than their affine counterparts.
In contrast, discontinuous aggregation rules degrade stability even without violating co-coercivity. 
Indeed, a discontinuity induces the sensitivity term $S_{n, C}(F, \rho)$ to not decay nicely with more data, and hence remains a bottleneck independent of the total dataset size.
% \tbtodo{Discontinuous rules broadly accommodate methods in robust learning (e.g., the $\operatorname{filter}$ algorithm~\citealt{diakonikolas2019robust}, $\operatorname{SMEA}$~\citealt{pmlr-v202-allouah23a}, $\operatorname{CAF}$~\citealt{allouah2025towards}, $\operatorname{Krum}$~\citealt{blanchard-ml-with-adversaries} or $\operatorname{NNM}$~\citealt{pmlr-v206-allouah23a}), whereas continuous non-affine rules represent the complement of these methods.}
% While discontinuities readily produce such a setup, achieving this under continuous aggregation is more involved.
% Consequently, we separate our constructions depending on whether the non-affine aggregation rule is discontinuous or continuous. 
% Discontinuous rules broadly accommodate methods in robust learning (e.g., the $\operatorname{filter}$ algorithm~\citealt{diakonikolas2019robust}, $\operatorname{SMEA}$~\citealt{pmlr-v202-allouah23a}, $\operatorname{CAF}$~\citealt{allouah2025towards}, $\operatorname{Krum}$~\citealt{blanchard-ml-with-adversaries} or $\operatorname{NNM}$~\citealt{pmlr-v206-allouah23a}), whereas continuous rules represent the complement of these methods.

\noindent\textbf{Lemma~\ref{thm:1-step-dsep}}
    \emph{
        Assume the setting of Theorem~\ref{cor:unavoidable-instability} and Assumption~\ref{assump:one-step-lb-stab-continuous}, with $T=2$. 
        Then, for any $\theta_0 \in \R^d$, there exist convex, $L-$smooth and $C-$Lipschitz loss functions $\{\ell_j\}_{j \in [n]}, \ell'_i$ such that
        \begin{equation*}
            %\|\theta_{2} - \theta'_{2}\|_2 > \sigma_{F,1} > 0,
            \mathcal{I}(F) > 0.
        \end{equation*}
        % where $\sigma_{F,1} \vcentcolon= \| \theta_1 - \gamma \F(\theta_1) - (\theta_1' - \gamma \F(\theta_1')) \|_2  - \| \theta_1 - \theta_1' \|_2$ captures the expansivity of the update.
    }
\begin{proof} 
    We proceed constructively and fix the initialization to $\theta_0 = 0$ for convenience.

    \paragraph{Setup \& loss functions.}
    By Assumption~\ref{assump:one-step-lb-stab-continuous}, there exist $i \in [n]$, $g_{-i} \in \widebar B(0, C)^{n-1}$, $p \in \widebar B(0, C)$ and $\tilde \delta \neq 0$ such that for any $r_b > 0$, there exist $\epsilon > 0, b_0 \in B(p, r_b) \cap \widebar{B}(0, C)$ such that $\epsilon \tilde \delta = f_{g_{-i}}(b_0)-f_{g_{-i}}(p)$.
    Moreover, for any $r_q > 0$, there exists $q \in B(p, r_q) \cap \widebar B(0, C)$ 
    %such that $\langle p - q, \tilde \delta \rangle > 0$ and $\langle f(p) - f(q), \tilde \delta \rangle < 0$, i.e., $f_{g_{-i}}(p)-f_{g_{-i}}(q) \not\in \operatorname{cone}(p-q)$. 
    \begin{equation}\label{eq:int_sep}
        \langle p - q, \tilde \delta \rangle > 0 \quad \text{and} \quad \langle f_{g_{-i}}(p) - f_{g_{-i}}(q), \tilde \delta \rangle < 0,
    \end{equation}
    i.e., $f_{g_{-i}}(p)-f_{g_{-i}}(q) \not\in \operatorname{cone}(p-q)$.
    We set $b_1 \vcentcolon= p$, $\delta \vcentcolon= \epsilon \tilde \delta$.
    Because $\langle p - q, \delta \rangle > 0$, there exists a symmetric matrix\footnote{Explicitly, $H_1 = \frac{1}{\gamma} \frac{(p - q) (p - q)^\intercal}{\langle p - q, \delta \rangle} + \iota \left( I - \frac{\delta \delta^\intercal}{\|\delta\|^2} \right)$ for sufficiently small $\iota > 0$.} $H_1 \succ 0$ such that $\gamma H_1 \delta = p - q$.
    Moreover, we can ensure $\lambda_{\max}(H_1) \leq \frac{L}{2}$ by choosing $r_q$ sufficiently small.
    
    For $\theta \in \R^d$, we define for $j \neq i$, $\ell_j(\theta) = g_j^\intercal \theta$ so $\nabla \ell_j(\theta) = g_j$, and the following quadratics\footnote{
        Formally, to ensure $\ell_i$ and $\ell_i'$ represent empirical risks over adjacent datasets $S$ and $S'$ of size $m$, we define the individual sample loss as $\ell(\theta; (H, b)) = \frac{1}{2}\theta^\intercal H \theta + b^\intercal \theta$. 
        Set $\ell_i(\theta) = \frac{1}{m} \sum_{k=1}^{m} \ell(\theta; (H_1, b_1)) = \ell(\theta; (H_1, b_1)) $ and $\ell_i'(\theta) = \frac{m-1}{m} \ell(\theta; (H_1, b_1)) + \ell(\theta; (H', b'))$ with $H' = H_1 + m\Delta$ and $b_1 + m(b_0 - b_1)$. 
        To preserve convexity, $L$-smoothness and $C-$Lipschitzness, we must bound $\|\Delta\|_2 < \frac{1}{m}\lambda_{\min}(H_1)$ and $\|b_0 + m(b_0-b_1)\|_2 < C$, which can be done by shrinking $r_q$ and $r_b$.
        We also enforce the $C-$Lipchitz continuous constraint for all $\theta \in \R^d$ by truncating the quadratics outside the neighborhood where the proof is established.}
    \begin{align*}
        \ell_i(\theta) &= \frac{1}{2} \theta^\intercal H_1 \theta + b_1^\intercal \theta, \text{ thus } g_i(\theta) = H_1 \theta + b_1,\\
        \ell_i'(\theta) &= \frac{1}{2} \theta^\intercal H_0 \theta + b_0^\intercal \theta, \text{ thus } g_i'(\theta) = H_0 \theta + b_0.
    \end{align*}
    At $\theta_0 = 0$, $g_i(0) = b_1$ and $g_i'(0) = b_0$. 
    Hence $\theta_1 = -\gamma f_{g_{-i}}(b_1)$, $\theta_1' = -\gamma f_{g_{-i}}(b_0)$ and $\delta = \frac{\theta_1 - \theta_1'}{\gamma}$.
    We set $H_0$ so that $g_i'(\theta_1') = g_i(\theta_1') = b_1 - \gamma H_1 f_{g_{-i}}(b_0)$. 
    As $g_i'(\theta_1') = b_0 - \gamma H_0 f_{g_{-i}}(b_0)$, we require
    \begin{equation*}
        \gamma (H_0 - H_1) f_{g_{-i}}(b_0) = b_0 - b_1.
    \end{equation*}
    Let $x \vcentcolon= \gamma f_{g_{-i}}(b_0)$ and $y \vcentcolon= b_0 - b_1$. 
    We satisfy this by setting $H_0 = H_1 + \Delta$, where $\Delta$ is the symmetric rank-2 matrix $\Delta = \frac{y x^\intercal + x y^\intercal}{\|x\|^2} - \frac{\langle x, y \rangle}{\|x\|^4} x x^\intercal$. 
    We can ensure $H_0 \succ 0$ and $\lambda_{\max}(H_0) \leq L$ by chosing $r_b$ sufficiently small such that such that $\|\Delta\|_2 < \min\left(\frac{L}{2}, \lambda_{\min}(H_1)\right) = \lambda_{\min}(H_1)$. 

    \paragraph{Separation at $T=1$.} 
    Our construction yield the following identity. 
    \begin{equation}\label{eq:step3}
        g_i(\theta_1) = b_1 - \gamma H_1 f_{g_{-i}}(b_1) = p - \gamma H_1 f_{g_{-i}}(p).
    \end{equation}
    \begin{equation}\label{eq:step4}
        g_i(\theta_1') = g_i'(\theta_1') = b_1 - \gamma H_1 (f_{g_{-i}}(b_1) + \delta) = g_i(\theta_1) - \gamma H_1 \delta = g_i(\theta_1) - (p - q) 
    \end{equation}
    We next evaluate the separation property. 
    We already have
    \begin{equation*}
        \langle g_i(\theta_1) - g_i(\theta_1'), \delta \rangle = \langle \gamma H_1 \delta, \delta \rangle > 0.
    \end{equation*}
    Then, to ensure the second inequality we rely on the continuity of $f_{g_{-1}}$ in the neighborhood of $p$. 
    To do so, we analyze the limit of $g_i(\theta_1)$ and $g_i(\theta_1')$ as $q$ approaches $p$.
    Notice that we can choose $H_1$ such that $\lim_{r_q\to0} \|H_1\|_2 = 0$, therefore $\lim_{r_q\to0}g_i(\theta_1) = p$ and $\lim_{r_q\to0} g_i(\theta_1') = q$.
    Hence, 
    \begin{equation*}
        \lim_{r_q \to 0}\ \langle f_{g_{-i}}(g_i(\theta_1)) - f_{g_{-i}}(g_i(\theta_1')), \delta \rangle = \langle f_{g_{-i}}(p) - f_{g_{-i}}(q), \delta \rangle \underset{\eqref{eq:int_sep}}{<} 0. 
    \end{equation*}

    \paragraph{Consequence.}
    A single step of update~\eqref{eq:algo-update} yields (i) strict trajectories expansivity $\sigma_{F, 1} > 0$, with (ii) no triangular inequality slack $\xi_{F, 1} = 0$. 

    \paragraph{(ii)} By~\eqref{eq:step4}, we have $\mathcal{F}'(\theta_1') = f_{g_{-i}}(g_i'(\theta_1')) = f_{g_{-i}}(g_i(\theta_1')) =  \mathcal{F}(\theta_1')$.
    This proves
    \begin{equation*}
        \xi_{F,1} = \| \theta_{2} - \theta_{2}' \|_2 - \left( \| \theta_1 - \gamma \F(\theta_1) - (\theta_1' - \gamma \F(\theta_1')) \|_2 + \gamma \| \F'(\theta_1') - \F(\theta_1') \|_2 \right) 
        = 0.
    \end{equation*}

    \paragraph{(i)}
    Moreover, we prove $\sigma_{F, 1} = \|\theta_1 - \gamma \F(\theta_1) - (\theta_1' - \gamma \F(\theta_1'))\|_2 - \|\theta_1' - \theta_1\|_2 > 0$.
    In fact, squaring $\|\theta_1 - \gamma \F(\theta_1) - (\theta_1' - \gamma \F(\theta_1'))\|_2$ and isolating the dot product, we have
    \begin{equation*}
        - 2\gamma \langle \mathcal{F}(\theta_1') - \mathcal{F}(\theta_1), \theta_1' - \theta_1 \rangle
        = 2 \gamma^2 \langle f_{g_{-i}}(g_i(\theta_1')) - f_{g_{-i}}(g_i(\theta_1)),  \delta \rangle 
        > 0
    \end{equation*}
    Overall, we have
    $
        \|\theta_{2} - \theta'_{2}\|_2
        = \sigma_{F,1} + \gamma \|\delta_{F,0}\|_2
        > \sigma_{F,1}.
    $
\end{proof}

\section{\texorpdfstring{Alternative proof of~Theorem~\ref{thm:universal_convexity}}{}}\label{sec:alternative-monotonone}

In what follows, we present an alternative proof of Theorem~\ref{thm:universal_convexity} based on an analysis of the Jacobian of the induced operator $\F$. 
Our first step is to establish that $\F$ is differentiable almost everywhere, a property that carries over to the aggregation rule $F$ itself. 
While this approach yields comparable insights, it comes with a notable drawback: the argument requires integrating back the Jacobians to recover global properties of $F$. 
In full generality, this integration step necessitates additional, albeit mild, regularity assumptions on $F$ (the weakest being absolute continuity on lines), as well as careful manipulation of sets of Lebesgue measure zero and extensions of properties of monotone operators to input-wise monotone operators.
While these technical requirements are artifacts of the proof technique rather than intrinsic to the result, we present the derivation because it provides a different perspective, even if similar, and the extension of known monotone operator properties to input-wise monotone operators are of independent interest.

\begin{theorem}\label{lem:monotone-differentiable}
    Let $f: \mathbb{R}^d \to \mathbb{R}^d$ be a monotone operator defined on the entire space. 
    Then $f$ is differentiable almost everywhere (with respect to the Lebesgue measure).
\end{theorem}
\begin{proof}
    Since $f$ is monotone, Zorn's Lemma guarantees the existence of a maximal monotone operator $\tilde{f}: \mathbb{R}^d \to \mathcal{P}(\mathbb{R}^d)$ that extends $f$~\citep[Theorem 20.21]{Bauschke2011ConvexAA}. 
    This means that the graph of $f$ is contained in the graph of $\tilde{f}$, $\operatorname{Graph}(f) \subseteq \operatorname{Graph}(\tilde{f}) \vcentcolon= \{(x,y) \in \mathbb{R}^d \times \mathbb{R}^d\ |\ y \in \tilde f(x)\}$, which is maximal with respect to graph inclusion among monotone operators.
    In other words, $f(x) \in \tilde{f}(x)$ for all $x\in\R^d$.

    From~\citet[Theorem 3.2]{albamb96}, any maximal monotone operator on an open domain is differentiable almost everywhere. 
    Let $D \subset \mathbb{R}^d$ be the set of points where $\tilde{f}$ is differentiable. 
    We know that $\mathbb{R}^d \setminus D$ has Lebesgue measure zero.

    Fix a point $x_0 \in D$. By the definition of differentiability for a set-valued operator:
    (1) $\tilde{f}(x_0)$ must be a singleton and, with a slight abuse of notation, we identify this set with its unique element;
    (2) There exists a linear map (Jacobian) $L_{x_0}: \mathbb{R}^d \to \mathbb{R}^d$ such that
    \[
        \lim_{\substack{x \to x_0 \\ y \in \tilde{f}(x)}} \frac{\| y - \tilde{f}(x_0) - L_{x_0}(x - x_0) \|_2}{\| x - x_0 \|_2} = 0.
    \]
    Since $\operatorname{Graph}(f) \subseteq \operatorname{Graph}(\tilde{f})$, the condition $y \in \tilde{f}(x)$ is satisfied by $y = f(x)$. 
    Furthermore, since $\tilde{f}(x_0)$ is a singleton, $f(x_0) = \tilde{f}(x_0)$.
    Substituting these into the limit expression, we restrict the limit to points in the graph of $f$:
    \[
        \lim_{x \to x_0} \frac{\| f(x) - f(x_0) - L_{x_0}(x - x_0) \|_2}{\| x - x_0 \|_2} = 0.
    \]
    This is exactly the definition of Fréchet differentiability for the single-valued function $f$ at $x_0$.
    Since $f$ is differentiable at every point in $D$, and $\mathbb{R}^d \setminus D$ has Lebesgue measure zero, $f$ is differentiable almost everywhere.
\end{proof}
The central condition for the monotonicity of the induced aggregator is that for any admissible loss Hessian, the Jacobian of the induced aggregator must be positive semi-definite. 
Crucially, this condition implies structure over the partial Jacobians of the underlying aggregation rule used. 
\begin{lemma}\label{lem:psd-crux}
  Let $d \geq 2$, and $D \in \R^{d \times d}$ a matrix. 
  If for all symmetric p.s.d.\ matrix $H \in \mc{S}^d_+$, $DH$ is p.s.d., that is
  \begin{equation*}
    \forall H \in \mc{S}^d_+, \forall v \in \R^d, v^\intercal DH v = v^\intercal \left(\frac{DH + (DH)^\intercal}{2}\right) v \geq 0,
  \end{equation*}
  then there exists $\alpha \in \R$ such that $D = \alpha I$.
\end{lemma}
\begin{proof}
  In what follows, we restrict our attention to the case $d=2$ without loss of generality. 
  Indeed, our analysis only requires verifying whether a given matrix is positive semi-definite, a property that can be falsified by examining its $2\times2$ principal submatrices. 
  Specifically, a matrix is positive semi-definite if and only if all principal submatrices of its symmetric part are positive semi-definite. 
  Consequently, establishing a violation of positive semi-definiteness in dimension $d=2$ immediately extends to all dimensions $d \geq 2$, by considering matrices $H$ and vectors $v$ that act nontrivially on only two coordinates and vanish on the remaining ones.
  More precisely, let $a, b \in [d]$ be distinct indices, $a < b$, and $D^{(2)} \vcentcolon= \begin{pmatrix} D_{aa} & D_{ab} \\ D_{ba} & D_{bb} \end{pmatrix} \in \R^{2 \times 2}$ denote the principal submatrix of $D$ restricted to these coordinates. 
  For any $H^{(2)} \in \mc{S}^2_+$ and $v^{(2)} \in \R^2$, we construct $H \in \mc{S}^d_+$ and $v \in \R^d$ by embedding $H^{(2)}$ and $v^{(2)}$ at indices $a, b$ and setting all other entries to zero, i.e., $H_{k,l} = H_{l,k} = 0, \forall k, l \notin \{ a, b \}$, $H_{aa} = H^{(2)}_{11}$, $H_{bb} = H^{(2)}_{22}$, $H_{ab} = H_{ba} = H^{(2)}_{12}$, and $v_k = 0, \forall k \notin \{ a, b \}$, $v_a = v^{(2)}_1$, $v_b = v^{(2)}_2$.
  The condition $v^\intercal D H v \geq 0$ simplifies to
  \[
    v^\intercal D H v 
    = \sum_{i,j \in \{ a, b \}} v_i \left( \sum_{k \in \{ a, b \}} D_{ik} H_{kj} \right) v_j 
    = v^{(2)\intercal} D^{(2)} H^{(2)} v^{(2)} \geq 0.
  \]
  We conclude by applying the following derivation to all combinaison of $a, b \in [d]$.

  Consider the $2 \times 2$ submatrix $D^{(2)}$ defined above, and let
  \[
    S \vcentcolon= \frac{D^{(2)} + D^{(2)\intercal}}{2},
    \quad
    A \vcentcolon= \frac{D^{(2)} - D^{(2)\intercal}}{2}
  \]
  denote its symmetric and anti-symmetric parts, respectively, where we omit the superscript for readability.
  We work in the rotated coordinate system that diagonalizes $S$ (i.e., $\tilde{v} = Qv$, where $Q$ is some orthogonal matrix). 
  By the spectral theorem, there exists an orthogonal matrix $Q$ such that
  \[
    S = Q \tilde{S} Q^\intercal,
    \quad\text{where}\quad
    \tilde{S} = \begin{pmatrix} \lambda_1 & 0 \\ 0 & \lambda_2 \end{pmatrix},\
    \lambda_1, \lambda_2 \in \R. 
  \]
  Moreover, since orthogonal transformations map orthonormal bases to orthonormal bases and preserve anti-symmetry, the transformed matrix $\tilde{A}$ remains anti-symmetric and therefore admits the following canonical representation 
  \[
    \tilde{A} = Q^\intercal A Q = \begin{pmatrix} 0 & -\alpha \\ \alpha & 0 \end{pmatrix},
    \quad\text{where}\quad
    \alpha \in \R.
  \]
  Since $\tilde{D}^{(2)}H$ must be p.s.d.\ for all $H \in \mathcal{S}^2_+$, we test this condition using two specific choices,
  \begin{equation*}
    H_1 = \begin{pmatrix} 1 \\ 1 \end{pmatrix} \begin{pmatrix} 1 \\ 1 \end{pmatrix}^\intercal = \begin{pmatrix} 1 & 1 \\ 1 & 1 \end{pmatrix}
    \quad\text{and}\quad
    H_2 = \begin{pmatrix} 1 \\ -1 \end{pmatrix} \begin{pmatrix} 1 \\ -1 \end{pmatrix}^\intercal = \begin{pmatrix} 1 & -1 \\ -1 & 1 \end{pmatrix}.
  \end{equation*}
  Specifically, for each choice, we compute the determinant of the symmetric part of $\tilde{D}^{(2)}H$ to verify positive semi-definiteness.
  \begin{itemize}
    \item[(1)] For $H_1$, we have
      \[
        \tilde{D}^{(2)} H_1 = (\tilde{S} + \tilde{A}) H_1
        =
        \begin{pmatrix}
          \lambda_1 & -\alpha \\
          \alpha & \lambda_2
        \end{pmatrix}
        \begin{pmatrix}
          1 & 1 \\
          1 & 1
        \end{pmatrix}
        =
        \begin{pmatrix}
          \lambda_1 - \alpha & \lambda_1 - \alpha \\
          \lambda_2 + \alpha & \lambda_2 + \alpha
        \end{pmatrix}.
      \]
      Thus, the symmetric part is
      \[
        \frac{\tilde{D}^{(2)}H_1 + (\tilde{D}^{(2)}H_1)^{\intercal}}{2}
        =
        \begin{pmatrix}
          \lambda_1 - \alpha & \frac{\lambda_1 + \lambda_2}{2} \\
          \frac{\lambda_1 + \lambda_2}{2} & \lambda_2 + \alpha
        \end{pmatrix}.
      \]
      Taking the determinant yields
      \begin{align*}
        \det\left(\frac{\tilde{D}^{(2)}H_1 + (\tilde{D}^{(2)}H_1)^{\intercal}}{2}\right)
        &= (\lambda_1 - \alpha)(\lambda_2 + \alpha) - \left(\frac{\lambda_1 + \lambda_2}{2}\right)^2 \\
        &= - \alpha^2 + \alpha(\lambda_1 - \lambda_2) + \left(\lambda_1\lambda_2 - \left(\frac{\lambda_1 + \lambda_2}{2}\right)^2\right) \\
        &= -\left[\alpha^2 - 2\alpha \frac{\lambda_1 - \lambda_2}{2} + \left(\frac{\lambda_1 - \lambda_2}{2}\right)^2\right]\\
        &= -\left(\alpha - \frac{\lambda_1 - \lambda_2}{2}\right)^2.
      \end{align*}

    \item[(2)] For $H_2$, we have
      \[
        \tilde{D}^{(2)} H_2 = (\tilde{S} + \tilde{A}) H_2
        =
        \begin{pmatrix}
          \lambda_1 & -\alpha \\
          \alpha & \lambda_2
        \end{pmatrix}
        \begin{pmatrix}
          1 & -1 \\
          -1 & 1
        \end{pmatrix}
        =
        \begin{pmatrix}
          \lambda_1 + \alpha & -\lambda_1 - \alpha \\
          \alpha - \lambda_2 & \lambda_2 - \alpha
        \end{pmatrix}.
      \]

      Thus, the symmetric part is
      \[
        \frac{\tilde{D}^{(2)}H_2 + (\tilde{D}^{(2)}H_2)^{\intercal}}{2}
        =
        \begin{pmatrix}
          \lambda_1 + \alpha & -\frac{\lambda_1 + \lambda_2}{2} \\
          -\frac{\lambda_1 + \lambda_2}{2} & \lambda_2 - \alpha
        \end{pmatrix}.
      \]

      Taking the determinant yields
      \begin{align*}
        \det\left(\frac{\tilde{D}^{(2)}H_2 + (\tilde{D}^{(2)}H_2)^{\intercal}}{2}\right)
        &= (\lambda_1 + \alpha)(\lambda_2 - \alpha) - \left(\frac{\lambda_1 + \lambda_2}{2}\right)^2 \\
        &= -\alpha^2 - \alpha(\lambda_1 - \lambda_2) + \left(\lambda_1\lambda_2 - \left(\frac{\lambda_1 + \lambda_2}{2}\right)^2\right)\\
        &= -\left[\alpha^2 + 2\alpha \frac{\lambda_1 - \lambda_2}{2} + \left(\frac{\lambda_1 - \lambda_2}{2}\right)^2\right]\\
        &= -\left(\alpha + \frac{\lambda_1 - \lambda_2}{2}\right)^2.
      \end{align*}
  \end{itemize}
  For $\tilde{D}^{(2)}H_1$ and $\tilde{D}^{(2)}H_2$ to be p.s.d.\ in both case, i.e., the determinants to be non-negative, we must have
  \[
    \lambda_1 - \lambda_2 = 0 \quad \text{(isotropy)}, \qquad
    \alpha = 0 \quad \text{(symmetry)}.
  \]

  This implies that the considered principal submatrix of $\tilde{D}^{(2)} = \lambda I^{(2)}$ of $D$ is a scalar multiple of the identity matrix in the rotated basis. 
  Since the identity matrix is rotation-invariant, $D^{(2)} = Q^\intercal \tilde{D}^{(2)} Q = Q^\intercal \lambda I^{(2)} Q = \lambda Q^\intercal Q = \lambda I^{(2)}$ is a scalar multiple of the identity in the original basis. 
  As this result holds for any pair of indices $a, b \in [d]$, we conclude that $D = \lambda I$ for a $\lambda \in \R_+$.
\end{proof}

This structure, together with the assumed regularity assumptions, enables one to conclude, by integration, that the aggregation rule must be affine.
To prove that the result holds under the weak regularity assumption of absolute continuity on lines, we establish the following technical lemmas.
\begin{lemma}\label{lem:wwmonotone-localintegrability}
Let $F: (\mathbb{R}^d)^n \to \mathbb{R}^d$ be a input-wise monotone function. Then $F$ is locally bounded.
\end{lemma}
\begin{proof}
We proceed by induction on the number of inputs $n$.

\paragraph{Base Case ($n=1$).}
$F$ is a standard monotone operator defined on the entire space $\mathbb{R}^d$, therefore $F$ is locally bounded~\citep[Theorem 2]{Borwein1989LocalBO}.

\paragraph{Inductive Step.}
Assume local boundedness holds for $n-1$ inputs.
Let $K$ be a compact subset of $(\mathbb{R}^d)^n$. 
Let $K_1$ denote the projection of $K$ onto the first input's coordinate, and $K_{-1}$ the projection onto the remaining coordinates.
We construct a ``cage''  $P = \{p_1, \dots, p_m\} \subset \mathbb{R}^d$ consisting of $m\geq d+1$ affinely independent points, such that $K_1$ is contained in the interior of the convex hull of $P$.
We rely on the rwo geometric quantity determined by $P$ and $K$.
\begin{itemize}
  \item The maximum radius, $R \vcentcolon= \max_{k\in[m]} \max_{g_1 \in K_1} \|p_k - g_1\|_2 < \infty$.
  \item The minimum unormalized cosine measure, representing the ability of the cage to bound all directions,
  \[ 
    \gamma \vcentcolon= \min_{g_{1} \in K_1} \operatorname{ucm}(g_1), \text{ where } \operatorname{ucm}(g_1) \vcentcolon= \min_{\|v\|_2=1} \max_{k \in \{1,\dots,m\}} \langle v, p_k - g_1 \rangle. 
  \]
  The condition that $K_1 \subset \operatorname{int}(\operatorname{conv}(P))$ ensures that for any $g_1 \in K_1$, the vectors $\{p_k - g_1\}_{k \in [m]}$ positively span $\mathbb{R}^d$, or equivalently, that $\operatorname{ucm}(g_1) > 0$ for all $g_1 \in K_1$~\citep[Theorem 2.3]{conn2009introduction}.
  Finally, since $ucm$ is continuous and stricly positive on the compact $K_1$, it attains a minimum, which is strictly positive $\gamma > 0$.
\end{itemize}

Consider any configuration $g = (g_1, g_{-1}) \in K$. 
If $F(g) = 0$, the bound $\|F(g)\|_2 = 0$ holds trivially. 
Assume henceforth that $\|F(g)\|_2 > 0$.
For each vertex $p_k \in P$, monotonicity in the first variable implies
\[ 
  \langle F(g) - F(p_k, g_{-1}), g_1 - p_k \rangle \geq 0. 
\]
Rearranging terms,
\begin{equation}\label{eq:monotone-bounded-upperbound}
  \langle F(g), p_k - g_1 \rangle \leq \langle F(p_k, g_{-1}), p_k - g_1 \rangle. 
\end{equation}
By the definition of $\gamma$, there exists an index $k^*$ such that $\langle \frac{F(g)}{\|F(g)\|_2}, p_{k^*} - g_1 \rangle \geq \gamma$. 
Multiplying by $\|F(g)\|_2$, we get
\begin{equation}\label{eq:monotone-bounded-lowerbound}
    \langle F(g), p_{k^*} - g_1 \rangle \geq \gamma \|F(g)\|_2.
\end{equation}
Combining the lower bound~\ref{eq:monotone-bounded-lowerbound} with the upper bound~\ref{eq:monotone-bounded-upperbound} gives
\begin{equation}\label{eq:monotone-bounded-combined}
    \|F(g)\|_2 \leq \frac{1}{\gamma} \langle F(p_{k^*}, g_{-1}), p_{k^*} - g_1 \rangle.
\end{equation}

For each fixed $p_k$, $f_{p_{k}} \vcentcolon= F(p_{k}, \cdot)$ depends on $n-1$ variables and is input-wise monotone. 
By the inductive hypothesis, it is locally bounded. 
Since $K_{-1}$ is compact, there exists a constant $M_{n-1} < \infty$ such that $\max_{k\in[m]}\sup_{g'_{-1} \in K_{-1}} \|F(p_k, g'_{-1})\|_2 \leq M_{n-1}$.
Consequently, applying the Cauchy–Schwarz inequality to~\eqref{eq:monotone-bounded-lowerbound} yields
\[
  \|F(g)\|_2 \leq \frac{R}{\gamma} M_{n-1},
\]
where $R$, $\gamma$ and $M_{n-1}$ are constants independent of $g$. 
Since $g \in K$ was fixed arbitrarily, this proves uniform boundedness of $F$ on $K$. 
Finally, as $K$ was an arbitrary compact subset of $(\R^d)^n$, $F$ is locally bounded on $(\R^d)^n$.
\end{proof}

By showing that the ACL assumption, along with the other hypotheses, ensures Sobolev regularity of the aggregation, we deduce that it coincides almost everywhere with an affine aggregator.
\begin{comment}
\begin{lemma}\label{lem:ACL-equal-ae}
  Let $d \geq 2$ and $F: (\R^d)^n \to \R^d$ be ACL (cf. Definition~\ref{def:ACL}), input-wise monotone and input-wise almost everywhere differentiable (cf. Definition~\ref{def:workerwisemonotone}).
  If $F$ has worker-wise constant partial Jacobians almost everywhere, that is there exist $\alpha \in \R^n$ such that for all $i \in [n]$, for all $g_{-i} \in (\R^d)^{n-1}$, and for almost every $g_i \in \R^d$,
  \begin{equation*}
    D_i F(g_1, \ldots, g_n) = \alpha_i I_d,
  \end{equation*}
  then, for almost every $(g_1, \ldots, g_n) \in (\mathbb{R}^d)^n$,
  \[ 
    F(g) = \sum_{j=1}^n \alpha_j g_j + C \quad \text{where $C\in\R^d$ is constant of integration}.
  \]
\end{lemma}
\end{comment}
\begin{lemma}\label{lem:ACL-equal-ae}
  Let $d \geq 2$ and $F: (\R^d)^n \to \R^d$ be ACL (cf. Definition~\ref{def:ACL}), input-wise monotone and input-wise almost everywhere differentiable (cf. Definition~\ref{def:workerwisemonotone}).
  If $F$ has input-wise constant partial Jacobians almost everywhere (i.e., for all $i \in [n]$, for all $g_{-i} \in (\R^d)^{n-1}$, and for a.e.\ $g_i \in \R^d$, $D_i F(g_1, \ldots, g_n) = \alpha_i I_d$, where $\alpha \in \R^n$),
  then $F$ is input-wise affine almost everywhere (i.e., for a.e.\ $(g_1, \ldots, g_n) \in (\mathbb{R}^d)^n$, $F(g_1, \ldots, g_n) = \sum_{j=1}^n \alpha_j g_j + C$, where $C\in\R^d$).
\end{lemma}
\begin{proof}
Define the linear function $A: (\mathbb{R}^d)^n \to \mathbb{R}^d, g = (g_1, \ldots g_n) \mapsto \sum_{j=1}^n \alpha_j g_j$.
Consider the difference function $R \vcentcolon= F - A$.
Since $A$ is linear and $F$ is ACL, the difference $R$ is ACL.
By linearity of the derivative, the input-wise partial Jacobians of $R$ satisfies, for $i\in[n]$,
\[ D_i R = D_i F - D_i A = \alpha_i I_d - \alpha_i I_d = 0 \quad \text{almost everywhere}. \]
In other words, all partial derivatives of $R$ vanish almost everywhere.

To demonstrate that $R$ is constant almost everywhere, we establish that it belongs to the Sobolev space $W^{1,1}_{loc}((\mathbb{R}^d)^n; \mathbb{R}^d)$ (cf. Definition~\ref{def:sobolevfunc}).
First, we observe that $R \in L^1_{loc}((\mathbb{R}^d)^n; \mathbb{R}^d)$ because $F \in L^1_{loc}((\mathbb{R}^d)^n; \mathbb{R}^d)$, as ensured by Lemma~\ref{lem:wwmonotone-localintegrability}.
We now apply the characterization of Sobolev spaces via the ACL property.
According to~\citet[Theorem 11.45, applied locally]{leoni2017first}, a function belongs to $W^{(1,1)}_{loc}((\R^d)^n; \R^d)$ if and only if it is locally integrable and admits an ACL representative whose classical partial derivatives are locally integrable.
This condition is satisfied by $R$, which is ACL and has classical partial derivatives equal to zero almost everywhere (implying the are trivially locally integrable).
Thus $R \in W^{1,1}_{loc}((\mathbb{R}^d)^n; \mathbb{R}^d)$.
Finally, since the domain $(\R^d)^n$ is connnected, and observing that the weak partial derivatives of $R$ coincide with its classical derivaties almost everywhere~\citep[Theorem 11.45]{leoni2017first}, the fact that these derivatives vanish almost everywhere implies that $R$ is constant almost everywhere~\citep[see, e.g.,][Exercise 10.28]{leoni2017first}. %~\citep[Remark 7 following Proposition 9.3]{brezis2010functional}.
\end{proof}

The following measure-theoretic slicing property is a direct consequence of Fubini’s theorem and will be instrumental in our subsequent arguments.
\begin{lemma}\label{lem:measure-Fubini-slice}
Let $E \subset (\mathbb{R}^d)^n$ be a set of full Lebesgue measure. Let $g \in (\mathbb{R}^d)^n$ be decomposed as $g = (g_1, g_{-1})$, where $g_1 \in \mathbb{R}^d$ and $g_{-1} \in (\mathbb{R}^d)^{n-1}$.
Then, for almost every $g_{-1} \in (\mathbb{R}^d)^{n-1}$, the slice $E_{g_{-1}} = \{ u \in \mathbb{R}^d \mid (u, g_{-1}) \in E \}$ has full Lebesgue measure in $\mathbb{R}^d$.
\end{lemma}
\begin{proof}
Let $\lambda^{(k)}$ denote the Lebesgue measure on the product space $\mathbb{R}^k$, $k\in\mathbb{N}$.
Consider the indicator function $1_{(\mathbb{R}^d)^n \setminus E}: (\mathbb{R}^d)^n \to \{0, 1\}$. Then,
\[
  \int_{(\mathbb{R}^d)^n} 1_{(\mathbb{R}^d)^n \setminus E}(g) \, d\lambda^{(dn)}(g) = \lambda^{(dn)}((\mathbb{R}^d)^n \setminus E) = 0.
\]
Since $1_{(\mathbb{R}^d)^n \setminus E}$ is non-negative and measurable, Fubini's Theorem~\citep[Theorem B.58]{leoni2017first} allows us to compute this integral as an iterated integral
\begin{equation}\label{eq:fubini-intermediate}
  \int_{(\mathbb{R}^d)^{n-1}} \left( \int_{\mathbb{R}^d} 1_{(\mathbb{R}^d)^n \setminus E}(g_1, g_{-1}) \, d\lambda^{(d)}(g_1) \right) d\lambda^{(d(n-1))}(g_{-1}) = 0.
\end{equation}
Define the function $h: (\mathbb{R}^d)^{n-1} \to [0, \infty]$ by the inner integral,
\[
  h(g_{-1}) \vcentcolon= \int_{\mathbb{R}^d} 1_{(\mathbb{R}^d)^n \setminus E}(g_1, g_{-1}) \, d\lambda^{(d)}(g_1) = \lambda^{(d)}(\mathbb{R}^d \setminus E_{g_{-1}}).
\]
Substituting $h$ back into~\eqref{eq:fubini-intermediate}, we have
\[
  \int_{(\mathbb{R}^d)^{n-1}} h(g_{-1}) \, d\lambda^{(d(n-1))}(g_{-1}) = 0.
\]
We have a non-negative function $h$ whose integral is zero, hence $h$ must be zero almost everywhere.
As a result, for almost every $g_{-1} \in (\R^d)^{n-1}, \lambda^{(d)}(\mathbb{R}^d \setminus E_{g_{-1}}) = 0$.
\end{proof}

Finally, we upgrade the almost-everywhere equality to equality everywhere by combining the input-wise monotonicity of one function with the continuity of the other.
\begin{lemma}\label{lem:ACL-equal-everywhere}
Let $F: (\mathbb{R}^d)^n \to \mathbb{R}^d$ be a input-wise monotone function (cf. Definition~\ref{def:workerwisemonotone}) and $A: (\mathbb{R}^d)^n \to \mathbb{R}^d$ be a continuous function.
If $F = A$ almost everywhere, then $F = A$ everywhere.
\end{lemma}
\begin{proof}
We proceed by induction on the number of inputs $n$.

\paragraph{Base Case ($n=1$).}
The function $F$ is monotone on $\mathbb{R}^d$ (with only one input, input-wise monotonicity is standard monotonicity). 
Let $E = \{g \in \mathbb{R}^d : F(g) = A(g)\}$. 
Since the equality holds almost everywhere, $\mathbb{R}^d \setminus E$ has measure zero, implying $E$ is dense in $\mathbb{R}^d$.
Fix an arbitrary point $g' \in \mathbb{R}^d$.
By the monotonicity of $F$, for any $g \in E$,
\begin{equation}\label{eq:acl-eq-intermediate}
  \langle F(g) - F(g'), g - g' \rangle = \langle A(g) - F(g'), g - g' \rangle \geq 0. 
\end{equation}
Fix an arbitrary unit vector $v \in \mathbb{R}^d$ and a scalar $\epsilon > 0$.
By the density of $E$, there exists a sequence $\{ g_k \}_{k \in \mathbb{N}} \subset E$ converging to $g' + \epsilon v$.
Substituting $g_k$ into~\eqref{eq:acl-eq-intermediate}, taking the limit $k \to \infty$, and using the continuity of $A$,
\[ 
  \langle A(g' + \epsilon v) - F(g'), (g' + \epsilon v) - g' \rangle = \epsilon \langle A(g' + \epsilon v) - F(g'), v \rangle \geq 0. 
\]
Dividing by $\epsilon > 0$ and taking the limit $\epsilon \to 0$, again using the continuity of $A$,
\[ 
  \langle A(g') - F(g'), v \rangle \ge 0. 
\]
Since $v$ was arbitrary, we can repeat the argument for $-v$ to get $\langle A(g') - F(g'), -v \rangle \geq 0$.
Therefore, $\langle A(g') - F(g'), v \rangle = 0$ for all unit vector $v\in\R^d$, implying $F(g') = A(g')$.
\begin{comment}
%%%%%%%%%%%%%%%%%% PROOF WITH MAXIMAL MONOTONICITY? %%%%%%%%%%%%%%%%%%
Let $E = \{g \in \mathbb{R}^d : F(g) = A(g)\}$. 
Since $A$ is continuous and coincides with the monotone function $F$ on $E$, $A$ is monotone on $\mathbb{R}^d$.
Being continuous and monotone, $A$ is maximally monotone~\citep[Corollary 20.28]{Bauschke2011ConvexAA}.
Fix any $g' \in \mathbb{R}^d$. For any $g \in E$, monotonicity implies
\[
    \langle F(g') - A(g), g' - g \rangle = \langle F(g') - F(g), g' - g \rangle \geq 0.
\]
Since $E$ is dense and $A$ is continuous, this inequality extends to all $g \in \mathbb{R}^d$.
By the definition of maximal monotonicity, the only vector satisfying this condition against all $g$ is $A(g')$. 
Thus, $F(g') = A(g')$.
%%%%%%%%%%%%%%%%%%%%%%%%%%%%%%%%%%%%%%%%%%%%%%%%%%%%%%%%%%%%%%%%%%%%%%
\end{comment}

\paragraph{Inductive Step.}
Assume the theorem holds for $n-1$ inputs.
Let $g = (g_1, \ldots, g_n) \in (\mathbb{R}^d)^n$ be decomposed as $(g_1, g_{-1})$, where $g_1 \in \mathbb{R}^d$ is the first input's input and $g_{-1} \in (\mathbb{R}^d)^{n-1}$ represents the remaining inputs.
Let $E \subset (\mathbb{R}^d)^n$ be the set of full measure where $F = A$.

\textit{Step 1 (Extension along the first input, i.e., fixing $g_{-1}$).}
By Fubini's Theorem (cf. Lemma~\ref{lem:measure-Fubini-slice} for full details), there exists a set $\Omega \subset (\mathbb{R}^d)^{n-1}$ of full Lebesgue measure such that for any fixed configuration $g_{-1} \in \Omega$, the section $E_{g_{-1}} = \{ u \in \mathbb{R}^d : (u, g_{-1}) \in E \}$ has full measure in $\mathbb{R}^d$.
Fix an arbitrary $g_{-1} \in \Omega$. 
Define the partial function $f_{g_{-1}}: \mathbb{R}^d \to \mathbb{R}^d, u \mapsto F(u, g_{-1})$.
The function $f_{g_{-1}}$ is monotone (by input-wise monotonicity of $F$) and equals the continuous function $u \mapsto A(u, g_{-1})$ almost everywhere (specifically on $E_{g_{-1}}$).
By the base case $n=1$ above, $f_{g_{-1}}(u) = A(u, g_{-1})$ for all $u \in \mathbb{R}^d$.
Thus, $F(g_1, g_{-1}) = A(g_1, g_{-1})$ holds for all $g_1 \in \mathbb{R}^d$ whenever $g_{-1} \in \Omega$.

\textit{Step 2 (Extension to the remaining inputs, i.e., fixing $g_1$).}
Now, fix an arbitrary $g_1 \in \mathbb{R}^d$. 
Define the partial function $F_{g_1}: (\mathbb{R}^d)^{n-1} \to \mathbb{R}^d, g_{-1} \mapsto F(g_1, g_{-1})$.
The function $F_{g_1}$ inherits input-wise monotonicity from $F$.
From the result of Step 1, we know that $F_{g_1}(g_{-1}) = A(g_1, g_{-1})$ for all $g_{-1} \in \Omega$.
Since $\Omega$ has full measure in $(\mathbb{R}^d)^{n-1}$, $F_{g_1}$ matches the continuous function $g_{-1} \mapsto A(g_1, g_{-1})$ almost everywhere.
By the inductive hypothesis on $n-1$ inputs, applied to $F_{g_1}$, $F_{g_1}$ must match $A(g_1,\cdot)$ everywhere.
Therefore, $F_{g_1}(g_{-1}) = A(g_1, g_{-1})$ for all $g_{-1} \in (\mathbb{R}^d)^{n-1}$.
Since $g_1$ was arbitrary, we conclude $F = A$.
\end{proof}

Finally, we present an alternative proof of~Theorem~\ref{thm:universal_convexity}, by leveraging the lemmas established above.
In this proof, because we differentiate both the aggregation rule and the induced operator almost everywhere and then integrate back expressions involving the Jacobians of $F$, we require $F$ to be absolutely continuous on lines (ACL): for each coordinate, the restriction $F$ to compact intervals must be absolutely continuous for almost every fixed choice of the remaining coordinates (cf. Definition~\ref{def:AC}).
The ACL property is standard in real and functional analysis and is satisfied by broad classes of functions including Sobolev and locally Lipschitz functions. 
It originates in the analysis of Sobolev functions, where ACL plays a key role in relating weak and pointwise differentiability \citep[cf.][Section 11.3]{leoni2017first}.
In particular, the ACL assumption is required only to justify the very last integration step in the proof, since monotonicity alone only guarantees bounded variations but does not rule out singular behaviors such as Cantor-type singular functions \citep[cf.][Paragraph 5.6]{albamb96}.
%\tbtodo{
%    Add examples? 
%    (1) an ACL function can go to infinity at specific points (e.g.\ $\ln(1 / \|x\|_2)$);
%    (2) an ACL function can be discontinuous at a point (e.g.\ projection onto the unit circle $x / \|x\|_2$).
%    (3) an ACL function can be non Lipschitz continuous (e.g.\ $\sqrt{\|x\|_2}$). 
%    Also, maybe it is a discussion to place after the theorem.
%}

\begin{theorem}\label{thm:universal_convexity-acl}
    Let $d>1$, and consider the model update rule~\eqref{eq:algo-update} with $F: (\R^d)^n \to \R^d$ an absolute continuous on lines (ACL) aggregation rule.  
    $\F$ is a monotonic operator for all gradients $\{g_i(\cdot)\}_{i\in[n]}$ induced by convex functions
    if and only if $F$ is an affine aggregator, i.e., there exists $\alpha \in \mb{R}_+^n$ and $C \in \mb{R}^d$, independent of the inputs, such that for any input vectors $(g_1, \ldots, g_n) \in (\mb{R}^d)^n$,
    \[
        F(g_1, \ldots, g_n) = \sum_{k=1}^n \alpha_k g_k + C.
    \]
\end{theorem}
\begin{proof}
    In what follows, we only use differentiability in the Fréchet sense, and consider the Lebesgue measure over the Borel $\sigma$-algebra.
 
    First, if $F$ is an affine aggregator, it is immediate that $\F$ preserves monotonicity.
    We now turn to the converse. The proof proceeds in three steps.
    
    \paragraph{\textit{\textbf{Step 1.}}}
    $\F$ must be differentiable almost everywhere and its Jacobian, when it exists, must be positive semi-definite.

    Indeed, by Lemma~\ref{lem:monotone-differentiable}, $\F$ is almost everywhere differentiable for any set of convex function $\{\ell_i\}_{i\in[n]}$.
    Hence, fixing an arbitrary set of functions $\{\ell_i\}_{i\in[n]}$, for almost every $\theta \in \R^d$, $\F$ admits the following Taylor expansion for arbitrary $u \in \R^d$ and $\varepsilon \in \R$,
    \[
      \F(\theta + \varepsilon u) - \F(\theta) = \varepsilon \nabla\F(\theta)u + r(\varepsilon u),
    \]
    where the remainder $r$ satisfies $\lim_{\varepsilon \to 0} \frac{\|r(\varepsilon u)\|_2}{\varepsilon} = 0$. 
    Then, by the monotonicity of $\F$, we have
    \begin{equation}\label{eq:taylor-monotonicity-intermediate}
      \langle \varepsilon u , \varepsilon \nabla\F(\theta)u + r(\varepsilon u) \rangle 
      = \varepsilon^2 \langle u , \nabla\F(\theta)u \rangle + \varepsilon \langle u, r(\varepsilon u) \rangle 
      \geq 0.
    \end{equation}
    Hence, by dividing~\eqref{eq:taylor-monotonicity-intermediate} by $\varepsilon^2$ and taking the limit as $\varepsilon \to 0$, we obtain
    \begin{equation*}
      \langle u , \nabla\F(\theta)u \rangle \geq 0.
    \end{equation*}
    That is, the jacobian $\nabla\F(\theta)$ is positive semi-definite but not necessarily symmetric.
    Equivalently, its symmetric part verifies
    \begin{equation}\label{eq:psdjacobianoperator}
      \frac{\nabla\F(\theta) + \nabla\F(\theta)^\intercal}{2} \succeq 0.
    \end{equation}
    
    \paragraph{\textit{\textbf{Step 2.}}}
    For any input index $i \in [n]$ and any configuration $g_{-i} \vcentcolon= (\ldots, g_{i-1}, g_{i+1}, \ldots) \in (\R^d)^{n-1}$, the partial mapping $f_i(\theta) \vcentcolon= F(\ldots, g_{i-1}, \theta, g_{i+1}, \ldots)$ must be monotone and differentiable almost everywhere.
    We refer to a function satisfying these conditions as input-wise monotone and input-wise almost everywhere differentiable.

    Indeed, fix an arbitrary index $i\in[n]$ and arbitrary vectors $g_{-i} \in (\R^d)^{n-1}$. 
    Define the functions 
    \[
      \ell_i(\theta) = \frac{1}{2} \theta^\intercal \theta \quad\text{and for } j \neq i,\quad \ell_j(\theta) = g_j^\intercal \theta. 
    \]
    We have, $\F(\theta) = f_i(\theta)$. 
    Therefore, $f_i$ inherits properties from $\F$, i.e., it is monotone and almost everywhere differentiable.

    \paragraph{\textit{\textbf{Step 3.}}} $F$ must be an affine aggregator.
    \begin{enumerate}
        \item[\textit{(i)}]
        We now study the implications of Step 1 and Step 2. 
        At any point $\theta$ where $\F$ is differentiable and $F$ admits partial Jacobians at $(\nabla \ell_1(\theta), \ldots, \nabla \ell_n(\theta))$, the chain rule gives
        \begin{equation}\label{eq:chain-rule-app2}
            \nabla \F(\theta) = \sum_{i=1}^n D_i F(\nabla \ell_1(\theta), \ldots, \nabla \ell_n(\theta))\ \nabla^2 \ell_i(\theta),
        \end{equation}
        where, for $i \in [n]$, $D_i F$ is the $d \times d$ partial Jacobian of $F$ with respect to its $i$-th input $g_i$.
        Fix an arbitrary index $i \in [n]$ and an arbitrary configuration $g_{-i} \in (\R^d)^{n-1}$. 
        Let $g_i \in \R^d$ be chosen from the set of full measure where the partial Jacobian of $F$ with respect to the $i$-th input exists.
        Let $H$ be an arbitrary symmetric p.s.d.\ matrix, chosen independently from $(g_1, \ldots, g_n)$. 
        Define the following convex loss functions,
         \[
          \ell_j(\theta) = g_j^{\intercal} \theta \quad\text{for } j \in [n]\setminus\{i\}, \quad\text{and}\quad \ell_i(\theta) = g_i^{\intercal} \theta + \frac{1}{2} \theta^\intercal H \theta. 
        \]
        At $\theta = 0$ and for each $j\in[n]$, $\nabla \ell_j(0) = g_j$ and $\nabla^2 \ell_j(0) = 0$ for $j \neq i$ or $H$ for $j=i$.
        Since $D_i F (g_1, \ldots, g_n)$ exists and $\nabla^2 \ell_j(0) = 0$ for $j \neq i$, eq.~\ref{eq:chain-rule-app2} implies that $\F$ is differentiable at $0$, with derivative given by
        \begin{equation}\label{eq:chain-rule-uno}
            \nabla \F(0) = D_{i} F(g_1, \ldots, g_n) H.
        \end{equation}
        Consequently $D_{i} F(g_1, \ldots, g_n) H$ must be positive semi-definite, an implication that holds for any symmetric p.s.d.\ $H$. 
        Therefore, by Lemma~\ref{lem:psd-crux}, for every $i \in [n]$, every configuration $g_{-i} \in (\R^d)^{n-1}$, there must exist a scalar function $\alpha_i$ such that for almost every $g_i \in \R^d$, 
        \begin{equation}\label{eq:intermediate-diagonal-cvx}
            D_i F(g_1, \ldots, g_n) = \alpha_i(g_1, \ldots, g_n) I_d.
        \end{equation}

        \item[\textit{(ii)}]
        Let $F = (F_1, \ldots, F_d)^\intercal$ the $d$ scalar component functions of $F$ and $g_i = (g_{i1}, \ldots, g_{id})^\intercal$, where we denote $g_{ik}$ the $k-$th component of the vector $g_i$. 
        The $(k, l)$-th entry of the $d \times d$ matrix $D_i F$ is $\frac{\partial F_k}{\partial g_{il}}$. 
        \eqref{eq:intermediate-diagonal-cvx} implies that for almost every $(g_1, \ldots, g_n) \in (\R^d)^n$,
        \begin{enumerate}
            \item[\textit{(a)}] for off-diagonal ($l \neq k$) terms: $\frac{\partial F_k}{\partial g_{il}}(g_1, \ldots, g_n) = 0$;
            \item[\textit{(b)}] for diagonal ($l = k$) terms: $\frac{\partial F_k}{\partial g_{ik}}(g_1, \ldots, g_n) = \alpha_i(g_1, \ldots, g_n)$.
        \end{enumerate}
        From \textit{(a)} and $i \in [n]$, the $k$-th component $F_k$ does not depend on the $l$-th component of $g_i$ when $l \neq k$. 
        Since this holds for all $i$, $F_k$ can only depend on the $k$-th component of all input vectors, i.e., $(g_{1k}, \ldots, g_{nk})$.
        This means $F$ is a component-wise aggregator, i.e., there exist functions $\phi_k: \R^n \to \R$ such that, almost everywhere
        \begin{equation}\label{eq:componentwise-intermediate}
            F_k(g_1, \ldots, g_n) = \phi_k(g_{1k}, \ldots, g_{nk}).
        \end{equation}
        From \textit{(b)} and eq.~\ref{eq:componentwise-intermediate} we have, for $i \in [n]$, $k \neq l \in [d]$, and almost every $(g_1, \ldots, g_n) \in (\R^d)^n$,
        \[
            \frac{\partial \phi_k}{\partial g_{ik}}(g_{1k}, \ldots, g_{nk}) 
            = \alpha_i(g_1, \ldots, g_n)
            =  \frac{\partial \phi_l}{\partial g_{il}}(g_{1l}, \ldots, g_{nl}).
        \]
        Since the left-hand side only depends on the variables $\{ g_{ik} \}_{i\in[n]}$ and the right-hand side depends exclusively on the disjoint set $\{ g_{il} \}_{i\in[n]}$ both terms must be independent of their respective arguments. 
        Therefore, $\alpha_i(g_1, \ldots, g_n)$ must be a constant function (we denote this constant $\alpha_i$, slightly overloading notation).
        Therefore, for almost every $(g_1, \ldots, g_n) \in (\R^d)^n$
        \begin{equation}\label{eq:jacobians-constant-a.e.}
            D_i F(g_1, \ldots, g_n) = \alpha_i I_d. 
        \end{equation}
        
        \item[\textit{(iii)}]
        Let denote $A: (\R^d)^n \to \R^d, (g_1, \ldots, g_n) \mapsto \sum_{i=1}^n \alpha_i g_i + C$, where $C \vcentcolon= F(0, \ldots, 0) \in\R^d$.
        Under the ACL assumption,~\eqref{eq:jacobians-constant-a.e.} together with the input-wise montonicity of $F$ imply, via Lemma~\ref{lem:ACL-equal-ae}, that $F=A$ almost everywhere, and Lemma~\ref{lem:ACL-equal-everywhere} upgrades this to equality everywhere.
    \end{enumerate}
\end{proof}

Notably, and as mentioned in~Section~\ref{sec:impossinility}, an additional structural assumption to make the preservation of monotonicity (and co-coercive inequality) possible with non-affine aggregation rule is that the loss function is coordinate-wise separable 
(e.g., $\operatorname{CWTM}$, cf.\ Lemma~\ref{sec:positive-result-tm}).
%(e.g., $\operatorname{CWTM}$, cf.\ Lemma~\ref{lemma-cwtm-1d-nonexpansive} and Remark~\ref{cwtm-1d}).
We recover the underlying reason also in the alternative proof: the key step~\eqref{eq:chain-rule-uno} relies on the ability to choose arbitrary symmetric p.s.d.\ matrices as Hessians of the underlying loss function. 
When the structure of the loss function restricts this flexibility---for example, when it is coordinate-wise separable and thus admits only diagonal Hessians---the argument no longer applies.

More generally, the condition stated in Lemma~\ref{lem:psd-crux} can be of great value for two reasons. 
First, it provide insights into the necessary assumptions that a loss function must satisfy to guarantee that monotonicity could be preserved by any aggregation rule.
Second, it also shed light on the structural properties that candidate aggregation rules must exhibit to preserve monotonicity under these assumptions when integrated.        

Interestingly, the alternative proof extend to the case when we consider the co-coercivity inequality. 
That is, the aggregated operator $\F(\theta) = F(g_1(\theta), \ldots, g_n(\theta))$, preserves the co-coercive inequality of every convex and $L$-smooth loss function gradients $\{ g_i \}_{i \in [n]}$, i.e., there exists $L_\F > 0$ such that for all $\theta, \omega \in \Theta$:
\begin{equation*}
	\langle \theta - \omega, \F(\theta) - \F(\omega) \rangle \geq \frac{1}{L_\F} \| \F(\theta) - \F(\omega) \|_2^2,
\end{equation*} 
if and only $F$ is an affine aggregator.
Indeed, in our proof, we only used convex and smooth loss functions. Moreover, we no longer require to assume $F$ ACL as, by the Cauchy-Schwarz inequality, $\F$ is $L_\F-$Lipschitz continuous, a property that extend to $F$.
%An important regularity property is Lipschitz continuity, which in particular implies absolute continuity. 
%Moreover, if $\F$ is co-coercive, then by the Cauchy-Schwarz inequality it is Lipschitz continuous, a condition that forces $F$ itself to be Lipschitz continuous.
\begin{lemma}\label{lem:lipschitz-from-F}
    Let $d>1$, and consider the model update rule~\eqref{eq:algo-update}.
    If the aggregated operator $\F$ is Lipschitz continuous for every input gradients $\{ g_i \}_{i \in [n]}$ induced by $L$-smooth functions,
    then $F$ is itself Lipschitz continuous. 
\end{lemma}
\begin{proof}
    Let us show this by contradiction. Assume $F$ is not Lipschitz continuous. 
    This implies that $F$ must be non-Lipschitz with respect to at least one of its arguments. 
    In fact, if $F$ were separately Lipschitz in all arguments, it would be globally Lipschitz. 
    This is due to the fact that, we can construct a telescoping path from any vector $G = (g_1, \ldots, g_n)$ to any other vector $G' = (g_1', \ldots, g_n')$, with $G^{(i)} = (g_1, \ldots, g_i', \ldots, g_n')$, $i\in[n]$,
    \begin{multline*}
        \| F(G) - F(G') \|_2 = \| \sum_{i=1}^n F(G^{(i-1)}) - F(G^{(i)}) \|_2 
        \leq \sum_{i=1}^n \| F(G^{(i-1)}) - F(G^{(i)}) \|_2 \\
        \leq \sum_{i=1}^n L_i \| G^{(i-1)} - G^{(i)} \|_2 
        \leq \max_{i\in[n]} L_i \sum_{i=1}^n \| g_i - g'_i \|_2 
        = \max_{i\in[n]} L_i \sum_{i=1}^n \sqrt{\| g_i - g'_i \|_2^2} \\
        \leq \big(\sqrt{n} \max_{i\in[n]} L_i \big) \| G - G' \|_2
    \end{multline*} 
    Let's assume $F$ is not Lipschitz in its $i$-th argument. 
    This means that there exist a set of $n-1$ vector $g_{-i}:=(g_1,\ldots,g_{i-1},g_{i+1},\ldots,g_n)$ such that $F_{g_{-i}}:g\rightarrow F(g_1,\ldots,g_{i-1},g,g_{i+1},\ldots,g_n)$ is not Lipschitz i.e., $\forall M>0$, $\exists g_i,g_i'$ with
    \begin{equation}
        \label{eq:non-lip}
        \| F(G) - F(G') \|_2 = \|F_{g_{-i}}(g_i) - F_{g_{-i}}(g_i')\| > M \| g_i - g_i' \|_2 = M \| G - G' \|_2 ,
    \end{equation}
    where $G = (g_1, \ldots, g_i, \ldots, g_n)$ and $G' = (g_1, \ldots, g_i', \ldots, g_n)$.

    Based on $g_{-i}$, we can construct the following set of $L$-smooth convex functions: for $j \ne i$, let $\ell_j(\theta) = g_j^T \theta$, and $\ell_i(\theta) = \frac{L}{2} \| \theta \|_2^2$. 
    Using these functions, the aggregated operator $\F(\theta)$ becomes
    \[
        \F(\theta) = F(g_1, \ldots, L\theta, \ldots, g_n).
    \]

    From Step 1, we know that there exist a constant $L_\F$ such that $\F$ is $L_\F$-Lipschitz. Fix $M=\frac{L_\F}{L}$ and let $g_i,g_i'$ be the associated vector from the condition~\eqref{eq:non-lip}. By the $L_{\F}$-Lipschitz condition on $\F$ we have
    \[
        \| F(G) - F(G') \|_2 = \| \F(\frac{1}{L}g_i) - \F(\frac{1}{L}g_i') \|_2 \le \frac{L_{\F}}{L} \| g_i - g_i' \|_2.
    \]
    Combining this with~\eqref{eq:non-lip}, we get
    \[
        \frac{L_{\F}}{L} \| g_i - g_i' \|_2 < \| F(G) - F(G') \|_2 \leq \frac{L_{\F}}{L} \| g_i - g_i' \|_2,
    \]
    which simplifies to
    \[
        \frac{L_{\F}}{L} < \frac{L_{\F}}{L}.
    \]
    This is a contradiction, proving that $F$ must be Lipschitz continuous.
\end{proof}

\section{Component-Wise Trimmed Mean Properties}\label{tm}

In this section, we establish properties of the Component-Wise Trimmed Mean aggregation rule ($\mathrm{CWTM}$) useful for analyzing algorithmic stability of the update~\eqref{eq:algo-update} when $F=\operatorname{CWTM}$.

To establish non-expansiveness under coordinate-wise separable loss functions, we rely on the following technical lemma, which is central to the proof of Lemma~\ref{lemma-cwtm-1d-nonexpansive}.
\begin{lemma}~\label{trimmed-mean-lowerbound}    
    Let $x, x' \in \mb{R}^n$. There exist indices $\alpha, \beta \in \{1, \ldots, n\}$ such that: 
    \[
        {\left[ x - x' \right]}_\beta \leq \mathrm{TM}(x) - \mathrm{TM}(x') \leq {\left[ x - x' \right]}_\alpha.
    \]
\end{lemma}
%\begin{proof}
\textbf{Proof}\;
    Let $n \in \mb{N}$, $f \in \{0, \ldots, \floor{\frac{n}{2}}\}$, and $x, x' \in \mb{R}^n$
    \begin{itemize}
        \item[\textit{(i)}] Suppose $\mathrm{TM}(x) = \mathrm{TM}(x')$. We then consider two cases:
            \begin{itemize}
                \item If there exists an index $i \in \{1, \ldots, n\}$ such that $x_i = x'_i$, the lemma is trivially true with $\alpha = \beta = i$.
                \item Otherwise, $x_i \neq x'_i, \forall i \in \{1, \ldots, n\}$. In this case, there must be at least one index
                $\beta$ where ${\left[x - x' \right]}_\beta < 0$ and at least one index $\alpha$ where ${\left[x - x' \right]}_\alpha > 0$.
                If all differences were strictly positive or strictly negative, this would contradict the assumption that $\mathrm{TM}(x) = \mathrm{TM}(x')$.
            \end{itemize}
        \item[\textit{(ii)}] Suppose $\mathrm{TM}(x) \neq \mathrm{TM}(x')$. Here, we invoke the translation equivariance property of the trimmed mean. Specifically,  
        let $\mathbf{1} = (1, \ldots, 1) \in \mb{R}^n$, and define the translation operator $\mc{T}_{a\mathbf{1}}(x) \vcentcolon= x - a\mathbf{1}$ for any $a \in \mb{R}$.
        Then, for any $a \in \mb{R}$, the trimmed mean satisfies:
        \[
            \mathrm{TM}(\mc{T}_{a\mathbf{1}}(x)) = \mathrm{TM}(x - a\mathbf{1}) = \mathrm{TM}(x) - a
        \]
        Substituting $a = \mathrm{TM}(x) - \mathrm{TM}(x')$, we have \( \mathrm{TM}(\mc{T}_{a\mathbf{1}}(x)) - \mathrm{TM}(x') = 0 \). We then invoke the first case: there exist indices $\alpha, \beta \in \{1, \ldots, n\}$ such that: 
        \begin{align*}
            &{\left[ \mc{T}_{a\mathbf{1}}(x) - x' \right]}_\beta \leq  \mathrm{TM}(\mc{T}_{a\mathbf{1}}(x)) - \mathrm{TM}(x')  \leq  {\left[ \mc{T}_{a\mathbf{1}}(x) - x' \right]}_\alpha , \\
            & \implies  {\left[ x - x' \right]}_\beta - a  \leq  \mathrm{TM}(x) - \mathrm{TM}(x') - a  \leq  {\left[ x - x' \right]}_\alpha - a ,\\
            & \implies {\left[ x - x' \right]}_\beta  \leq  \mathrm{TM}(x) - \mathrm{TM}(x')  \leq  {\left[ x - x' \right]}_\alpha .
        \end{align*}
        This concludes the proof. \hfill$\BlackBox$
    \end{itemize}
%\end{proof}
Lemma~\ref{lemma-cwtm-1d-nonexpansive} contrasts with the general convex and strongly convex smooth settings, where, as illustrated in Remark~\ref{cwtm-2d-expansivity}, a non-expansive update cannot, in general, be guaranteed. 
\begin{remark}\label{cwtm-2d-expansivity}
    To establish, in general, a non-expansivity result for the $\mathrm{CWTM}$ update rule, it would be necessary that the scalar product in
    \begin{multline*}
        \| G^{\mathrm{CWTM}}_{\gamma}(\theta) - G^{\mathrm{CWTM}}_{\gamma}(\omega) \|_2^2  \\
        = \| \theta - \omega \|_2^2 
        + \gamma^2 \| \mathrm{CWTM}(\nabla \ell_1 (\theta), \ldots, \nabla \ell_n (\theta)) - \mathrm{CWTM}(\nabla \ell_1 (\omega), \ldots, \nabla \ell_n (\omega)) \|_2^2 \\
        - 2 \gamma \langle \mathrm{CWTM}(\nabla \ell_1 (\theta), \ldots, \nabla \ell_n (\theta)) - \mathrm{CWTM}(\nabla \ell_1 (\omega), \ldots, \nabla \ell_n (\omega)) , \theta - \omega \rangle
    \end{multline*}
    is strictly positive, that is the $\mathrm{CWTM}$ preserve the co-coercivity of the underlying loss function. 
    As we will illustrate bellow, this is false in general. 
    This is because the sets of selected indices differ across dimensions---that is $S^{(k)} \neq S^{(k)\prime}$ with $k \in \{1, \ldots, d\}$---which disrupts the alignment needed to preserve the co-coercivity property.

    Below, we present an example illustrating that $\mathrm{CWTM}$ does not preserve the co-coercivity property of an underlying convex and smooth loss function.
    Specifically, let $d=2$, $n=3$ and $f=1$. We explicitly construct the example illustrated in~Figure~\ref{fig:cwtm-noineq}. We consider: 
    \begin{itemize}
        \item the convex and smooth loss function:
            \begin{align*}
            \begin{split}
                \forall w \in \mb{R}^d, (x,y) \in B(0, \sqrt{L}) \times \mb{R}: \quad
                & \ell(w, (x, y)) = \frac{1}{2} {(w^T x - y)}^2 \\
                L \in R_+^* \quad \quad \quad \quad \quad \quad \quad 
                & \nabla \ell(w, (x, y)) = (w^T x - y) x\\
                & 0 \preceq  \nabla^2 \ell(w, (x, y)) = xx^T \preceq L I_d
            \end{split}
            \end{align*}
        
        \item the sample pool composed by $z_1 = (v, 0)$, $z_2 = (x, 0)$ and $z_3 = (x, 1)$ for $x, v \in B(0, \sqrt{L})$. 

        \item the parameters $\theta =  \frac{x}{L}$ and $\omega = 0$.
    \end{itemize}
    
    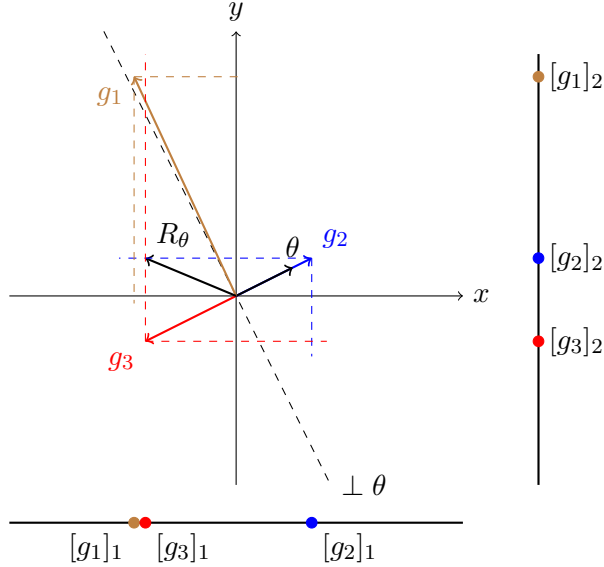
\begin{figure}[ht!]
    \begin{center}
    \begin{tikzpicture}
        % Draw the orthogonal line to g_E
        \draw[dashed] (-1.75, 3.5) -- (1.25, -2.5) node[right, thick] {$\perp \theta$};

        % Draw axes
        \draw[->] (-3, 0) -- (3, 0) node[right] {$x$};
        \draw[->] (0, -2.5) -- (0, 3.5) node[above] {$y$};
        
        % Define the vectors
        \coordinate (g_E) at (1, 0.5);
        \coordinate (g_F) at (-1.2, -0.6);
        \coordinate (g_1) at (-1.35, 2.9);
        \coordinate (theta_tau) at (0.75, 0.375);
        \coordinate (R_tau) at (-1.2, 0.5);
        
        % Draw the vectors
        % Chosen vector
        \draw[->, thick, blue] (0,0) -- (g_E) node[above right] {$g_2$};
        \draw[->, thick, red] (0,0) -- (g_F) node[below left] {$g_3$};
        \draw[->, thick, black] (0,0) -- (theta_tau) node[above] {$\theta$};
        \draw[->, thick, brown] (0,0) -- (g_1) node[below left] {$g_1$};
        \draw[->, thick, black] (0,0) -- (R_tau) node[above right] {$R_\theta$};
        
        % Draw dashed lines for coordinate value
        \draw[dashed, blue] (g_E) -- (1, -0.8) node[midway, right] {};
        \draw[dashed, blue] (g_E) -- (-1.55, 0.5) node[midway, above left] {};
        \draw[dashed, red] (g_F) -- (-1.2, 3.2) node[midway, right] {};
        \draw[dashed, red] (g_F) -- (1.2, -0.6) node[midway, below left] {};
        \draw[dashed, brown] (g_1) -- (-1.35, -0.1) node[midway, right] {};
        \draw[dashed, brown] (g_1) -- (0.1, 2.9) node[midway, above left] {};

        \draw[thick] (4, 3.2) -- (4, -2.5);
        \draw[fill, brown] (4, 2.9) circle (2pt);
        \node[right] at (4, 2.9) {\([g_1]_2\)};
        \draw[fill, red] (4, -0.6) circle (2pt);
        \node[right] at (4, -0.6) {\([g_3]_2\)};
        \draw[fill, blue] (4, 0.5) circle (2pt);
        \node[right] at (4, 0.5) {\([g_2]_2\)};

        \draw[thick] (-3, -3) -- (3, -3);
        \draw[fill, brown] (-1.35, -3) circle (2pt);
        \node[below left] at (-1.35, -3) {\([g_1]_1\)};
        \draw[fill, red] (-1.2, -3) circle (2pt);
        \node[below right] at (-1.2, -3) {\([g_3]_1\)};
        \draw[fill, blue] (1, -3) circle (2pt);
        \node[below right] at (1, -3) {\([g_2]_1\)};
    \end{tikzpicture}
    \caption{$\mathrm{CWTM}$ does not preserve the co-coercive inequality of smooth and convex functions.}\label{fig:cwtm-noineq}
    \end{center}
    \end{figure}
    We formally have: $g_1 = \theta^T v v$, $g_2 = \theta^T x x = \| x \|_2^2 \theta$ and $g_3 = \underbrace{(\| x \|_2^2 - L)}_{<0} \theta$.
    We next examine the conditions under which the $\mathrm{CWTM}$ aggregation rule does not preserve the co-coercivity inequality. Specifically, we denote:
    $R_\theta = \mathrm{CWTM}\left( g_1, g_2, g_3 \right)$ and $R_\omega = \mathrm{CWTM}\left( 0, 0, -x \right) = 0$:
    \begin{equation*}
        \langle \theta - \omega , R_\theta - R_\omega \rangle
        = \langle \theta , R_\theta \rangle 
        < 0
    \end{equation*}
    Instances of such $v, x$ yielding this inequality are illustrated in Figure~\ref{fig:cwtm-noineq} ($v$ is colinear to $g_1$ and $x$ colinear to $g_2$).
\end{remark}

Next, we analyze the Lipschitz continuity of the Trimmed Mean ($\mathrm{TM}$) operation on the real line, which is useful for showing the expansivity property of $\mathrm{CWTM}$ robust update.
%\begin{comment}
\begin{lemma}~\label{tm-lipschitz}
    The trimmed mean operation is $\frac{1}{n-2f}$-Lipschitz continuous with respect to the $\|\cdot\|_1$ norm, 
    and $\frac{\sqrt{n}}{n-2f}$-Lipschitz continuous with respect to the $\|\cdot\|_2$ norm. 
\end{lemma}
%\begin{proof}
\textbf{Proof}\;
We recall the definition of the trimmed mean operation: given $f, n \in \mb{N}$, $f < n/2$, and real values $x_1, \ldots, x_n$, 
    we denote by $\sigma$ a permutation on $\{1, \ldots, n\}$ that sorts the values in non-decreasing order 
    $x_{\sigma(1)} \leq \cdots \leq x_{\sigma(n)}$. Let $S_x=\{i \in [n]; f+1 \leq \sigma(i) \leq n-f\}$, 
    the trimmed mean is given by: 
    \[
    \operatorname{TM}(x_1, \ldots, x_n) 
    = \frac{1}{n-2f} \sum_{i=f+1}^{n-f} x_{\sigma(i)} 
    = \frac{1}{n-2f} \sum_{i \in S_x} x_i
    \]

We first consider two sets of real numbers that differ in exactly one entry 
$\iota \in \{1, \ldots, n\}$: $x_1, \ldots, x_\iota, \ldots, x_n$ and $x'_1, \ldots, x'_\iota, \ldots, x'_n$ 
such that $x_j = x'_j$ for all $j \neq \iota$. Let $S_x$ and $S_{x'}$ denote the sets of selected indices in each case.
We necessarily have $| S_{x'} \setminus S_{x} | \leq 1$. 
In the figure below, we illustrate different scenarios to provide an intuition for our proof.
\begin{figure}[ht]
    \centering
    \begin{tikzpicture}
        \draw[thick] (-1,0) -- (11,0);

        \draw[fill, red!50] (0,0) circle (2pt);
        \node[below] at (0,0) {\(x_1\)};

        \draw[fill, red!50] (2,0) circle (2pt);
        \node[below] at (2,0) {\(x_2\)};
        \draw[dashed] (2,-3) -- (2,0.5);
        \node[below] at (2,1) {\(\underline{\iota}\)};

        \draw[fill, green!50] (4,0) circle (2pt);
        \node[below] at (4,0) {\(x_3\)};

        \draw[fill, green!50] (6,0) circle (2pt);
        \node[below] at (6,0) {\(x_4\)};

        \draw[fill, red!50] (8,0) circle (2pt);
        \node[below] at (8,0) {\(x_5\)};
        \draw[dashed] (8,-3) -- (8,0.5);
        \node[below] at (8,1) {\(\overline{\iota}\)};

        \draw[fill, red!50] (10,0) circle (2pt);
        \node[below] at (10,0) {\(x_6\)};

        \draw[thick] (5,0.1) -- (5,-0.1);
        \node[below] at (5,0) {$\mathrm{TM}(x)$};

        \draw[thick] (-1,-1.5) -- (11,-1.5);
        \node[above] at (-2,-1.5) {\(x_4 \neq x_4'\)};

        \draw[fill, red!50] (0,-1.5) circle (2pt);
        \node[below] at (0,-1.5) {\(x_1\)};

        \draw[fill, red!50] (2,-1.5) circle (2pt);
        \node[below] at (2,-1.5) {\(x_2\)};

        \draw[fill, green!50] (4,-1.5) circle (2pt);
        \node[below] at (4,-1.5) {\(x_3\)};

        \draw[fill, red!50] (9,-1.5) circle (2pt);
        \node[below] at (9,-1.5) {\(x_4'\)};

        \draw[fill, green!50] (8,-1.5) circle (2pt);
        \node[below] at (8,-1.5) {\(x_5\)};

        \draw[fill, red!50] (10,-1.5) circle (2pt);
        \node[below] at (10,-1.5) {\(x_6\)};

        \draw[thick] (6,-1.5+0.1) -- (6,-1.5-0.1);
        \node[below] at (6,-1.5) {$\mathrm{TM}(x')$};
        \draw[dashed, <->] (5,-1.3) -- (6,-1.3) node[midway, above] {};
        \draw[dashed, <->] (6,-1.1) -- (9,-1.1) node[midway, above] {};

        \draw[thick] (-1,-2.5) -- (11,-2.5);
        \node[above] at (-2,-2.5) {\(x_6 \neq x_6'\)};

        \draw[fill, red!50] (0,-2.5) circle (2pt);
        \node[below] at (0,-2.5) {\(x_1\)};

        \draw[fill, green!50] (2,-2.5) circle (2pt);
        \node[below] at (2,-2.5) {\(x_2\)};

        \draw[fill, green!50] (4,-2.5) circle (2pt);
        \node[below] at (4,-2.5) {\(x_3\)};

        \draw[fill, red!50] (6,-2.5) circle (2pt);
        \node[below] at (6,-2.5) {\(x_4\)};

        \draw[fill, red!50] (8,-2.5) circle (2pt);
        \node[below] at (8,-2.5) {\(x_5\)};

        \draw[fill, red!50] (-1,-2.5) circle (2pt);
        \node[below] at (-1,-2.5) {\(x_6'\)};

        \draw[thick] (3,-2.5+0.1) -- (3,-2.5-0.1);
        \node[below] at (3,-2.5) {$\mathrm{TM}(x')$};
        \draw[dashed, <->] (3,-2.3) -- (5,-2.3) node[midway, above] {};
        \draw[dashed, <->] (-1,-2.1) -- (10,-2.1) node[midway, above] {};
    \end{tikzpicture}
\caption{Example with $n=6$, $f=2$ and the differing point being either $x_6$ or $x_4$. Horizontal dashed lines represent either trimmed mean differences or changed-point differences. Green points indicate the contributors to the trimmed-mean, while red points denote those excluded from the computation.}
\end{figure}
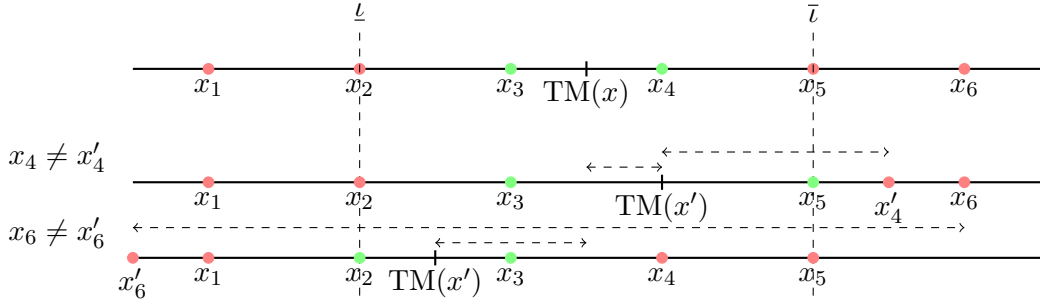

We consider the two possible cases.
\begin{itemize}
    \item Case (i) $| S_{x'} \setminus S_{x} | = 1$. Let $x_{\underline{\iota}}$ and $x_{\overline{\iota}}$ be the boundary elements not included in $S_x$.
    That is, if we denote by $\sigma$ a permutation on $\{1, \ldots, n\}$ that sorts the values in non-decreasing order $x_{\sigma(1)} \leq \cdots \leq x_{\sigma(n)}$, then $\underline{\iota}$ and $\overline{\iota}$ are defined as follow: $\sigma(\underline{\iota}) = f$ and $\sigma(\overline{\iota})= n - f + 1$.
    \begin{itemize}
        \item[$\circ$] If $\iota \notin S_x$ and $\iota \notin S_{x'}$, then with $\{j\} = S_{x} \setminus S_{x'}$, we have two sub-cases:
        \begin{itemize}
            \item[] Sub-case (a) $x_\iota > x_{\overline{\iota}}$ and $x_\iota' < x_{\underline{\iota}}$. In this case, $S_{x'} \setminus S_{x} = \{ x_{\underline{\iota}} \}$. Therefore, 
            \begin{align*}
                \textstyle |\mathrm{TM}(x) - \mathrm{TM}(x')| = \frac{1}{n-2f} | x_j-x_{\underline{\iota}} | \leq \frac{1}{n-2f} | x_{\overline{\iota}}-x_{\underline{\iota}} | \leq \frac{1}{n-2f} | x_\iota - x_\iota' |.
            \end{align*} 
            \item[] Sub-case (b) $x_\iota < x_{\underline{\iota}}$ and $x_\iota' > x_{\overline{\iota}}$. In this case, $S_{x'} \setminus S_{x} = \{ x_{\overline{\iota}} \}$. Therefore,
            \begin{align*}
                \textstyle |\mathrm{TM}(x) - \mathrm{TM}(x')| = \frac{1}{n-2f} | x_{\overline{\iota}} - x_j | \leq \frac{1}{n-2f} | x_{\overline{\iota}} - x_{\underline{\iota}} | \leq \frac{1}{n-2f} | x_\iota - x_\iota' |.
            \end{align*}
        \end{itemize}
        \item[$\circ$] If $\iota \in S_x$ but $\iota \notin S_{x'}$ then we have the following two sub-cases:
        \begin{itemize}
            \item[] Sub-case (c) $x_{\underline{\iota}} \in S_{x'}$. In this case, $x_\iota' \leq x_{\underline{\iota}}$. Therefore,
            \begin{align*}
                \textstyle |\mathrm{TM}(x) - \mathrm{TM}(x')| = \frac{1}{n-2f}|x_{\underline{\iota}}-x_\iota| \leq \frac{1}{n-2f}|x'_\iota-x_\iota| .
            \end{align*} 
            \item[] Sub-case (d) $x_{\overline{\iota}} \in S_{x'}$. In this case, $x_\iota' \geq x_{\overline{\iota}}$. Therefore,
            \begin{align*}
                \textstyle |\mathrm{TM}(x) - \mathrm{TM}(x')| = \frac{1}{n-2f}|x_{\overline{\iota}}-x_\iota| \leq \frac{1}{n-2f}|x'_\iota-x_\iota| .
            \end{align*} 
        \end{itemize}
        \item[$\circ$] By symmetry, we obtain the same result for the case when $\iota \notin S_x$ but $\iota \in S_{x'}$.
        \end{itemize}
    \item Case (ii) $| S_{x'} \setminus S_{x} | = 0$. The lemma is trivially true in this case.
\end{itemize}
From above, we obtain that $|\mathrm{TM}(x)-\mathrm{TM}(x')| \leq \frac{1}{n-2f} \|x-x'\|_1$.

We can generalize this bound to two sets of $n$ real numbers that can differ in all the entries, using triangle inequality. Specifically,
\begin{align*}
    |\mathrm{TM}(x)-\mathrm{TM}(x')|
    & = | \mathrm{TM}(x)-\mathrm{TM}(x_1,\ldots, x_{n-1}, x'_n) \\ & \qquad + \mathrm{TM}(x_1,\ldots, x_{n-1}, x'_n) - \mathrm{TM}(x_1,\ldots, x_{n-2}, x'_{n-1}, x'_n) \\ & \qquad + \cdots + \mathrm{TM}(x_1, x'_2, \ldots, x'_n) - \mathrm{TM}(x') | \\
    & \leq | \mathrm{TM}(x)-\mathrm{TM}(x_1,\ldots, x_{n-1}, x'_n)| \\ & \qquad + \cdots + |\mathrm{TM}(x_1, x'_2, \ldots, x'_n)-\mathrm{TM}(x') |\\ 
    & \leq \frac{1}{n-2f} \sum_{i=1}^n |x_i-x'_i| 
    = \frac{1}{n-2f} \|x-x'\|_1
    \leq \frac{\sqrt{n}}{n-2f} \|x-x'\|_2 \tag*{\BlackBox}
\end{align*}
Lemma~\ref{tm-lipschitz} and~\ref{trimmed-mean-lowerbound} enable to establish the expansivity of the $\operatorname{CWTM}$ update in the case of smooth and nonconvex loss functions, as shown in the following result. 
\begin{lemma}~\label{lemmarobustexpansivity}
    Define the robust gradient update~\eqref{eq:algo-update} when $F=\operatorname{CWTM}$ with learning rate $\gamma > 0$, for arbitrary $z^{(i)} \in \mc{Z}$, any $i \in \{1, \ldots, n\}$, and any coordinate $k \in \{1, \ldots, d\}$, as follows.
    % \begin{multline*}
    %     {\left[ G^{\mathrm{CWTM}}_{\gamma}(\theta) \right]}_k 
    %     = {\left[ \theta \right]}_k - \gamma \mathrm{TM} \left( {\left[ \nabla \ell (\theta; z^{(1)}) \right]}_k, \ldots, {\left[ \nabla \ell (\theta; z^{(n)}) \right]}_k \right) \\
    %     = {\left[ \theta \right]}_k - \frac{\gamma}{n-2f} \sum_{i=f+1}^{n-f} {\left[ \nabla \ell (\theta; z^{(\sigma_k(i))}) \right]}_k
    %     = {\left[ \theta \right]}_k - \frac{\gamma}{n-2f} \sum_{i \in S^{(k)}} {\left[ \nabla \ell (\theta; z^{(i)}) \right]}_k.
    % \end{multline*}
    % where $\sigma_k$ denotes the permutation that sorts the values $({\left[ \nabla \ell (\theta; z^{(1)}) \right]}_k, \ldots, {\left[ \nabla \ell (\theta; z^{(n)}) \right]}_k)$ in ascending order, 
    % and $S^{(k)} = \{ i \in [n]; f+1 \leq \sigma_k(i) \leq n-f \}$ denotes the index set of central values in coordinate $k$.
    % We can write, 
    \[
        G^{\mathrm{CWTM}}_{\gamma}(\theta) = \theta - \gamma \mathrm{CWTM}(\nabla \ell (\theta; z^{(1)}), \ldots, \nabla \ell (\theta; z^{(n)})).
    \] 
    Assume \( \forall z \in \mc{Z}, \ell(\cdot, z): \Theta \to \mb{R} \) is L-smooth. 
    Then $G^{\mathrm{CWTM}}_{\gamma}$ is $(1+ \gamma L\min\{\frac{n}{n-2f}, \sqrt{n}, \sqrt{d}\})$-expansive \citep[Definition 2.3]{hardt2016train}. 
\end{lemma}
% \begin{proof}
\textbf{Proof}\;
    Let $\theta, \omega \in \mb{R}^d$ and $z^{(1)}, \ldots, z^{(n)} \in \mc{Z}$. For convenience, we occasionally adopt the shorthand notation: $\ell(\cdot; z^{(i)}) = \ell_i(\cdot)$.
    Using Lemma~\ref{tm-lipschitz} and the $L$-smoothness assumption, we obtain that
    \begin{align}
        & \| \mathrm{CWTM}(\nabla \ell_1 (\theta), \ldots, \nabla \ell_n (\theta)) - \mathrm{CWTM}(\nabla \ell_1 (\omega), \ldots, \nabla \ell_n (\omega)) \|_2^2 \nonumber \\
        & = \sum_{k=1}^d {\left( \mathrm{TM}({\left[ \nabla \ell_1 (\theta) \right]}_k, \ldots, {\left[ \nabla \ell_n (\theta) \right]}_k) - \mathrm{TM}({\left[ \nabla \ell_1 (\omega) \right]}_k, \ldots, {\left[ \nabla \ell_n (\omega) \right]}_k) \right)}^2 \nonumber \\
        & \leq \frac{n}{{(n-2f)}^2} \sum_{k=1}^d \| ({\left[ \nabla \ell_1 (\theta) \right]}_k, \ldots, {\left[ \nabla \ell_n (\theta) \right]}_k) -  ({\left[ \nabla \ell_1 (\omega) \right]}_k, \ldots, {\left[ \nabla \ell_n (\omega) \right]}_k) \|_2^2 \nonumber \\
        & = \frac{n}{{(n-2f)}^2} \sum_{k=1}^d \sum_{i=1}^n {\left( {\left[ \nabla \ell_i (\theta) \right]}_k - {\left[ \nabla \ell_i (\omega) \right]}_k \right)}^2
        = \frac{n}{{(n-2f)}^2} \sum_{i=1}^n \| \nabla \ell_i (\theta) - \nabla \ell_i (\omega) \|_2^2 \nonumber \\
        & \leq \frac{n}{{(n-2f)}^2} \sum_{i=1}^n L^2 \| \theta - \omega \|_2^2
        = \frac{n^2 L^2}{{(n-2f)}^2} \| \theta - \omega \|_2^2 . \label{eqn:bnd_cwtm-1}
    \end{align}
    Alternatively, we can invoke Lemma~\ref{trimmed-mean-lowerbound}, which states that $\forall k \in \{1, \ldots, d\}$, there exists $i_k \in \{1, \ldots, n\}$ such that:
        \begin{align*}
        & \| \mathrm{CWTM}(\nabla \ell_1 (\theta), \ldots, \nabla \ell_n (\theta)) - \mathrm{CWTM}(\nabla \ell_1 (\omega), \ldots, \nabla \ell_n (\omega)) \|_2^2 \\
        & = \sum_{k=1}^d {\left( \mathrm{TM}({\left[ \nabla \ell_1 (\theta) \right]}_k, \ldots, {\left[ \nabla \ell_n (\theta) \right]}_k) - \mathrm{TM}({\left[ \nabla \ell_1 (\omega) \right]}_k, \ldots, {\left[ \nabla \ell_n (\omega) \right]}_k) \right)}^2 \\
        & \leq \sum_{k=1}^d {\left[ \nabla \ell_{i_k} (\theta) - \nabla \ell_{i_k} (\omega) \right]}_k^2.        
    \end{align*}
        From above, we obtain the following:
    \begin{enumerate}
        \item Since $\sum_{k=1}^d {\left[ \nabla \ell_{i_k} (\theta) - \nabla \ell_{i_k} (\omega) \right]}_k^2
        \leq \sum_{k=1}^d \sum_{i=1}^n {\left[ \nabla \ell_i (\theta) - \nabla \ell_i (\omega) \right]}_k^2
        \leq n L^2 \| \theta - \omega \|_2^2$,
        \begin{align*}
            \| \mathrm{CWTM}(\nabla \ell_1 (\theta), \ldots, \nabla \ell_n (\theta)) - \mathrm{CWTM}(\nabla \ell_1 (\omega), \ldots, \nabla \ell_n (\omega)) \|_2^2 \leq n L^2 \| \theta - \omega \|_2^2. 
        \end{align*}
        \item Since $\sum_{k=1}^d {\left[ \nabla \ell_{i_k} (\theta) - \nabla \ell_{i_k} (\omega) \right]}_k^2 \leq \sum_{k=1}^d \| \nabla \ell_{i_k} (\theta) - \nabla \ell_{i_k} (\omega) \|^2_2$,
        \begin{align*}
            \| \mathrm{CWTM}(\nabla \ell_1 (\theta), \ldots, \nabla \ell_n (\theta)) - \mathrm{CWTM}(\nabla \ell_1 (\omega), \ldots, \nabla \ell_n (\omega)) \|_2^2 \leq d L^2 \| \theta - \omega \|_2^2. 
        \end{align*}
    \end{enumerate}
    Therefore, 
    \begin{align}
        \| \mathrm{CWTM}(\nabla \ell_1 (\theta), \ldots, \nabla \ell_n (\theta)) - \mathrm{CWTM}(\nabla \ell_1 (\omega), \ldots, \nabla \ell_n (\omega)) \|_2^2 \leq \min\{n, d\} L^2 \| \theta - \omega \|_2^2. \label{eqn:alt_bnd}
    \end{align}
    The bound in~\eqref{eqn:alt_bnd} has a tighter dependence on the number of input vectors than the bound in~\eqref{eqn:bnd_cwtm-1}, especially in the regime when $f$ approaches $\frac{n}{2}$.
    We conclude the proof by combining~\eqref{eqn:bnd_cwtm-1} and~\eqref{eqn:alt_bnd}, and by applying the triangle inequality, yielding
    \begin{equation*}\label{eq:cwtm-expansiveness}
        \| G^{\mathrm{CWTM}}_{\gamma}(\theta) - G^{\mathrm{CWTM}}_{\gamma}(\omega) \|_2 \leq  \left(1 + \gamma L \min \left\{\frac{n}{n-2f}, \sqrt{n}, \sqrt{d } \right\} \right) \| \theta - \omega \|_2. 
        \tag*{\BlackBox}
    \end{equation*}
\section{Non-Affine Aggregations Can Preserve Smoothness: Consequence for the Algorithmic Stability of Robust Distributed Learning}\label{deferred-sec-5}
%In the following section, we establish uniform algorithmic upper bounds under data poisoning and nonconvex objective for $\mathrm{CWTM}$.
To broaden our analysis, we show how the smoothness of the loss function (i.e., the Lipschitz-continuity of its gradient) can be leveraged to derive sharper stability guarantees for robust $\mathrm{SGD}$ under data poisoning in the nonconvex setting. 
This motivates $\mathrm{CWTM}$ (cf. Definition~\ref{cwtm-definition}), as it is an instance of an aggregation rule that preserves the loss function's smoothness (cf.~Appendix~\ref{tm}), due to the Lipschitz continuity of the trimmed mean operation (cf. Lemma~\ref{tm-lipschitz}) or its translation equivariance property (cf. Lemma~\ref{trimmed-mean-lowerbound}).
\begin{lemma}~\label{lem:cwtm-smoothness}
    Let $d, n \in \mb{N} \setminus \{0\}$.
    The aggregated operator $\F = \mathrm{CWTM}(\nabla \ell_1(\cdot), \ldots, \nabla \ell_n(\cdot))$ preserves the smoothness inequality of every $L$-smooth loss functions $\{ \ell_i \}_{i \in [n]}$, i.e., for all $\theta, \omega \in \R^d$
    \[
        \| \F(\theta) - \F(\omega) \| \leq \min\{\frac{n}{n-2f}, \sqrt{n}, \sqrt{d}\} L \| \theta - \omega \|_2^2.
    \]
\end{lemma}

\paragraph{Consequence for robust distributed nonconvex and smooth learning}\label{sec3.3-noncvx}
To illustrate the practical benefit of preserving smoothness, we provide an enhanced uniform algorithmic stability analysis for SGD under data poisoning.
Specifically, we consider the robust setting where updates represent gradients of a loss $\ell$, a fraction $\frac{f}{n}$ of which are computed on arbitrarily corrupted data.
Lemma~\ref{lem:cwtm-smoothness} enables the following stability upper bounds established in~Appendix~\ref{deferred-sec-5}.
%In the following, we derive sharper uniform stability upper bound with $\mathrm{CWTM}$ aggregation rule for smooth and nonconvex loss functions.
%This improvement---particularly in its dependence on the number of local samples---is enabled by aggregation rules that preserve the smoothness of the loss function.
\begin{theorem}\label{poi-cwtm-ub}
    Consider the setting described in~Section~\ref{I-intro} under data poisoning.
    Let $\mc{A}=\mathrm{SGD}$, with $\mathrm{CWTM}$. Suppose $\mc{A}$ is run for $T\in\mb{N} \setminus \{0\}$ iterations and that there exist $\nu>0$ constant such that $n\geq (2+\nu)f$ --- that is, $f/n$ is strongly bounded away from $1/2$.
    Let we assume $\forall z \in \mc{Z}$, $\ell(\cdot;z)$ bounded by $\ell_\infty$, nonconvex, $C$-Lipschitz and $L$-smooth.
    Then, with a monotonically non-increasing learning rate, $\gamma_t \leq \frac{\nu}{2+\nu}\frac{c}{Lt}$, $t \in \{0, \ldots, T-1\}$, $c>0$,
    we have for any neighboring datasets $\mc{S}, \mc{S'}$:
    \begin{equation}\label{ub-noncvx-pois3}
        \sup_{z \in \mathcal{Z}} \mb{E}_{\mc{A}}[ | \ell(\mc{A}(S);z) - \ell(\mc{A}(S');z) | ] 
        \leq 2 {\left( \frac{2 C^2 \nu^2}{{(2+\nu)}^2 L} \right)}^{\frac{1}{c+1}} \frac{{\big(T \ell_\infty \big)}^{\frac{c}{c+1}}}{m{\sqrt{n}}^{\frac{1}{c+1}}}
    \end{equation}
\end{theorem}
% \begin{proof}
\textbf{Proof}\;
    Let denote $\{{\theta_t\}}_{t \in \{0, \ldots, T-1\}}$ and ${\{\theta'_t\}}_{t \in \{0, \ldots, T-1\}}$ the optimization trajectories resulting from two neighboring datasets $\mc{S}, \mc{S}'$, for $T$ iterations of $\mathrm{SGD}$ with aggregation rule $\mathrm{CWTM}$ and learning rate schedule ${\{ \gamma_t \}}_{t \in \{0, \ldots, T-1\}}$. 
    Importantly, poisoned vectors are treated as legitimate gradients, despite being computed on arbitrary data.
    In the following, we draw a growth recursion for the sensitivity of the parameters to the perturbation introduced. 
    Let $a, b \in \mc{H} \times \{1, \ldots, m\}$ the indices of the differing samples. For $t \in \{0, \ldots, T-1\}$, we denote: 
    \[
        z_t^{(i)} = 
        \begin{cases}
        \text{an arbitrary poisoned data point}, & \text{if } i \notin \mathcal{H} \\
        z^{(i, J_t^{(i)})}, J_t^{(i)} \text{ sampled uniformly from } \{1, \ldots, m\} & \text{if } i \in \mathcal{H}
        \end{cases},
    \]
    $\delta_{t} = \| \theta_{t} - \theta'_{t} \|_2$,
    $g^{(i)}_t = \nabla \ell (\theta_t, z_t^{(i)})$,
    $G^{\mathrm{CWTM}}_{\gamma_t}(\theta_t) = \theta_t - \gamma_t \mathrm{CWTM}(g^{(1)}_t, \ldots, g^{(n)}_t)$, 
    and the prime notation ${(\cdot)}^\prime$ denotes the corresponding quantities for $\mc{S}'$. 

    The proof leverages the expansivity property of the robust aggregation update without directly comparing it to the honest stochastic gradient update.
    Specifically, we either directly apply the expansivity result Lemma~\ref{lemmarobustexpansivity} with probability $(1-\frac{1}{m})$: 
    \[
        \| G^{\mathrm{CWTM}}_{\gamma}(\theta_t) - G^{\prime \mathrm{CWTM}}_{\gamma}(\theta'_t) \|_2 \leq (1+\gamma_t L \min\{\frac{n}{n-2f}, \sqrt{n}, \sqrt{d}\}) \| \theta_t - \theta'_t \|_2
    \]
    or perform the following calculation with probability $\frac{1}{m}$:
    \begin{multline*}
        \| \mathrm{CWTM}(g^{(1)}_t, \ldots, g^{(n)}_t) - \mathrm{CWTM}(g^{\prime(1)}_t, \ldots, g^{\prime(n)}_t) \|_2^2 \\
        \qquad\qquad = \sum_{k=1}^d {\left( TM({\left[ g^{(1)}_t \right]}_k, \ldots, {\left[ g^{(n)}_t \right]}_k) - TM({\left[ g^{\prime(1)}_t \right]}_k, \ldots, {\left[ g^{\prime(n)}_t \right]}_k) \right)}^2 \hfill \\
        \qquad\qquad \leq \frac{n}{{(n-2f)}^2} \sum_{k=1}^d \| ({\left[ g^{(1)}_t \right]}_k, \ldots, {\left[ g^{(n)}_t \right]}_k) -  ({\left[ g^{\prime(1)}_t \right]}_k, \ldots, {\left[ g^{\prime(n)}_t \right]}_k) \|_2^2 \hfill \\
        \qquad\qquad = \frac{n}{{(n-2f)}^2} \sum_{k=1}^d \sum_{i=1}^n {\left( {\left[ g^{(i)}_t \right]}_k - {\left[ g^{\prime(i)}_t \right]}_k \right)}^2 \hfill \\
        \qquad\qquad = \frac{n}{{(n-2f)}^2} \sum_{i=1}^n \| \nabla \ell (\theta_t; z_t^{(i)}) - \nabla \ell (\theta'_t; z_t^{\prime(i)}) \|_2^2 \hfill \\
        \qquad\qquad = \frac{n}{{(n-2f)}^2} \Bigg(  \sum_{i=1, i \neq a}^n \| \nabla \ell (\theta_t; z^{(i)}) - \nabla \ell (\theta'_t; z^{(i)}) \|_2^2
        + \| \nabla \ell (\theta_t; z^{(a, b)}) - \nabla \ell (\theta'_t; z^{\prime(a, b)}) \|_2^2 \Bigg) \hfill \\
        \qquad\qquad \leq \frac{n}{{(n-2f)}^2} \left( 4C^2 + (n-1) L^2 \| \theta_t - \theta'_t \|_2^2 \right) \hfill
    \end{multline*}
    where we used Lemma~\ref{tm-lipschitz} in the first inequality. Using the sub-additivity of the square root we prove that we have with probability $\frac{1}{m}$:
    \[
        \| G^{\mathrm{CWTM}}_{\gamma}(\theta_t) - G^{\prime \mathrm{CWTM}}_{\gamma}(\theta'_t) \|_2 \leq (1+\frac{\gamma_t L \sqrt{n(n-1)}}{n-2f}) \| \theta_t - \theta'_t \|_2 + \frac{2 \gamma_t C \sqrt{n}}{n-2f}.
    \]
    Consequently, with $t_0 \in \{0, \ldots, T-1\}$, we obtain the following recursion:
    \begin{align*}
        \mb{E}_{\mc{A}}[\delta_{t+1}|\delta_{t_0}=0] 
        & \leq (1+\frac{\gamma_t L n}{n-2f}) \mb{E}_{\mc{A}}[\delta_{t}|\delta_{t_0}=0] + \frac{2 \gamma_t C \sqrt{n}}{(n-2f)m} \\
        & \leq (1+\frac{\gamma_t L (2+\nu)}{\nu}) \mb{E}_{\mc{A}}[\delta_{t}|\delta_{t_0}=0] + \sigma_{f, m, n} \frac{c}{Lt}\\
        & = (1+\frac{c}{t}) \mb{E}_{\mc{A}}[\delta_{t}|\delta_{t_0}=0] + \sigma_{f, m, n} \frac{c}{Lt}.
    \end{align*}
    where we used that as $n\geq (2+\nu)f$, then $\frac{n}{n-2f} \leq \frac{2+\nu}{\nu}$.
    and defined $\sigma_{f, m, n} = \frac{2C\nu\sqrt{n}}{(2+\nu)(n-2f)m} \leq \frac{2C\nu^2}{{(2+\nu)}^2\sqrt{n}m}$.
    For $0 \leq t_0 \leq \min(m, T)$, $t \in \{t_0, \ldots, T-1\}$, $\gamma_t \leq \frac{\nu}{2+\nu}\frac{c}{Lt}$, $c>0$, 
    we use $1+x \leq e^x$, unroll the recursion and then invoke Lemma E.2 from~\citet{boudou2025generalization} %Lemma~\ref{lem:sum-product-exponential}:
    \begin{align*}
        \mb{E}_{\mc{A}}[ \delta_{T} |\delta_{t_0}=0] 
        & \leq \frac{c \sigma_{f, m, n}}{L} \sum_{t=t_0}^{T-1} \frac{1}{t}  \prod_{s=t+1}^{T-1} e^{\frac{c}{s}}
        \leq \frac{\sigma_{f, m, n}}{L} {\left( \frac{T}{t_0} \right)}^{c}
    \end{align*}
    Substituting our derivation into the bound from Lemma E.1 from~\citet{boudou2025generalization} yields %Lemma~\ref{lem:lemmagen2} yields:
    \begin{align*}
        \mb{E}_{\mc{A}}[ | \ell(\theta_T;z) - \ell(\theta_T';z) | ] 
        & \leq \min_{0 \leq t_0 \leq \min(m, T)} \frac{t_0\ell_\infty}{m} + \frac{\sigma_{f, m, n}C}{L} {\left( \frac{T}{t_0} \right)}^{c} \\
        & \leq \min_{0 \leq t_0 \leq \min(m, T)} \frac{t_0\ell_\infty}{m} + \frac{2C^2\nu^2}{{(2+\nu)}^2 L \sqrt{n}m} {\left( \frac{T}{t_0} \right)}^{c}.
    \end{align*}
    We approximately minimize the expression with respect to $t_0$ by balancing the two terms, i.e., setting them equal:
    \[
        0 \leq \widetilde{t_0} = {\left( \frac{m\sigma_{f, m, n} C}{L \ell_\infty} \right)}^{\frac{1}{c+1}} T^{\frac{c}{c+1}} 
        \underset{\textit{for sufficiently large T} \geq \frac{m\sigma_{f, m, n} C}{L \ell_\infty}}{\leq} T.
    \]
    Finally, we obtain the result by directly substituting $\widetilde{t_0}$ into the bound:
    \begin{multline*}
        \mb{E}_{\mc{A}}[ | \ell(\theta_T;z) - \ell(\theta_T';z) | ] 
        \leq 2 {\left( \frac{\sigma_{f, m, n} C}{L} \right)}^{\frac{1}{c+1}} {\left( \frac{\ell_\infty T}{m} \right)}^{\frac{c}{c+1}} \\
        = 2 {\left( \frac{2 C^2 \nu^2}{{(2+\nu)}^2 L \sqrt{n}m} \right)}^{\frac{1}{c+1}} {\left( \frac{\ell_\infty T}{m} \right)}^{\frac{c}{c+1}} 
        = 2 {\left( \frac{2 C^2 \nu^2}{{(2+\nu)}^2 L} \right)}^{\frac{1}{c+1}} {\left( \ell_\infty T \right)}^{\frac{c}{c+1}}  \frac{1}{m {\sqrt{n}}^{\frac{1}{c+1}} }.
        \tag*{\BlackBox}
    \end{multline*}
% \end{proof}
For $f=0$,~\eqref{ub-noncvx-pois3} yields a suboptimal $\sqrt{n}$ dependence instead of $n$~\citep{pmlr-v235-le-bars24a}—a proof artifact, as this gap does not appear in~\citet[Theorem E.3]{boudou2025generalization}, which also applies to our data poisoning setting.  
We next compare our bound with the one obtained using an aggregation rule that does not preserve the smoothness inequality, such as the $\mathrm{SMEA}$ rule introduced in~\citet{pmlr-v202-allouah23a}. 
For $f > 0$, let $\varepsilon^{\text{poisoning}}_{\mathrm{SMEA}}$ denote the upper bound from~\citet[Theorem E.3]{boudou2025generalization}, and $\varepsilon^{\text{poisoning}}_{\mathrm{CWTM}}$ the bound with $\mathrm{CWTM}$ from~\eqref{ub-noncvx-pois3}. 
We compare:
$
    \varepsilon^{\text{poisoning}}_{\mathrm{SMEA}} / \varepsilon^{\text{poisoning}}_{\mathrm{CWTM}}
    = {\big( \frac{{(2+\nu)}^2}{\nu^2} \big( \frac{\sqrt{n}}{n-f} + f m \frac{\sqrt{n}}{n-f} \big) \big)}^{\frac{1}{c+1}}.
$
Interestingly, by leveraging the smoothness-preserving property of $\mathrm{CWTM}$, we obtain an improved dependence on $m$: as $m$ grows unbounded, $\varepsilon^{\text{poisoning}}_{\mathrm{CWTM}}$ decreases faster than $\varepsilon^{\text{poisoning}}_{\mathrm{SMEA}}$.
%As $f/n$ approaches $1/2$, this favorable dependence persists (bound omitted), though comparisons involving other variables are less direct.

%\clearpage\input{./appendices/appNEXT.tex}
%\clearpage\input{./appendices/appTMP.tex}

\vskip 0.2in
\bibliography{example_paper}

\end{document}